\documentclass{article}

\usepackage{microtype}
\usepackage{graphicx}
\usepackage{subfigure}
\usepackage{booktabs} 

\usepackage[dvipsnames]{xcolor}

\usepackage{hyperref}

\PassOptionsToPackage{numbers, compress}{natbib}

 \usepackage[preprint]{neurips_2023_hack}

\usepackage{amsmath}
\usepackage{amssymb}
\usepackage{mathtools}
\usepackage{amsthm}
\usepackage{apxproof}
\usepackage{bbm}
\graphicspath{{Experiments/}}
\usepackage[capitalize,noabbrev]{cleveref}

\usepackage{algorithmic}
\usepackage{algorithm}

\theoremstyle{plain}
\newtheorem{theorem}{Theorem}[section]
\newtheorem{proposition}[theorem]{Proposition}
\newtheorem{lemma}[theorem]{Lemma}

\newtheorem{claim}[theorem]{Claim}
\theoremstyle{definition}

\newtheorem{example}[theorem]{Example}

\theoremstyle{remark}

\usepackage{pifont}%for defining \xmark
\newcommand{\cmark}{\textcolor{ForestGreen}{\ding{51}}}%
\newcommand{\xmark}{\textcolor{RedOrange}{\ding{55}}}%

\usepackage[colorinlistoftodos,bordercolor=orange,backgroundcolor=orange!20,linecolor=orange,textsize=scriptsize]{todonotes}

\usepackage{nicefrac}

\usepackage{xspace}
\newcommand{\tblalgname}[1]{{\color{JungleGreen} \tiny \sf  #1}\xspace} %algname for table
\newcommand{\algname}[1]{{\color{JungleGreen} \small \sf  #1}\xspace}

\newcommand{\squeeze}{\textstyle} % ON

\newcommand{\cB}{{\cal B}}
\newcommand{\cC}{{\cal C}}
\newcommand{\cD}{{\cal D}}

\newcommand{\cN}{{\cal N}}
\newcommand{\cO}{{\cal O}}

\newcommand{\R}{\mathbb{R}}

\newcommand{\clip}{{\rm\sf\color{Bittersweet}clip}}
\newcommand{\Exp}[1]{\mathbf{E}\left[#1\right]}
\newcommand{\eqdef}{\coloneqq}
\newcommand{\norm}[1]{\left\|#1\right\|}

\newcommand{\FF}{F_0}
\newcommand{\GG}{G_0}

\title{Clip21: Error Feedback for Gradient Clipping}

\author{Sarit Khirirat\\Department of Machine Learning\\MBZUAI${}^{\dagger}$
\And Eduard Gorbunov\\Department of Machine Learning\\MBZUAI
\And Samuel Horv\'{a}th\\Department of Machine Learning\\MBZUAI
\And Rustem Islamov\\Department of Applied Mathematics\\Institut Polytechnique de Paris
\And Fakhri Karray\\Department of Machine Learning\\MBZUAI
\And Peter Richt\'{a}rik\\Artificial Intelligence Initiative\\KAUST${}^{\ddagger}$}

\begin{document}
	
\maketitle

\begin{abstract}
Motivated by the increasing popularity and importance of large-scale training under differential privacy (DP) constraints, we study distributed gradient methods with {\em gradient clipping}, i.e., clipping applied to the gradients computed from local information at the nodes. While gradient clipping is an essential tool for injecting formal DP guarantees into gradient-based methods~\citep{abadi2016deep}, it also induces bias which causes serious convergence issues specific to the distributed setting. Inspired by  recent progress in the error-feedback literature which is focused on taming the bias/error introduced by communication compression operators such as Top-$k$~\citep{richtarik2021ef21}, and  mathematical similarities between the clipping operator and contractive compression operators, we design \algname{Clip21} -- the first provably effective and practically useful error feedback mechanism for distributed methods with gradient clipping. We prove that our method converges at the same $\cO(\nicefrac{1}{K})$ rate as distributed gradient descent in the smooth nonconvex regime, which improves the previous best $\cO(\nicefrac{1}{\sqrt{K}})$ rate which was obtained under significantly stronger assumptions.
Our method converges significantly faster in practice than competing methods. 
\end{abstract}

%%%%%%%%%%%%%%%%%%
%%%%%%%%%%%%%%%%%%
\section{Introduction}
%%%%%%%%%%%%%%%%%%
%%%%%%%%%%%%%%%%%%

Gradient clipping is a popular and versatile  tool used in several areas of machine learning. For example, it is employed to i) enforce bounded $\ell_2$ sensitivity in order to obtain formal differentially privacy guarantees in gradient-based optimization methods~\citep{abadi2016deep,chen2020understanding}, ii) tame the exploding gradient problem in deep learning~\citep{pascanu2013difficulty,  Zhang2020Why}, iii) stabilize convergence of \algname{SGD} in the heavy-tailed noise regime~\cite{nazin2019algorithms,gorbunov2020stochastic}, and iv) design provably Byzantine-robust gradient aggregators~\citep{karimireddy2021learning}. 

\subsection{The problem}
Our work is motivated by the increasing popularity and importance of large-scale training under differential privacy (DP) constraints \citep{DP-ImageNet-2022,FL-DP-google-blog-2022}. 
In particular, we wish to solve the optimization problem
\begin{equation}\label{eq:main}  \squeeze
\min\limits_{x\in \R^d} \left[f(x)\eqdef \frac{1}{n}\sum\limits_{i=1}^n f_i(x)\right],
\end{equation}
where $n$ is the number of clients and $f_i$ is the loss of a model parameterized by vector $x\in \R^d$ over all private data $\cD_i$ owned by client $i\in [n]\eqdef \{1,\dots,n\}$. We assume each $f_i$ has $L_i$-Lipschitz gradient, i.e., \begin{equation}\label{eq:L_i-smooth}\squeeze \norm{\nabla f_i(x) - \nabla f_i(y)} \leq  L_i \norm{x-y} \end{equation} for all $x,y\in \R^d$, where $\norm{x}\eqdef \langle x,x\rangle^{1/2}$ and $\langle x,y\rangle \eqdef \sum_{i=1}^d x_i y_i$ is the standard Euclidean inner product.  This implies that $f$ has $L$-Lipschitz gradient, i.e.,
\begin{equation}\label{eq:L-smooth}\squeeze \norm{\nabla f(x) - \nabla f(y)} \leq  L \norm{x-y} \end{equation}
for all $x,y\in \R^d$, where $L$ satisfies $L\leq \frac{1}{n}\sum_{i=1}^n L_i$. Let $L_{\max}\eqdef \max_i L_i$. Moreover, we assume $f$ to be  lower bounded by some $f_{\inf}\in \R$. Note that we do not require any convexity assumptions.  Further, we do not assume $f_i$ nor $f$ to be Lipchitz, which is a typical assumption used to enforce bounded $\ell_2$ sensitivity -- a crucial property when proving formal DP guarantees.  

\subsection{Optimization with gradient clipping}
Perhaps the simplest DP algorithm for solving \eqref{eq:main}  in this regime is \algname{DP-Clip-GD} \citep{abadi2016deep,chen2020understanding}, performing the iterations
\begin{equation}\label{eq:DP-GD} \squeeze  x_{k+1} = x_k - \gamma_k \left(\frac{1}{n}\sum\limits_{i=1}^n \clip_\tau (\nabla f_i(x_k) ) + \zeta_k \right),\end{equation}
where $\gamma_k>0$ is a stepsize, $\zeta_k\sim \cN(0,\sigma^2 I_d)$ is zero-mean Gaussian noise with variance $\sigma^2 \geq 0$, and $\clip_\tau: \R^d \to \R^d$ is the clipping operator with threshold $\tau>0$, defined via 
\begin{equation}\label{eq:clip} \squeeze \clip_\tau(x) \eqdef \begin{cases} x & \text{if} \quad \norm{x} \leq \tau \\
\frac{\tau}{\norm{x}} x & \text{if} \quad \norm{x}> \tau \end{cases}.\end{equation}
Formal privacy guarantees for algorithm \eqref{eq:DP-GD} can be found in \citep[Theorem 1]{abadi2016deep}\footnote{\citep{abadi2016deep} consider the more general setting with {\em subsampled} gradients $\{\nabla f_i : i\in S_k\}$ for random $S_k\subseteq [n]$. }. 
When $\sigma^2=0$, we will refer to this method as \algname{Clip-GD}, dropping the DP designation.

	\begin{table*}
	\caption{\footnotesize Known and our new results for first order methods with clipping.  ${}^{(a)}$ It is not clear if the analysis  is correct:  \url{https://openreview.net/forum?id=hq7vLjZTJPk}. ${}^{(b)}$ Relies on bounded gradient and gradient similarity assumptions:  there exists $G\geq 0$ and $\sigma_g \geq 0$ such that $\norm{\nabla f_i(x)}^2 \leq G^2$ and $\norm{\nabla f_i(x) - \nabla f(x)}^2 \leq \sigma_g^2$ for all $x\in \R^d$ and all $i\in [n]$. ${}^{(c)}$ considers local steps for communication efficiency, but these lead to a worse communication efficiency guarantee. ${}^{(d)}$  $\sigma^2$ is the variance of the Gaussian noise.}\label{tbl:many_methods}
	\centering		
	\scriptsize
        \setlength{\tabcolsep}{1pt}
	\begin{tabular}{cccccc}
		%    \hline 
		\bf Algorithm &   \bf \begin{tabular}{c} Covers\\ $n>1$ Case?\end{tabular}  &  \bf \begin{tabular}{c} Nonconvex \\ Rate  \end{tabular} & \bf \begin{tabular}{c} DP \\  Guarantees \end{tabular} & \bf \begin{tabular}{c}Communication \\  Compression \end{tabular} & \bf Comment \\
		%%%%%%%%%%%%%%%%%%%%%%%%%%%		
		\hline 
		\hline 			
		\begin{tabular}{c} \tblalgname{Clip-GD} \\ \citep{Zhang2020Why} \end{tabular}&  \xmark & $\cO(\nicefrac{1}{K})$ & \xmark & --- & \begin{tabular}{c} $(L_0',L_1')$-smoothness \\
			applies to $n=1$ case only	\end{tabular} \\
		
		%%%%%%%%%%%%%%%%%%%%%%%%%%%
		\hline 
		\begin{tabular}{c} \tblalgname{CE-FedAvg} \\ \citep{zhang2022understanding}  \end{tabular}&  \cmark & $\cO(\nicefrac{1}{\sqrt{K}})$ & \xmark & \xmark  \; ${}^{(c)}$ & 
		\begin{tabular}{c}strong assumptions ${}^{(b)}$\\  slow rate  \end{tabular}  \\
		%%%%%%%%%%%%%%%%%%%%%%%%%%%	
		\hline 
		\begin{tabular}{c} \tblalgname{CELGC} \\ \citep{liu2022communication}   \end{tabular}&  \cmark & $\cO(\nicefrac{1}{\sqrt{K}})$ ${}^{(a)}$ & \xmark & \xmark & \begin{tabular}{c}strong assumptions ${}^{(b)}$ \\  slow rate \end{tabular} \\		
		
		%%%%%%%%%%%%%%%%%%%%%%%%%%%		
		\hline 
		\hline 				
		\begin{tabular}{c} \tblalgname{Clip21-Avg} \\ {\color{red}NEW} (Alg~\ref{alg:Clip21-Avg}) \end{tabular}&  \cmark & \begin{tabular}{c}$\cO(\max\{0,1-K\})$ \\ (Thm~\ref{thm:Clip21-Avg}) \end{tabular} & \xmark & \xmark & \begin{tabular}{c} Exact solution in $O(1)$ steps \\
			Solves average estimation			\end{tabular} \\\hline 		
		\begin{tabular}{c} \tblalgname{Clip21-GD} \\ {\color{red}NEW} (Alg~\ref{alg:Clip21-GD})  \end{tabular}&  \cmark & \begin{tabular}{c}$\cO(\nicefrac{1}{K})$ \\ (Thm~\ref{thm:multinode-clippedEF21}) \end{tabular} & \xmark & \xmark & \begin{tabular}{c} Fast \tblalgname{GD}-like rate \\ under $L$-smoothness \end{tabular} \\%%%%			
		\hline 			
		\begin{tabular}{c} \tblalgname{DP-Clip21-GD} \\ {\color{red}NEW} (Alg~\ref{alg:DP-Clip21-GD}) \end{tabular}&  \cmark & \begin{tabular}{c} $\cO(e^{-K} + \frac{\ln (1/\delta)}{\epsilon}) $ ${}^{(d)}$ \\ (Thm~\ref{thm:DP-clippedEF21})\end{tabular} & \cmark & \xmark & \begin{tabular}{c} Linear rate up to $\cO(\sigma^2)$ distance \\ of optimum under PŁ assumption \end{tabular}\\
		
		%%%%
		
		\hline 			
		\begin{tabular}{c} \tblalgname{Press-Clip21-GD} \\ {\color{red}NEW} (Alg~\ref{alg:ClipPress21-GD})  \end{tabular}&  \cmark & \begin{tabular}{c}$\cO(\nicefrac{1}{K})$ \\ (Thm~\ref{thm:ClipPress21-GD})\end{tabular} & \xmark & \cmark & \begin{tabular}{c} Fast \tblalgname{GD}-like rate \\ under $L$-smoothness \end{tabular}\\				
		%%%%%%%%%%%%%%%%%%%%%%%%%%%			
		\hline 
		\hline        
	\end{tabular}      		   
\end{table*}

\subsection{Convergence of \algname{Clip-GD} in the $n=1$ case} \label{intro:subsec:n=1}
In the $n=1$ case, \algname{Clip-GD}  was studied by \citep{Zhang2020Why}, where it was shown that, under our assumptions\footnote{\citep{Zhang2020Why} need to additionally assume $f$ to be twice continuously differentiable.}, and for suitable stepsizes, $$\squeeze \norm{\nabla f(x_K)}^2 \leq \frac{20 L (f(x_0)-f_{\inf})}{K}.$$ This is the same rate as that of vanilla \algname{GD}, i.e., the clipping bias does not cause any convergence issues in the $n=1$ regime. \citep{Zhang2020Why} did not consider the  distributed ($n>1$) regime since their work was motivated by orthogonal considerations to ours: instead of tackling the issue of clipping bias inherent to the  distributed  regime, as we do, they set out to explain the efficacy of clipping as a tool for convergence stabilization of \algname{GD} for functions with rapidly growing gradients (this issue exists even in the single node  regime). To model functions with rapidly growing gradients, they consider the  $(L'_0, L'_1)$-smoothness assumption, which requires the bound $\norm{\nabla^2 f(x)} \leq L'_0 + L'_1 \norm{\nabla f(x)}$ to hold for all $x\in \R^d$. For twice continuously differentiable functions, this assumption specializes to  ours by setting $L'_0=L=L_1$ and $L'_1=0$. 
	
\subsection{Divergence of \algname{Clip-GD} in the $n>1$ case}
Surprisingly little is known about the convergence properties of \algname{Clip-GD} in the $n>1$ case.
The key reason behind this is the poor quality of \begin{equation} \label{eq:clip_estimator} \squeeze  g(x) \eqdef \frac{1}{n}\sum\limits_{i=1}^n \clip_\tau (\nabla f_i(x) ) \end{equation} as an estimator of $\nabla f(x)$. This issue does not arise in the $n=1$ case since clipped gradient retains the directional information of the original gradient and all that is lost is just some of the scaling information, which is in the light of the results of \citep{Zhang2020Why} described above is not problematic. However, the situation is more complicated in the  $n>1$ case as illustrated in the following example.

\begin{example} Let $d=1$, $n=2$ with $f_1(x)=\frac{\beta}{2} x^2$ and $f_2(x)=-\frac{\alpha}{2} x^2$, where $\beta > \alpha \geq \tau>0 $ and $x>0$. Both of these functions have Lipschitz gradients, and  $f$ has $(\beta-\alpha)$-Lipschitz gradient. Moreover, $f(x) = \frac{\beta-\alpha}{2}x^2$ is lower bounded below by $f_{\inf} =0$. 
So, this setup satisfies our assumptions.  Notice that while the gradient is equal to $\nabla f(x) = \frac{1}{2} \nabla f_1(x) + \frac{1}{2} \nabla f_1(x)  =  \beta x-\alpha x = (\beta-\alpha)x$, the gradient estimator \eqref{eq:clip_estimator} gives $g(x) = \frac{1}{2} \clip_\tau(\nabla f_1(x)) + \frac{1}{2} \clip_\tau(\nabla f_2(x)) = \min \{1, \frac{\tau}{| \beta x |}\} \beta x +  \min \{1, \frac{\tau}{ | - \alpha x | }\} (-\alpha x) = \tau-\tau = 0$ whenever $x \geq \frac{\tau}{\alpha}$. This means that \algname{Clip-GD} will not progress at all if initialized with  $x_0 \geq \frac{\tau}{\alpha}$. For example, if we choose $x_0 = \frac{\tau}{\alpha}$, then both the function value $f(x_0) = \frac{\beta-\alpha}{2}\frac{\tau^2}{\alpha^2}$ and the gradient $\nabla f(x_0) = (\beta - \alpha ) \frac{\tau}{\alpha}$ can be arbitrarily large (by fixing $\tau$ and $\alpha$, and increasing $\beta$), while the optimal value and gradient of $f$ are both zero.
\end{example}

As the above example illustrates, \algname{Clip-GD}  is a fundamentally flawed method  in the $n>1$ case, unable to converge from many starting points to {\em any} finite  degree of accuracy, however weak accuracy requirements we may have! This is true for any class of functions which includes the above example, and hence it is true for the class of functions we consider in this paper: lower bounded, with Lipschitz gradient.

\section{Summary of Contributions} \label{sec:contributions}

In the light of the above discussion, there are at least three ways forward: i) consider a different algorithm, ii) consider a different function class, or iii) change both the algorithm and the function class. In our work we explore the first of these three possible approaches: we design a new way of combining clipping and gradient descent -- one that does not suffer from any convergence issues.  

\subsection{\algname{Clip21-Avg}: Error feedback for average estimation with clipping} 
As a first step in our process of discovery of a fix for the bias caused by the clipping estimator  \eqref{eq:clip_estimator} in the estimation of  the gradient, we first study a simplified setting void of any optimization aspect, thus removing one source of dynamics, which greatly simplifies the situation. In particular, we study the problem of the estimation of the average of a number of {\em fixed} vectors $a^1,\dots,a^n\in \R^d$ via repeated use of clipping. 
We then propose an error feedback mechanism for this task, leading to method \algname{Clip21-Avg} (see Algorithm~\ref{alg:Clip21-Avg}), and prove that our method finds the {\em exact} average in $$K=\cO(\nicefrac{1}{\tau})$$ iterations  (see Theorem~\ref{thm:Clip21-Avg}). As a corollary, if $\tau$  is sufficiently large, then \algname{Clip21-Avg} finds the exact average in a single iteration, which is to be expected from any reasonable mechanism. In particular, \algname{Clip21-Avg} maintains a collection of auxiliary iterates $v_k^1, \dots, v_k^n$ used to estimate the vectors $a^1,\dots,a^n$ which we  evolve via the rule
\begin{equation}\label{eq:Avg-Step3}\squeeze v_k^i = v_{k-1}^i + \clip_\tau (a^i - v_{k-1}^i), \quad i\in [n].\end{equation}
  The average of the vectors  is then estimated by $v_k\eqdef \frac{1}{n}\sum_{i=1}^n v_k^i \approx  \frac{1}{n}\sum_{i=1}^n a^i$.

\subsection{\algname{Clip21-GD}: Error feedback for \algname{GD} with clipping} 
Having solved the simpler task of average estimation, we now use similar ideas to design an error feedback mechanism for fixing the bias caused by clipping in an optimization setting for solving problem \eqref{eq:main}. In particular, we propose to estimate the average of the gradients $\nabla f_i(x^k), \dots, \nabla f_n(x^k)$ by the output of a {\em single} iteration of \algname{Clip21-Avg}, i.e., by $v_k  = \frac{1}{n} \sum_{i=1}^n v_k^i$,
where
\begin{equation}\label{eq:98g98fd}\squeeze v_k^i = v_{k-1}^i + \clip_\tau (\nabla f_i(x_k) - v_{k-1}^i), \quad i\in [n],\end{equation}
 and use this estimator in lieu of the true gradient to progress in our optimization task: 
$ x_{k+1} = x_k - \gamma v_k.$
This approach leads to the main method of this paper: \algname{Clip21-GD} (see Algorithm~\ref{alg:Clip21-GD}).  
The re-introduction of gradient dynamics into the picture causes considerable issues in that our analysis can not rely on the arguments used to analyze \algname{Clip21-Avg}. Indeed, while the vectors $a^1,\dots,a^n$ whose average we are trying to estimate in  \algname{Clip21-Avg} remain static throughout the iterations of \algname{Clip21-Avg}, in \algname{Clip21-GD} they change after every step. Analyzing the combined dynamics of two these methods turned out challenging, but ultimately possible, and the result is satisfying. In particular, in Theorem~\ref{thm:multinode-clippedEF21} we prove that  \algname{Clip21-GD} enjoys the same  rate as vanilla \algname{GD}; that is, our method outputs  a (random) point $\hat{x}_K$ such that $$\Exp{\norm{\nabla f(\hat{x}_K)}^2} \leq  \cO(\nicefrac{1}{K}).$$ Our main result can be found in Section~\ref{sec:optimization}; see also Table~\ref{tbl:many_methods}.
Next, the closest competing method to ours is that of
\citep{zhang2022understanding}, which 
studies the convergence of a clip-enabled variant of the \algname{FedAvg} algorithm, called \algname{CE-FedAvg}, under the  additional assumptions that there exists $G\geq 0$ and $\sigma_g \geq 0$ such that $\norm{\nabla f_i(x)}^2 \leq G^2$ and $\norm{\nabla f_i(x) - \nabla f(x)}^2 \leq \sigma_g^2$ for all $x\in \R^d$ and all $i\in [n]$.  Given the discussion from the beginning of Section~\ref{sec:contributions}, their work  can thus be seen  as embarking on approach ii) towards taming the divergence issues associated with applying gradient clipping in the $n>1$ case.   We wish to note that these additional assumptions are rather strong. First, it is unclear why gradient clipping is required in the regime when the gradients are already bounded. Second, the function similarity assumption typically does not hold in practice~\citep{localSGD-AISTATS2020,richtarik2021ef21}. Finally, none of these assumptions hold even for convex quadratics. Even with these additional and arguably strong assumptions, \citep{zhang2022understanding} in their Theorem~3.1 establish
a result of the form $$\squeeze \norm{\nabla f(\hat{x}_K)}^2 \leq  \cO(\nicefrac{1}{\sqrt{K}}),$$ which is weaker than the $\cO(\nicefrac{1}{K})$ rate we achieve.

\subsection{Extension 1: Adding noise for DP guarantees} 
We add bounded Gaussian noise to \algname{Clip21-GD}, which leads to the new method \algname{DP-Clip21-GD} (see Algorithm~\ref{alg:DP-Clip21-GD}),  and prove convergence with privacy guarantees for solving nonconvex problems under the PŁ condition. The results can be found in Appendix~\ref{sec:DP}.

\subsection{Extension 2: Adding communication compression for increased communication efficiency}
We extend \algname{Clip21-GD} to communication-efficient learning applications. We call this method \algname{Press-Clip21-GD} where each node compresses the clipped vector before it communicates. The description and convergence theorem of \algname{Press-Clip21-GD} are provided in Appendix~\ref{sec_app:ClipPress21}.

\subsection{Experiments}  
Our experiments on regression and deep learning problems suggest that \algname{Clip21-GD} and \algname{DP-Clip21-GD} substantially outperform \algname{Clip-GD} and its DP variants.  See Section~\ref{sec:exp}.

%%%%%%%%%%%%%%%%%%
%%%%%%%%%%%%%%%%%%
\section{Related Work} \label{sec:related_work}
%%%%%%%%%%%%%%%%%%
%%%%%%%%%%%%%%%%%%

As we shall outline next, our work is complementary to the existing literature on methods utilizing the clipping operator for solving various  problems.

\subsection{Relation to literature  on exploding gradients}
Gradient clipping and normalization were studied in early subgradient optimization literature  by  \citep{Shor1985} and \citep{Ermoliev1988}, among others, as techniques  for enforcing convergence when minimizing rapidly growing (i.e., non-Lipschitz)  functions. In contrast to this and also more recent  literature on this topic~\cite{pascanu2013difficulty, goodfellow2016deep,Zhang2020Why, mai2021stability}, we assume $L$-Lipschitzness of the gradient. This is because the issue we are trying to overcome in our work exists even in this more restrictive regime.  In other words, we are not employing clipping as a tool for taming the exploding gradients problem, and our work is fully complementary to this literature. 
Citing \citep{abadi2016deep}, ``gradient clipping of this form is a popular ingredient of \algname{SGD} for deep networks for non-privacy reasons, though in that setting it usually suffices to clip after averaging.'' Clipping after averaging does not cause the severe bias and divergence issues we are addressing in our work.

\subsection{Relation to literature on heavy-tailed noise}
In contrast to the literature on using clipping to tame  stochastic gradient estimators with a heavy-tailed behavior~\cite{nazin2019algorithms, zhang2020adaptive, gorbunov2020stochastic, gorbunov2021near, cutkosky2021high}, we do not consider the heavy-tailed setup in our work. In fact, our key methods and results are fully meaningful in the deterministic gradient regime, which is  why we focus on it in much of the paper.

\subsection{Relation to literature  on Byzantine robustness}
In our work we do {\em not} consider the Byzantine setup and instead focus on standard distributed optimization with nodes whose outputs can be trusted. However, there is a certain similarity between our \algname{Clip21-Avg} method and the centered clipping mechanism employed   by \citep{karimireddy2021learning} to obtain a Byzantine-robust estimator of the gradient. We shall comment on this in Appendix~\ref{subsec:Byz}. 

\subsection{Relation to literature  on error feedback}
Error feedback (\algname{EF}), originally proposed by \citep{seide20141}, is a popular mechanism for stabilizing optimization methods that use  compressed gradients to reduce communication costs. 
Variants of \algname{EF} methods were originally  analyzed by \citep{alistarh2018convergence,stich2018sparsified,wu2018error} and later refined by \citep{tang2019doublesqueeze,karimireddy2019error,stich2020error,qian2021error,khirirat2020compressed}. The current best results  can be found in \citep{gorbunov2020linearly,qian2021error}.
However, these methods were analyzed either in the single-node setting, or homogeneous data setting, or otherwise suffer from  restrictive assumptions (e.g., bounded gradient-norm and bounded data dissimilarity conditions) and not fully satisfying rates (e.g., $\cO(\nicefrac{1}{K^{2/3}})$ in the nonconvex regime).
To address these problems, a new error-feedback mechnism called \algname{EF21} was proposed by \citep{richtarik2021ef21}, and shown to provide fast $\cO(\nicefrac{1}{K})$ convergence for distributed optimization over smooth, heterogeneous objective functions \cite{richtarik2021ef21,EF21BW,richtarik20223pc}, under weak assumptions.
 Our algorithmic approach behind \algname{Clip21-GD}  is inspired by the  error feedback mechanism \algname{EF21} of \citep{richtarik2021ef21} proposed in the context  of distributed optimization with contractive communication compression, but needs a different theoretical approach due to a difference between the properties of the clipping and compression operators which necessitates a   substantially more refined and involved analysis. We shall comment on this in more detail in Section~\ref{sec:optimization}.

%%%%%%%%%%%%%%%%%%
%%%%%%%%%%%%%%%%%%
\section{Error Feedback for Average Estimation with Clipping} \label{sec:average}
%%%%%%%%%%%%%%%%%%
%%%%%%%%%%%%%%%%%%

\begin{algorithm}[t]
		\centering
		\caption{\algname{Clip21-Avg} (Error Feedback for Average Estimation with Clipping)}\label{alg:Clip21-Avg}
		\begin{algorithmic}[1]
			\STATE \textbf{Input:}  initial shifts $v^1_{-1},\dots,v^n_{-1}\in \R^d$; clipping threshold $\tau>0$
			\FOR{$k=0,1, 2, \dots, K-1 $}
		\FOR{each worker $i=1,\ldots,n$ in parallel}		
			\STATE $v^i_{k} = v^i_{k-1} + \clip_\tau( a^i - v^i_{k-1} )$		
		\ENDFOR			
			\STATE $v_k =  \frac{1}{n}\sum_{i=1}^n v^i_k$ {\tiny \hfill $\diamond$ estimate of $a\eqdef \frac{1}{n}\sum_i a^i$}
	\ENDFOR				
		\end{algorithmic}
	\end{algorithm}
	
We shall now describe the properties of our \algname{Clip21-Avg} method (Algorithm~\ref{alg:Clip21-Avg}) for finding the average of $n$ vectors, $a^1,\dots,a^n \in \R^d$.
	
\subsection{Basic properties of the clipping operator}
It is easy to verify that $\clip_\tau$ is the projection operator onto the ball $\cB(0,\tau)\eqdef \{x : \norm{x} \leq \tau\}$, and that is satisfies the properties %described 
in the next lemma.

\begin{lemma} \label{lem:clip} The clipping operator $\clip_\tau:\R^d \to \R^d$ has the following properties for all $\tau>0$: 
\begin{itemize}
\item[(i)] $ \clip_{\gamma \tau}(x) = \gamma  \clip_{\tau} \left( \nicefrac{x}{\gamma} \right)$ for all $x\in \R^d$ and $\gamma>0$, 
\item[(ii)] $\norm{ \clip_\tau(x) -x } = 0$ if $\norm{x}\leq \tau$, 
\item[(iii)] $\norm{ \clip_\tau(x) -x } = \norm{x} - \tau $ if $ \norm{x} \geq \tau$, 
\item[(iv)]
 $ \norm{\clip_\tau(x) - x}^2 = \left(1-\nicefrac{\tau}{\norm{x}}\right)^2 \norm{x}^2$ if  $\norm{x} \geq \tau$.
\end{itemize}

\end{lemma}

We will use parts (ii)-(iii) of this lemma in the rest of this section. Part (iv) will be useful in Section~\ref{sec:optimization}.

\subsection{Estimating $a^i$}
We now analyze Step 3 of Algorithm~\ref{alg:Clip21-Avg} (i.e.,  \eqref{eq:Avg-Step3}). As we shall see, $v^i_k$ converges to $a^i$ exactly  in a finite \# of steps.
\begin{lemma}  \label{lem:a^i} 
For all iterates $k\geq 0$ of Algorithm~\ref{alg:Clip21-Avg}, 
\begin{equation}\label{eq:nb9f8g0dfd} \squeeze\norm{v^i_k - a^i} \leq \max \left\{ 0, \norm{v^i_{-1} - a^i} - (k+1) \tau \right\}, \forall i.\end{equation}
In particular, if $\norm{v^i_{-1} - a^i}\leq \tau$, then $v^i_0 = a^i$. Otherwise, if  $k \geq  \left \lceil \frac{1}{\tau}\norm{v^i_{-1} - a^i} - 1 \right \rceil$, then $v^i_{k} = a^i$.
\end{lemma}

\subsection{Estimating $a\eqdef \frac{1}{n}\sum_{i=1}^n a^i$}
A convergence result for $v_k \to a$ follows by applying Lemma~\ref{lem:a^i} for all $i\in [n]$ and using convexity of the norm.

\begin{theorem} \label{thm:Clip21-Avg}
For all  iterates $k\geq 0$ of Algorithm~\ref{alg:Clip21-Avg},
 $$\squeeze \norm{ v_k - a } \leq   \frac{1}{n} \sum\limits_{i=1}^n \max \left \{ 0, \norm{v^i_{-1} - a^i} - (k+1) \tau  \right \}  .$$ In particular, if $\norm{v^i_{-1} - a^i}\leq \tau$ for all $i$, then $v_0 = a$. Otherwise, if  $k \geq \max_i \left \lceil \frac{1}{\tau}\norm{v^i_{-1} - a^i} - 1 \right \rceil$, then $v_{k} = a$. \vspace{-5pt}
\end{theorem}

%%%%%%%%%%%%%%%%%%
%%%%%%%%%%%%%%%%%%
\section{Error Feedback for Distributed Optimization with Clipping} \label{sec:optimization}
%%%%%%%%%%%%%%%%%%
%%%%%%%%%%%%%%%%%%

\begin{algorithm}[t]
		\centering
		\caption{\algname{Clip21-GD} (Error Feedback for  Distributed Optimization with Clipping)}\label{alg:Clip21-GD}
		\begin{algorithmic}[1]
		\STATE \textbf{Input:} initial iterate $x_{0} \in \R^d$; learning rate $\gamma>0$; initial gradient shifts $v^1_{-1},\dots,v^n_{-1}\in \R^d$; clipping threshold $\tau>0$
		\FOR{$k=0,1, 2, \dots, K-1 $}
		\STATE Broadcast $x_k$ to all workers 
		\FOR{each worker $i=1,\ldots,n$ in parallel}
		\STATE Compute 	$g_k^i = \clip_\tau( \nabla f_i(x_{k}) - v^i_{k-1} )$
		\STATE Update $v^i_{k} = v^i_{k-1} + g_k^i$
		\ENDFOR
		\STATE  $v_{k} = v_{k-1} + \frac{1}{n}\sum_{i=1}^n g_k^i$			
		\STATE  $x_{k+1} = x_k - \gamma  \frac{1}{n}\sum_{i=1}^n v^i_k$ 
		
		\ENDFOR
	\end{algorithmic}
	\end{algorithm}

The design of our new method \algname{Clip21-GD} (Algorithm \ref{alg:Clip21-GD}) is inspired by the current state-of-the-art error feedback mechanism  called \algname{EF21} developed by \citep{richtarik2021ef21} (see \citep{EF21BW,richtarik20223pc} for extensions) the goal of  which  is to progressively remove the error introduced by a {\em contractive compression operator} applied to the gradients\footnote{This is why use the number 21 in the name \algname{Clip21}.}. A contractive operator is a  possibly randomized mapping $\cC:\R^d\to \R^d$ satisfying %the property 
\begin{equation} \label{eq:contraction}\squeeze\Exp{\norm{\cC(x)-x}^2} \leq (1-\alpha)\norm{x}^2, \quad \forall x\in \R^d\end{equation} for some $0<\alpha \leq 1$. However, the results of \citep{richtarik2021ef21} do not apply to our setup since the clipping operator is not contractive in the sense of \eqref{eq:contraction}. Our idea is to instead rely on the related identity
\[\squeeze \norm{\clip_\tau(x) - x}^2 =  \begin{cases}0 & \text{if } \norm{x} \leq \tau\\
\left(1-\frac{\tau}{\norm{x}}\right)^2 \norm{x}^2 & \text{if }\norm{x} \geq \tau \end{cases}\] 
established in Lemma~\ref{lem:clip}. Using this identity in lieu of \eqref{eq:contraction} is more complicated since the contraction factor can be arbitrarily if   $\norm{x}$ is large. We needed to develop a new analysis technique to handle this situation.

In the rest of this section we describe the strong theoretical properties of \algname{Clip21-GD} (Algorithm \ref{alg:Clip21-GD}).
	
\subsection{Single-node regime ($n=1$)}\label{subsec:singlenode_clip21}

We begin by studying \algname{Clip21-GD} in the single-node  case. For simplicity, in the subsequent text and statements, we drop the superscript $i$ from all iterates. 

In the light of the discussion from Section~\ref{intro:subsec:n=1}, in this case our error feedback mechanism for clipping is {\em not} needed; that is, \algname{Clip-GD} suffices, and there is no need for \algname{Clip21-GD}. However, one would hope that our approach offers comparable guarantees in this case to those obtained by \citep{Zhang2020Why} in the $L$-smooth regime; i.e., we expect to obtain a $\cO(\nicefrac{1}{K})$ rate. The key purpose of this section is to see that this is indeed the case. 
However, we believe that \algname{Clip21-GD} {\em is} needed even in the $n=1$ case if one wants to obtain results in the more general constrained or proximal regime\footnote{We believe such results can be obtained by using the techniques developed by \citep{EF21BW}.}, and hence the results of this section can serve as a basis for further exploration and extensions.
  
Our first result establishes a descent lemma for a certain Lyapunov function. 
This is a substantial departure from existing analyses of clipping methods which do not make use of the control variate sequence $\{v_k\}$.

\begin{lemma}[Descent lemma]\label{lem:singlenode-clippedEF21}
Consider the problem of minimizing $f:\R^d \to \R$, assuming it has $L$-Lipschitz gradient and lower bounded by $f_{\rm inf}\in \R$.
Let $v_{-1}=0\in \R^d$, $\eta \eqdef \min\left\{1,\frac{\tau}{\| \nabla f(x_0)\|}\right\}$, $\FF\eqdef f(x_0)-f_{\rm inf}$, and $\GG\eqdef  | \norm{ \nabla f(x_0)}  - \tau  |$.
Then, single-node \algname{Clip21-GD} (described in Algorithm \ref{alg:Clip21-GD} with $n=1$) with stepsize
\begin{align}
\squeeze 
\gamma \leq \frac{1}{L}\min \left\{ \frac{1-1/\sqrt{2}}{1 + \sqrt{1 + 2\beta_1}} , \frac{\tau^2}{4L \left[ \sqrt{\FF} + \sqrt{\beta_2} \right]^2 } \right\},	 \label{eq:stepsize_single_node_clipped_EF21}
\end{align}
where $\beta_1 \eqdef  \frac{ (1-\eta)^2(1+2/\eta)}{1-(1-\eta)(1-\eta/2)}$ and $\beta_2 \eqdef  \FF + \frac{\tau\GG}{\sqrt{2\eta} L}$, satisfies 
\begin{align}
\squeeze	\phi_{k+1} \leq \phi_k - \frac{\gamma}{2}\norm{ \nabla f(x_k)}^2, \label{eq:descent_inequality_single_node}
\end{align}
where $\phi_k \eqdef f(x_k)-f_{\rm inf}+A \norm{ \nabla f(x_k) - v_k }^2$ and $A \eqdef \frac{\gamma}{2[1-(1-\eta)(1-\eta/2)]}$.
\end{lemma}

Inequality \eqref{eq:descent_inequality_single_node} states that the Lyapunov function $\phi_k$ decreases in each iteration by an amount proportional to the squared norm of the gradient, regardless of whether the clipping operator is active or not.
We next prove the state of the clipping operator when the algorithm is run. 
Our next result states that if the clipping operator is ``active'' at the start (i.e., $\| \nabla f(x_0)\|>\tau$), it will become ``inactive'' (i.e., it will act as the identity mapping) after at most $\cO(\nicefrac{\norm{\nabla f(x_0)}}{\tau})$ iterations, and then stay inactive from then on.

\begin{proposition}[Finite-time to no clipping]\label{prop:singlenode-clippedEF21}
Let the conditions of Lemma~\ref{lem:singlenode-clippedEF21} hold.
\begin{itemize}
\item [(i)] On the one hand, if $x_0$ satisfies $\| \nabla f(x_0) \| \leq \tau$, then  $\|\nabla f(x_k)-v_{k-1}\| \leq \tau$ for all $k \geq 0$. That is, the clipping operator is inactive for all iterations. 
\item [(ii)] 
On the other hand, if $x_0$ satisfies $\| \nabla f(x_0) \| > \tau$, then $\|\nabla f(x_k)-v_{k-1}\| \leq \tau$ for $k \geq k^\star \eqdef \frac{2}{\tau}\left( \| \nabla f(x_0) \|-\tau \right)+1$. That is, the clipping operator becomes inactive after at most $k^\star=\cO(\nicefrac{1}{\tau})$ iterations.
\end{itemize}
\end{proposition}

Note that when clipping becomes inactive,  then
$v_k = v_{k-1} + \clip_\tau( \nabla f(x_{k}) - v_{k-1} ) = \nabla f(x_{k})$, 
which means that \algname{Clip21-GD} turns into  \algname{GD} after at most $k^\star$ iterations. Finally, Lemma~\ref{lem:singlenode-clippedEF21} and Proposition~\ref{prop:singlenode-clippedEF21} lead to our main convergence result.

\begin{theorem}[Convergence result]\label{thm:singlenode-clippedEF21}
Consider  single-node \algname{Clip21-GD} (described in Algorithm \ref{alg:Clip21-GD} with $n=1$).  Let the conditions of Lemma~\ref{lem:singlenode-clippedEF21} hold and let $\hat x_K$ be a  point selected from the set $\{x_0,x_1,\ldots,x_{K-1}\}$ for $K\geq 1$ uniformly at random. Then
\begin{align*}
\squeeze 	\Exp{\| \nabla f(\hat x_K) \|^2} \leq \frac{2}{\gamma} \frac{\phi_0}{K}.
\end{align*}
If $\gamma$ is chosen to be the right-hand side of \eqref{eq:stepsize_single_node_clipped_EF21}, then for $C\eqdef \max \left\{ L\FF, \norm{ \nabla f(x_0) }^2 \right\}$,
\begin{align*}
\squeeze 	\Exp{\| \nabla f(\hat x_K) \|^2} = \cO\left(\left[ \frac{C(\norm{\nabla f(x_0)} + \tau)}{\tau} + \frac{L^2(\FF)^2}{\tau^2}\right]\frac{1}{K}\right).
\end{align*}
\end{theorem}
Theorem \ref{thm:singlenode-clippedEF21} states that in the  $L$-smooth non-convex  regime,    single-node \algname{Clip21-GD} enjoys the $\mathcal{O}(\nicefrac{1}{K})$ rate. Up to constant factors, this is the same rate as that of \algname{GD}.

\subsection{Multi-node regime ($n>1$)}

Next, we turn our attention to  multi-node  \algname{Clip21-GD} as described in Algorithm \ref{alg:Clip21-GD}.
Note that this method becomes \algname{EF21} when we replace $\clip_\tau(\cdot)$ with  a contractive compressor, and becomes \algname{GD} when we let $ \tau \to +\infty$.
Our results for  multi-node  \algname{Clip21-GD}  have the same meaning as those for the single-node  case, and hence a  commentary comparing these results to the $n=1$ case should suffice.

\begin{lemma}[Descent lemma]\label{lem:multinode-clippedEF21}
Consider multi-node  \algname{Clip21-GD} (described in Algorithm \ref{alg:Clip21-GD} for general $n$) for solving 
\eqref{eq:main}.	Suppose that each $ f_i(x)$ has $L_i$-Lipschitz  gradient and that $f$ has $L$-Lipschitz  gradient and lower bounded by $f_{\rm inf} \in\R$.
	Let $v_{-1}^i=0$ for all $i$, $\eta\eqdef\min\left\{1,\frac{\tau}{\max_i\norm{   \nabla f_i(x_0) }}\right\}$, $\FF\eqdef f(x_0)-f_{\rm inf}$, and $\GG^2\eqdef \frac{1}{n}\sum_{i=1}^n (\norm{ \nabla f_i(x_0)  } - \tau)^2$,
	with stepsize
	\begin{align} \label{eq:stepsize_multi_node_clipped_EF21}
		\squeeze 	\gamma \leq \min\left\{ \frac{\phi_0}{ (B-\tau)^2  } ,  \frac{\nicefrac{(1-1/\sqrt{2})}{L}}{1+\sqrt{1+2\beta_1}} , \frac{\nicefrac{\tau^2}{L_{\max}^2}}{16 \left[ \sqrt{\FF} + \sqrt{\beta_2} \right]^2} \right\},	
	\end{align}
	where $\beta_1 \eqdef 2\frac{ (1-\eta)^2(1+2/\eta)}{1-(1-\eta)(1-\eta/2)} \left(\nicefrac{L_{\max}}{L} \right)^2$ and $\beta_2\eqdef \FF + \frac{\GG \tau}{2\sqrt{2\eta} L_{\max}}$. Then,
	\begin{align} \label{eq:descentIneq_multi_node_clipped_EF21}
		\squeeze 	\phi_{k+1} \leq \phi_k - \frac{\gamma}{2}\norm{\nabla f(x_k)}^2, 
	\end{align}
	where $\squeeze \phi_k \eqdef f(x_k)-f_{\rm inf}+\frac{A}{n}\sum\limits_{i=1}^n\norm{ \nabla f_i(x_k) - v_k^i  }^2$ and $A\eqdef \frac{\gamma}{2[1-(1-\eta)(1-\eta/2)]}$.
\end{lemma}

\begin{proposition}[Finite-time to no clipping]\label{prop:multinode-clippedEF21}
Let the conditions of Lemma~\ref{lem:multinode-clippedEF21} hold.
On the one hand, if $x_0$ satisfies $\| \nabla f_i(x_0) \| \leq \tau$, then  $\|\nabla f_i(x_k)-v^i_{k-1}\| \leq \tau$ for all $k \geq 0$. 
On the other hand, if $x_0$ satisfies $\| \nabla f_i(x_0) \| > \tau$, then $\|\nabla f_i(x_k)-v^i_{k-1}\| \leq \tau$ for $k \geq k^\star \eqdef \frac{2}{\tau}\left( \| \nabla f_i(x_0) \|-\tau \right)+1$.
\end{proposition}

\begin{theorem}[Convergence result]\label{thm:multinode-clippedEF21}
Let the conditions of Lemma~\ref{lem:multinode-clippedEF21} hold, and $\hat x_K$ be a point selected from the set $\{x_0,x_1,\ldots,x_K\}$ uniformly at random for $K\geq 1$. 
Then, multi-node \algname{Clip21-GD} (described in Algorithm~\ref{alg:Clip21-GD}) satisfies
\begin{align*}
\squeeze 	\Exp{\norm{ \nabla f(\hat x_K) }^2} \leq \frac{2}{\gamma} \frac{\phi_0}{K}.
\end{align*}
If $\gamma$ is chosen to be the right-hand side of \eqref{eq:stepsize_multi_node_clipped_EF21}, then 
\begin{align*}
\squeeze 	\Exp{ \norm{ \nabla f(\hat x_K) }^2} = \cO\left( \left[\left(1 + \frac{C_1}{\tau}\right) C_2+ \frac{L_{\max}^2(\FF)^2}{\tau^2}\right] \frac{1}{K}\right),
\end{align*}
where $C_1\eqdef \max_i \norm{\nabla f_i(x_0)}$, $C_2 \eqdef \max\{\FF \max(L,L_{\max}),C_1^2\}$.
\end{theorem}

Theorem~\ref{thm:multinode-clippedEF21} says that \algname{Clip21-GD} enjoys an $\cO(\nicefrac{1}{K})$ rate, which is faster than the previous state-of-the-art rate $\cO(\nicefrac{1}{\sqrt{K}})$ rate obtained by  \citep{zhang2022understanding,liu2022communication} on nonconvex problems. Also, we do not require bounded gradient and function similarity assumptions as they do (see Table~\ref{tbl:many_methods}).
Furthermore, Proposition \ref{prop:multinode-clippedEF21} says that if the clipping operator at each node  is ``active" at the start (i.e., $\norm{ \nabla f_i(x_0) } > \tau$), it will become ``inactive'' in at most $k^\star$  steps (i.e., $\norm{ \nabla f_i(x_k) - v^i_{k-1} } \leq \tau$, and \algname{Clip21-GD} will effectively become \algname{GD}. Moreover, when specializing our multi-node theory for \algname{Clip21-GD} to the $n=1$ case, and compare this to the theory from Section~\ref{subsec:singlenode_clip21}, we pay the price of a smaller maximum stepsize by a factor of $\approx 10$.  We comment more on this in Appendix~\ref{sec:9}.

\subsection{Adding DP noise}
To train the model under the privacy budget, we further add bounded Gaussian noise $z^i_{k-1}$ to clipped gradients $\clip_\tau(\nabla f_i(x_k)-v_{k-1}^i)$ before it is transmitted. 
This modified \algname{Clip21-GD} method, called \algname{DP-Clip21-GD}, achieves the linear rate with residual error due to the DP-noise for nonconvex problems under the PŁ condition. The results can be found in Appendix~\ref{sec:DP}.

\subsection{Adding communication compression}
In order to reduce the communication cost,  we further modify \algname{Clip21-GD} by replacing $\clip_\tau(\nabla f_i(x_k)-v_{k-1}^i)$ with $\cC(\clip_\tau(\nabla f_i(x_k)-v_{k-1}^i))$, where $\cC:\R^d \to \R^d$ is a contractive compressor. This method, which we call \algname{Press-Clip21-GD},   is shown to enjoy the  $\mathcal{O}(\nicefrac{1}{K})$ rate as well. The method and theory are relegated to  Appendix~\ref{sec_app:ClipPress21}.

%%%%%%%%%%%%%%%%%%%%%%%%%%%%%%%%%%%%%%%%%%%%%%%%%%%%%%%%%%%%%%%%%%%%%%%%%%%%%%%
%%%%%%%%%%%%%%%%%%%%%%%%%%%%%%%%%%%%%%%%%%%%%%%%%%%%%%%%%%%%%%%%%%%%%%%%%%%%%%%
\section{Experiments}\label{sec:exp}
%%%%%%%%%%%%%%%%%%%%%%%%%%%%%%%%%%%%%%%%%%%%%%%%%%%%%%%%%%%%%%%%%%%%%%%%%%%%%%%
%%%%%%%%%%%%%%%%%%%%%%%%%%%%%%%%%%%%%%%%%%%%%%%%%%%%%%%%%%%%%%%%%%%%%%%%%%%%%%%

\begin{figure}[t]
	\begin{center}
		\begin{tabular}{ccc}
			\multicolumn{3}{c}{\includegraphics[width=0.6\linewidth]{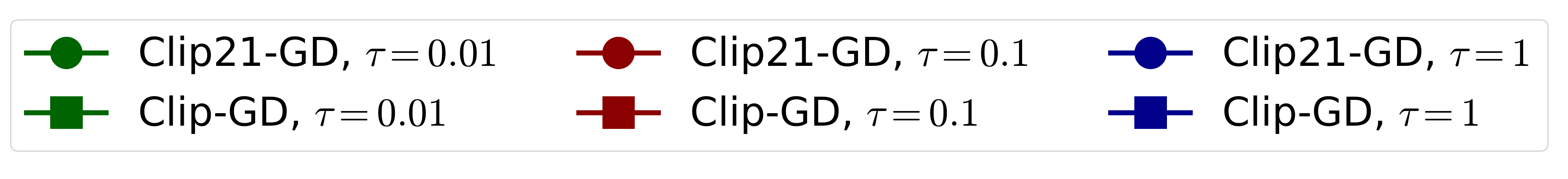}}\\
			\vspace{-5pt}
			\includegraphics[width=0.34\linewidth]{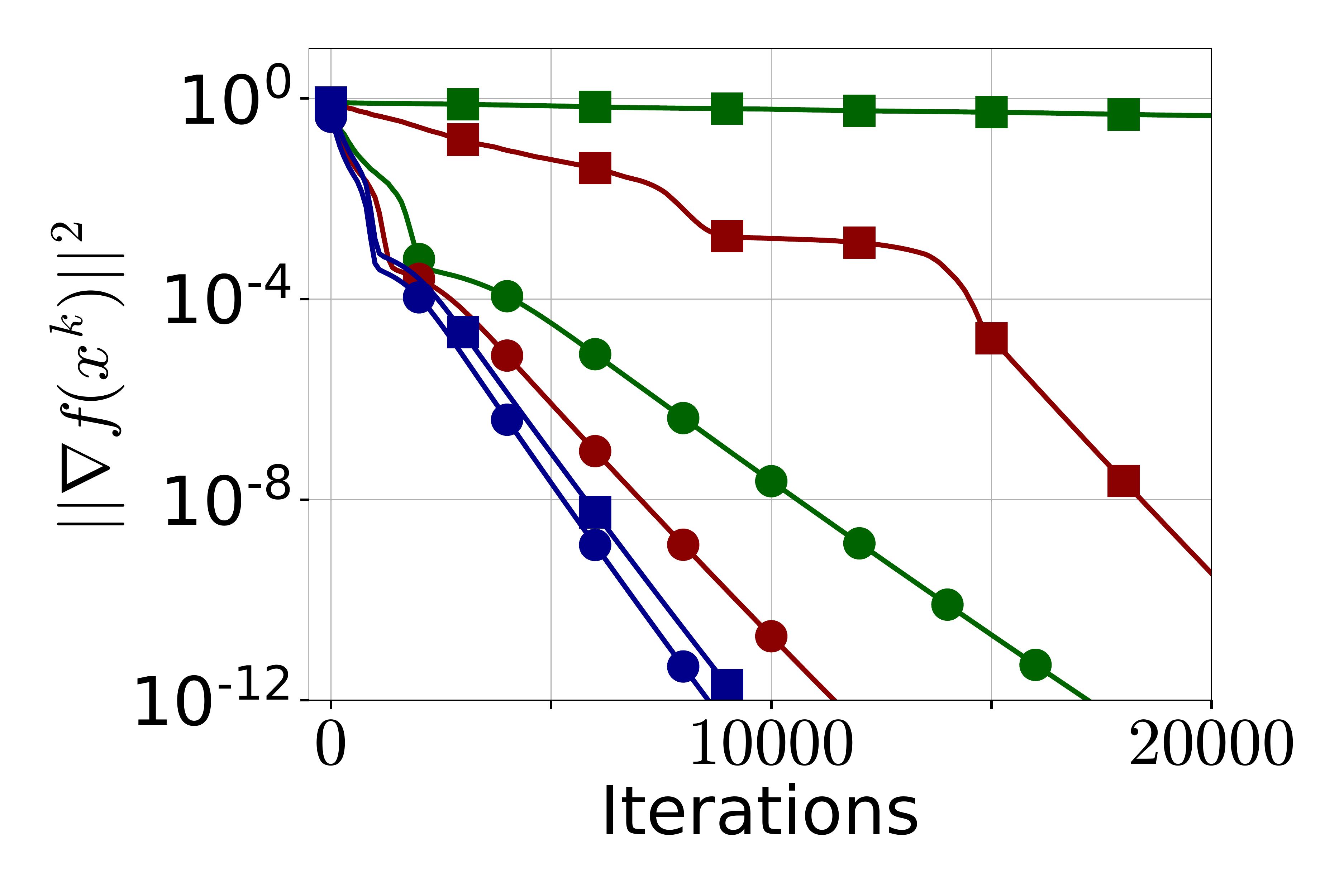} &
           \hspace{-15pt}\includegraphics[width=0.34\linewidth]{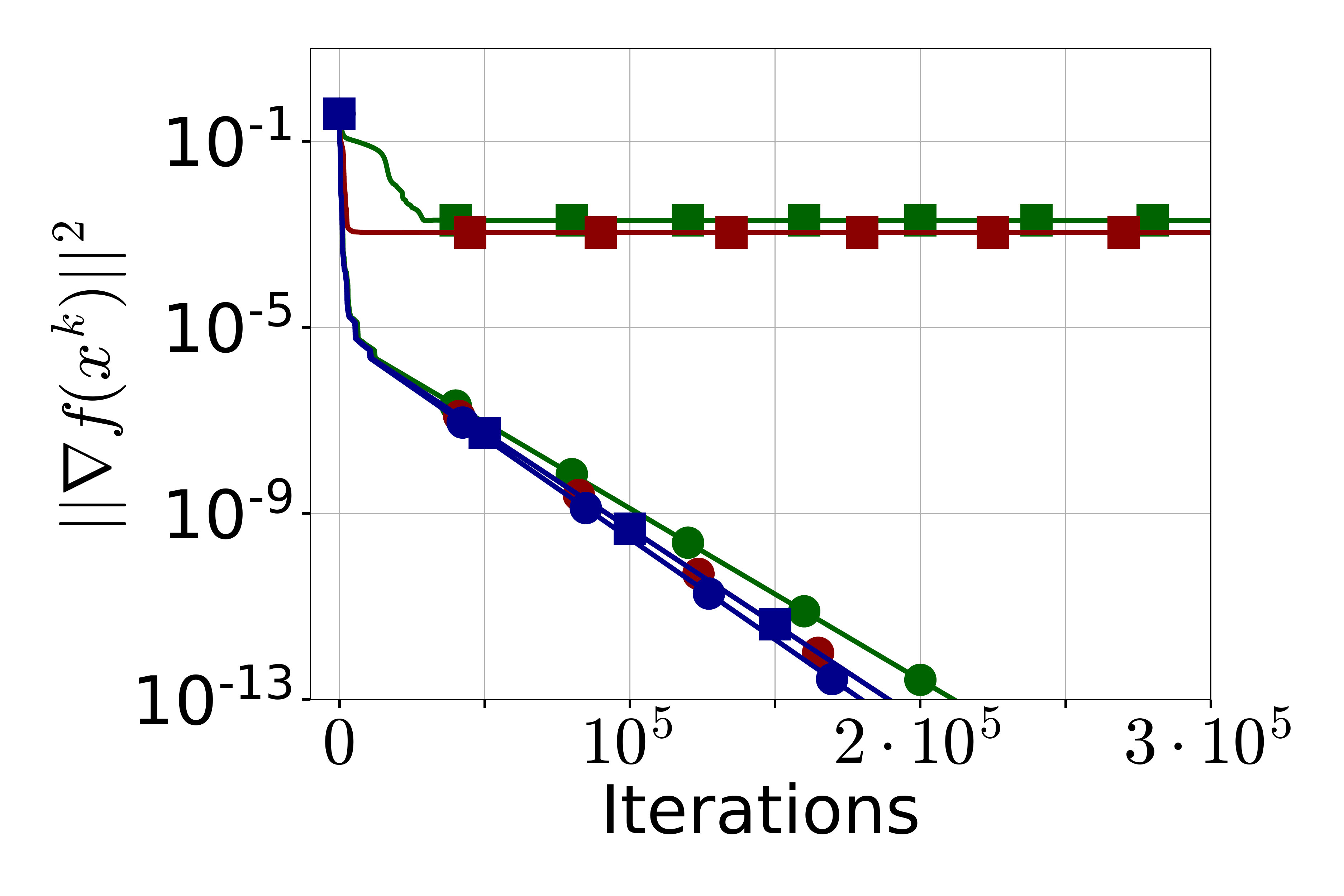}
		&	\hspace{-15pt}\includegraphics[width=0.34\linewidth]{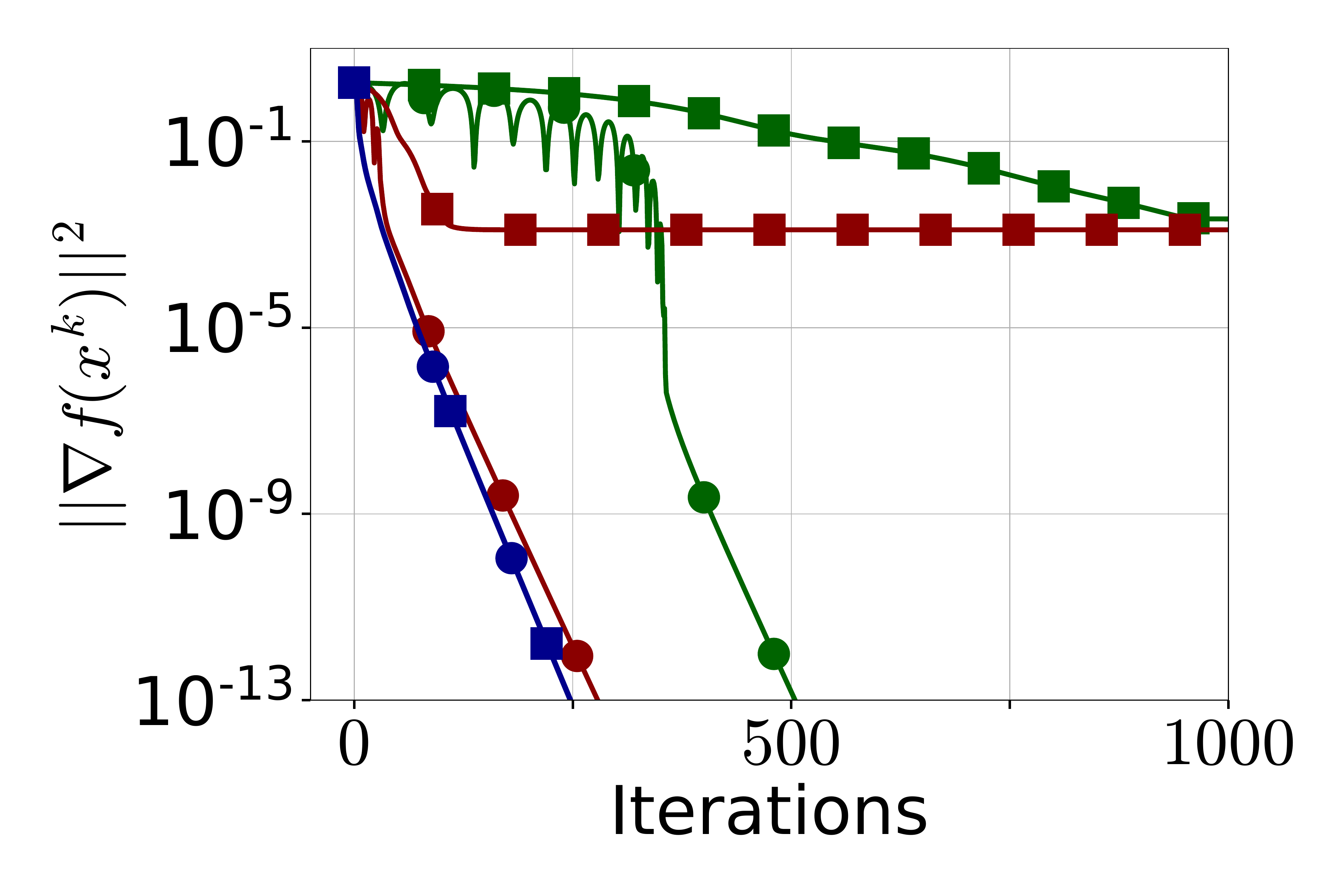} 		\\
  	{\small (a) madelon} &
			{\small (b) w7a} & 
   {\small (c) w8a} \vspace{-5pt}       
		\end{tabular}
	\end{center}
	\caption{Comparison of \tblalgname{Clip21-GD} vs \tblalgname{Clip-GD} with clipping threshold $\tau\in\{0.01,0.1,1\}$ on  logistic regression with  $\ell_2$-regularizer ({\bf first line}) and a nonconvex regularizer ({\bf second line}).}
	\label{fig:logreg_pure}
\end{figure}

To demonstrate strong performance of  \algname{Clip21-GD} and \algname{DP-Clip21-GD} over traditional clipped gradient methods, we evaluate all the methods on the  logistic regression problem,
i.e. the problem of minimizing 
$$f(x) \eqdef \frac{1}{n} \sum \limits_{i=1}^n f_i(x) + \lambda r(x),$$ 
where $$f_i(x) \eqdef \frac{1}{m} \sum\limits_{j=1}^m \log(1+e^{-b_{ij}a_{ij}^\top x}),$$ $a_{ij}\in \R^d$ is the $j^{\text{th}}$ training sample associated with class label $b_{ij}\in\{-1,1\}$, which is privately known by node $i$. 
Additional experiments on nonconvex linear regression and deep neural network training are deferred to Section~\ref{sec_app:exp}. 
We use datasets from the {\sf \small LibSVM} library \citep{chang2011libsvm}, and two types of regularization: 
\begin{itemize}
\item[(1)] $r(x) = \frac{1}{2}\|x\|^2$, which is the $\ell_2$-regularizer, and \item [(2)] $r(x) = \sum_{j=1}^d {x_j^2}/{(1+x_j^2)}$, which is a nonconvex regularizer.
\end{itemize}
Before running the algorithms, 
we preprocess each dataset as follows: we $(i)$  sort  the training samples according to the labels; $(ii)$  split the dataset into $n$ equal parts among the nodes; and $(iii)$ normalize each sample of each part by using  {\sf \small  StandardScaler} from the {\sf \small scikit-learn} library \citep{scikit-learn}.
By this preprocessing, the problem becomes more heterogeneous while the Lipschitz constants of functions $\nabla f_i$  are closer to each other. We reported the best convergence performance of the baseline \algname{Clip-GD} and 
\algname{Clip21-GD} by choosing  stepsizes $\gamma\in\{\nicefrac{1}{4L}, \nicefrac{1}{2L}, \dots, \nicefrac{8}{L}\}$.
We also set $n=10$; $\lambda=10^{-1}$ and $\lambda=10^{-4}$ for nonconvex and $\ell_2$ regularizers respectively.

\subsection{Performance of \algname{Clip21-GD} and \algname{Clip-GD}}
Figure~\ref{fig:logreg_pure} shows that   \algname{Clip21-GD} outperforms \algname{Clip-GD} in both the convergence speed and solution accuracy.
This happens because when the clipping operator is turned on at the beginning, it will be turned off in \algname{Clip-GD} with required iteration counts $k$ much larger than in \algname{Clip21-GD} for small values of $\tau$.
For $\tau\in \{0.01, 0.1\}$, \algname{Clip-GD} converges towards the neighborhood while \algname{Clip21-GD} always converges towards the stationary point. 
At iteration $k=10^4$  and for $\tau=0.01$, \algname{Clip21-GD} achieves roughly $6$ times more accurate solution than \algname{Clip-GD} for the madelon and w7a datasets. 

\subsection{Performance of \algname{DP-Clip21-GD} and \algname{DP-Clip-GD}}
We next showcase that \algname{DP-Clip21-GD} also outperforms \algname{DP-Clip-GD} for training over the privacy budget.  
In these experiments,  we set $\tau=0.1$ and   noise $\sigma\in \{0.01,0.05,0.1\}$.
Looking at Figure~\ref{fig:logreg_dp}, we see that \algname{DP-Clip21-GD} outperforms \algname{DP-Clip-GD} in solution accuracy: \algname{DP-Clip21-GD} converges towards the higher accurate solution as the noise variance $\sigma$ decreases while  \algname{DP-Clip-GD} converges towards the neighbourhood regardless of any values of $\sigma$. 
At $k > 10^4$ and $\sigma=0.01$, 
 solution accuracy from \algname{DP-Clip21-GD} is higher than \algname{DP-Clip-GD} by an order of magnitude for every benchmarked dataset.  
This happens because \algname{DP-Clip21-GD} can handle problem heterogeneity, as suggested by our theory. 

\begin{figure}[t]
	\begin{center}
		\begin{tabular}{ccc}
			\multicolumn{3}{c} {\includegraphics[width=0.7\linewidth]{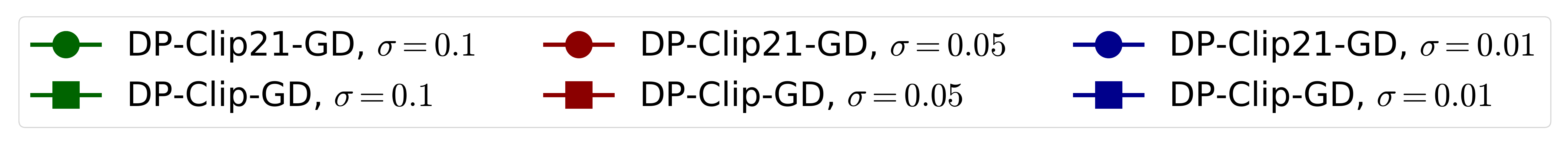}}\\
			\vspace{-5pt}
	\includegraphics[width=0.34\linewidth]{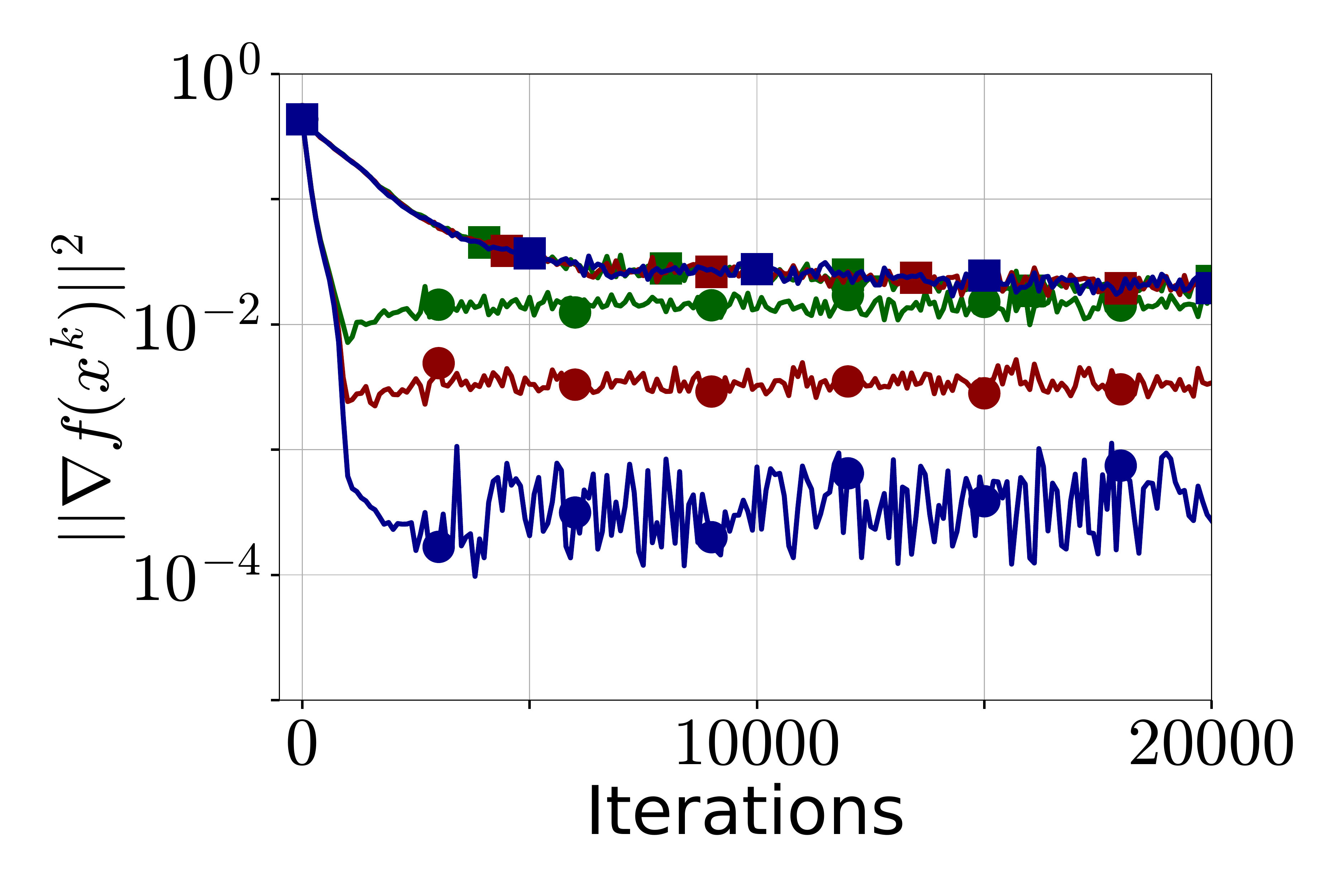} &
			\hspace{-15pt}\includegraphics[width=0.34\linewidth]{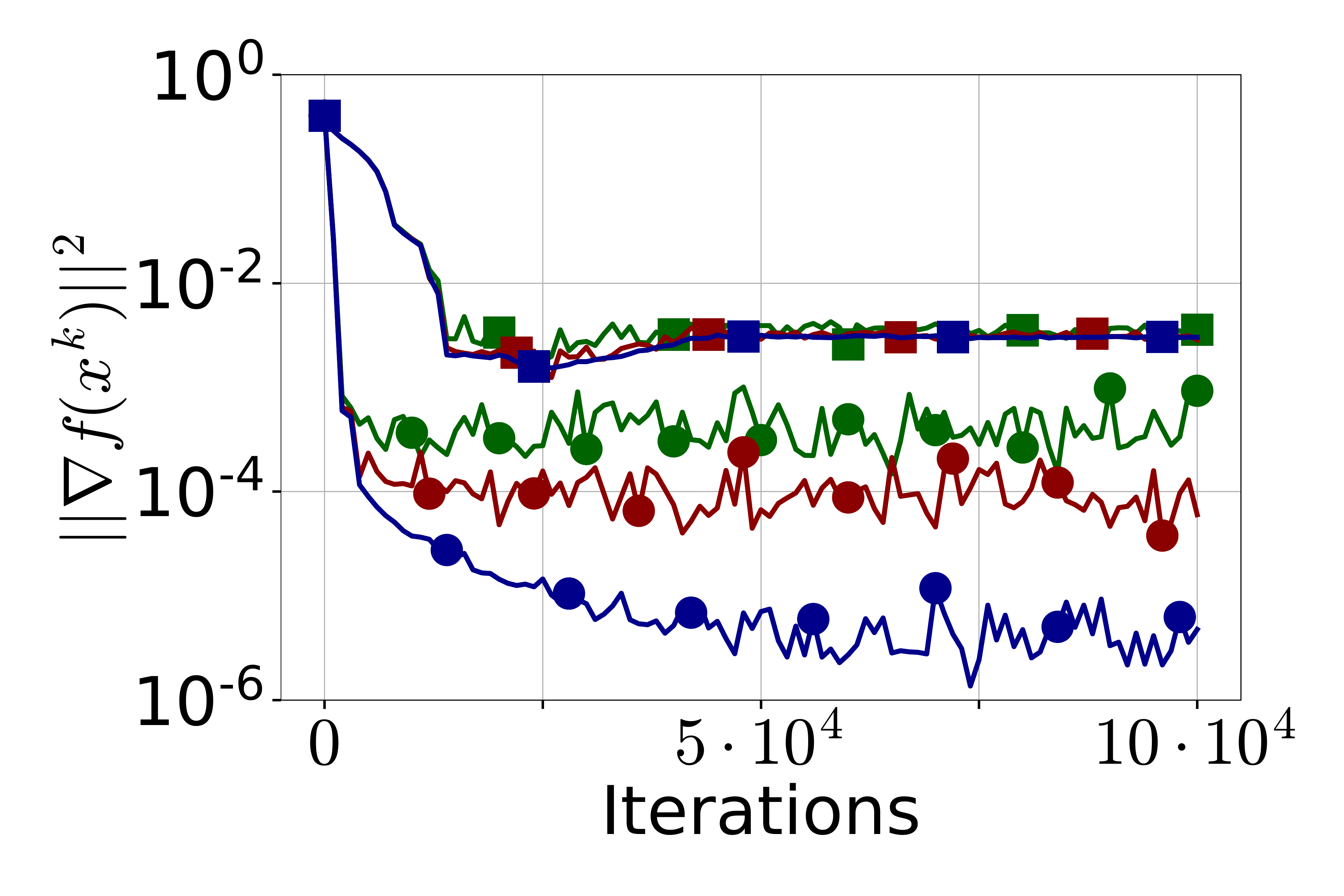}
		& 	\hspace{-15pt}\includegraphics[width=0.34\linewidth]{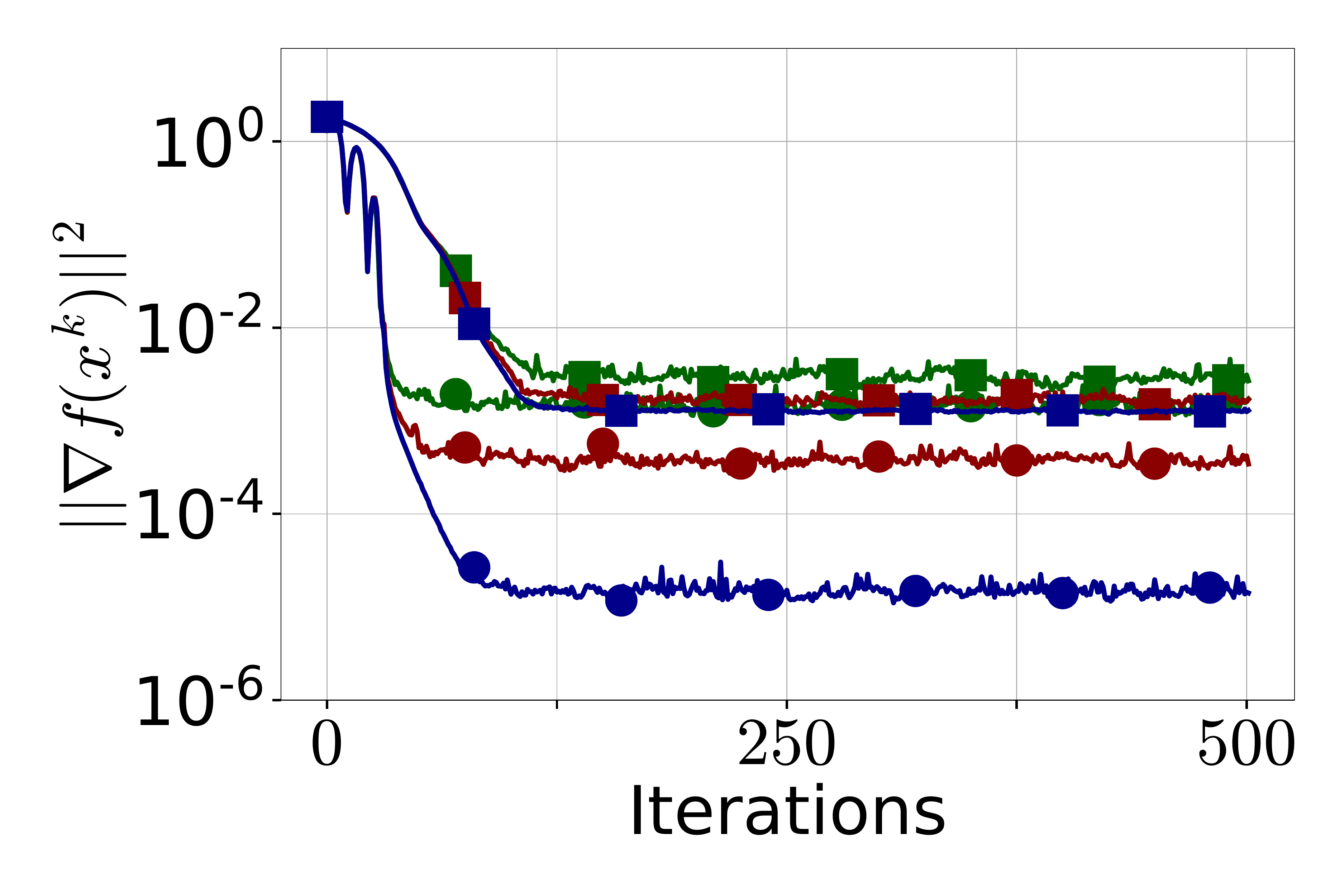} \\ 
  {\small (a) madelon} &
		{\small (b) phishing} &
  {\small (c) w8a}
		\end{tabular}
	\end{center}
	\caption{Comparison of \algname{DP-Clip21-GD} and \algname{DP-Clip-GD} with $\tau=0.1$ and $\sigma\in\{0.01,0.05,0.1\}$ on logistic regression with $\ell_2$-regularizer ({\bf first line}) and a nonconvex regularizer ({\bf second line}).}
	\label{fig:logreg_dp}
\end{figure}

\section{Conclusions, Limitations and Extensions}

We proposed \algname{Clip21-GD} -- an  error feedback mechanism for dealing with
the bias 
by  gradient clipping. 
We proved that \algname{Clip21-GD} enjoys the $\mathcal{O}(\nicefrac{1}{K})$ convergence  for nonconvex problems in single-node and multi-node settings. 
We also prove that its DP variant called \algname{DP-Clip21-GD} attains  privacy and utility guarantee for nonconvex functions that satisfy the PŁ condition.
Our numerical experiments indicate that  \algname{Clip21-GD} and  \algname{DP-Clip21-GD} attains faster convergence speed and higher solution accuracy than \algname{Clip-GD} and  \algname{DP-Clip-GD}, respectively. 
We plan to extend our theory for  \algname{Clip21-GD} to stochastic optimization as it works well in our experiments on training deep neural network models.

\bibliographystyle{unsrt}
\bibliography{refs.bib}

%%%%%%%%%%%%%%%%%%%%%%%%%%%%%%%%%%%%%%%%%%%%%%%%%%%%%%%%%%%%%%%%%%%%%%%%%%%%%%%
%%%%%%%%%%%%%%%%%%%%%%%%%%%%%%%%%%%%%%%%%%%%%%%%%%%%%%%%%%%%%%%%%%%%%%%%%%%%%%%
% APPENDIX
%%%%%%%%%%%%%%%%%%%%%%%%%%%%%%%%%%%%%%%%%%%%%%%%%%%%%%%%%%%%%%%%%%%%%%%%%%%%%%%
%%%%%%%%%%%%%%%%%%%%%%%%%%%%%%%%%%%%%%%%%%%%%%%%%%%%%%%%%%%%%%%%%%%%%%%%%%%%%%%
\newpage
\appendix
\onecolumn

%%%%%%%%%%%%%%%%%%%%%%%%%%%%%%%%%%%%%%%%%%%%%%%%%%%%%%%%%%%%%%%%%%%%%%%%%%%%%%%
% ADDITIONAL EXPERIMENTS
%%%%%%%%%%%%%%%%%%%%%%%%%%%%%%%%%%%%%%%%%%%%%%%%%%%%%%%%%%%%%%%%%%%%%%%%%%%%%%%

\part*{Appendix}

\section{Relation to literature on Byzantine robustness}\label{subsec:Byz}

Algorithm~\ref{alg:Clip21-Avg} is similar to the {\em centered clipping} subroutine used by \citep{karimireddy2021learning} to obtain a Byzantine-robust estimator of the gradient.  In their setting, the clients are partitioned into two groups, regular (majority) and Byzantine (minority), and the goal is to minimize the average of the functions owned by the regular clients without a-priori knowing which clients are regular. The Byzantine clients are allowed to report any vectors in an adversarial fashion in an attempt to induce bias into gradient estimation. In this application, it is assumed that  the regular workers share the same function as this ensures that there is enough ``signal'' for the optimization method to iteratively find out which clients are regular. In contrast, we allow all functions $f_i$ to be arbitrarily heterogeneous.  
While \citep{karimireddy2021learning}  use a single shared center/shift for all clients $i \in [n]$ the purpose of which is to ``learn'' who the regular (i.e., non-Byzantine) clients  are via tracking the (homogeneous) gradient of these regular clients, we use $n$ different centers/shifts ($v_k^1,\dots,v_k^n$) designed to track and ultimately find all the original vectors $a^1,\dots,a^n$, respectively, which can be arbitrarily different. Due to the different goal they have, their analysis is completely different to ours.

\section{Additional Experiments}\label{sec_app:exp}

We now include several additional experimental results.

\subsection{Nonconvex linear regression}

We run all clipped methods and their DP versions to solve the linear regression problem with the nonconvex regularization on the form: 
\begin{equation*}
	\min_{x\in \R^d} \left[f(x) = \frac{1}{n}\sum\limits_{i=1}^n f_i(x)\right],
\end{equation*}
where each local loss function is 
\begin{equation*}
	f_i(x) = \frac{1}{m}\sum\limits_{j=1}^m (a_{ij}^\top x - b_{ij})^2 + \lambda \sum\limits_{j=1}^d \frac{x^2_j}{1+x_j^2}.
\end{equation*}
Here, $\lambda>0$ is the nonconvex regularization parameter, and $(a_{i1},b_{i1}),\ldots,(a_{im},b_{im})$ where $a_{ij}\in\mathbb{R}^d$ and $b_{ij}\in\{-1,1\}$ are $m$ training data samples available for node $i$. 
We use datasets from LibSVM library \citep{chang2011libsvm}, and perform the preprocessing steps described in Section \ref{sec:exp}. These preprocessing steps allow some workers to have data samples with only one class label, which make the problem more heterogeneous. 
Also, by the normalization step of this preprocessing Lipschitz constants of local loss functions $f_i(x)$ become close to each other. We also used the same set of parameters (i.e. $n,\lambda,\tau,\sigma$) for the linear regression problem as the logistic regression problem.

From Figure~\ref{fig:linreg_dp_and_nondp} (a)-(b), the convergence speed of \algname{Clip21-GD} is faster than or the same as \algname{Clip-GD}. 
For instance, \algname{Clip21-GD} converges faster than \algname{Clip-GD} when $\tau=0.1$ while both methods have the same convergence performance when $\tau=1$.
This happens because when the clipping operator is turned on at the beginning, it will be turned off in \algname{Clip-GD} with required iteration counts $k$ much larger than in \algname{Clip21-GD} for small values of $\tau$ (e.g. for $\tau=0.1$, \algname{Clip-GD} does not converge at all).

We also reported the training performance of \algname{DP-Clip21-GD} and \algname{DP-Clip-GD}  under the DP budget in Figure~\ref{fig:linreg_dp_and_nondp} (c)-(d). 
In particular, \algname{DP-Clip21-GD} converges towards the higher accurate solution as the noise level decreases, while  \algname{DP-Clip-GD} converges towards the neighbourhood regardless of the noise level. 
This is because \algname{DP-Clip21-GD}, unlike \algname{DP-Clip-GD}, can handle the problem heterogeneity, as suggested by our theory.

\begin{figure}[ht]
	\begin{center}
		\begin{tabular}{cccc}
			\multicolumn{2}{c}{\includegraphics[width=0.38\linewidth]{legend.pdf}}&
			\multicolumn{2}{c}{\includegraphics[width=0.45\linewidth]{legend_dp.pdf}}\\
			\includegraphics[width=0.22\linewidth]{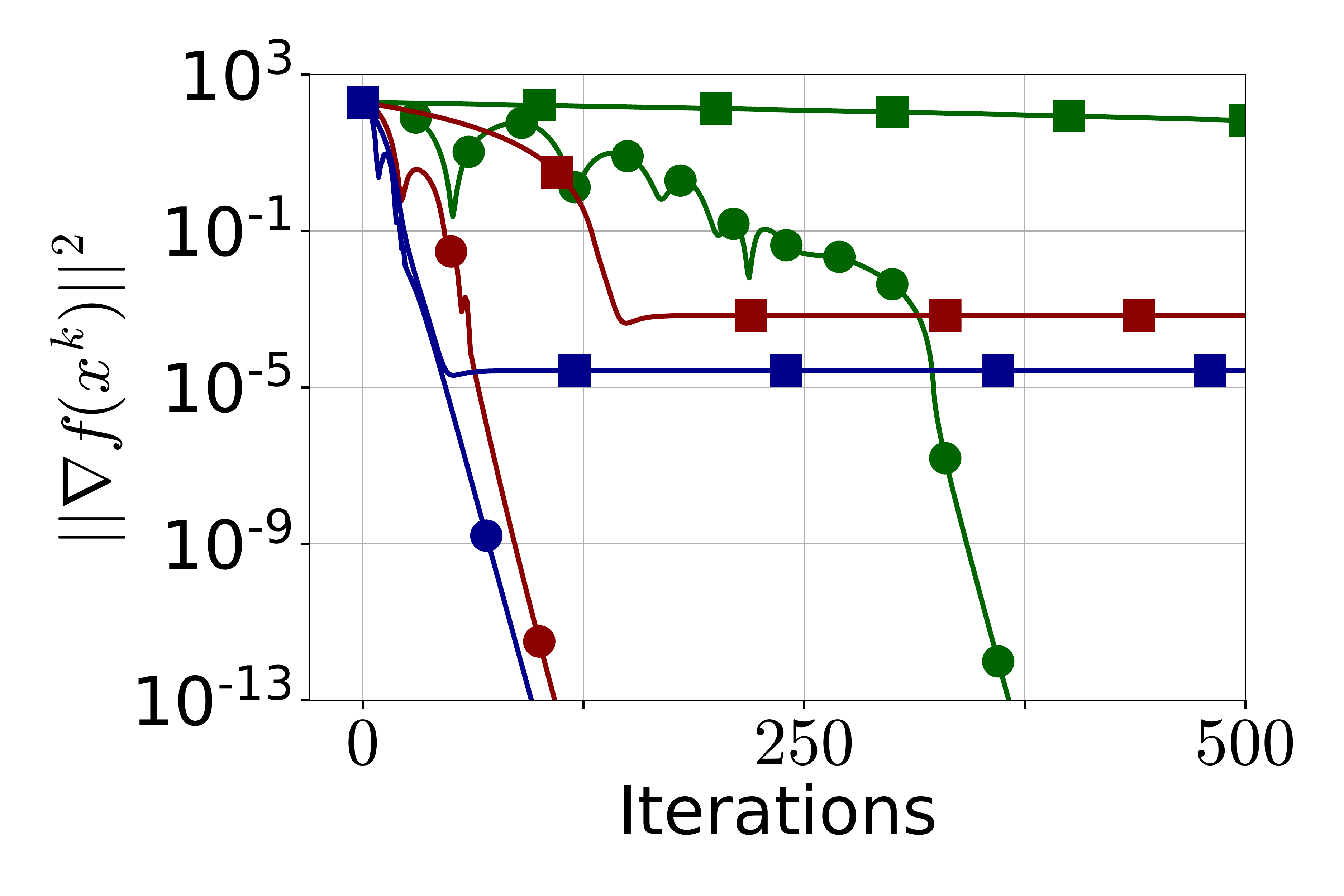} &
			\includegraphics[width=0.22\linewidth]{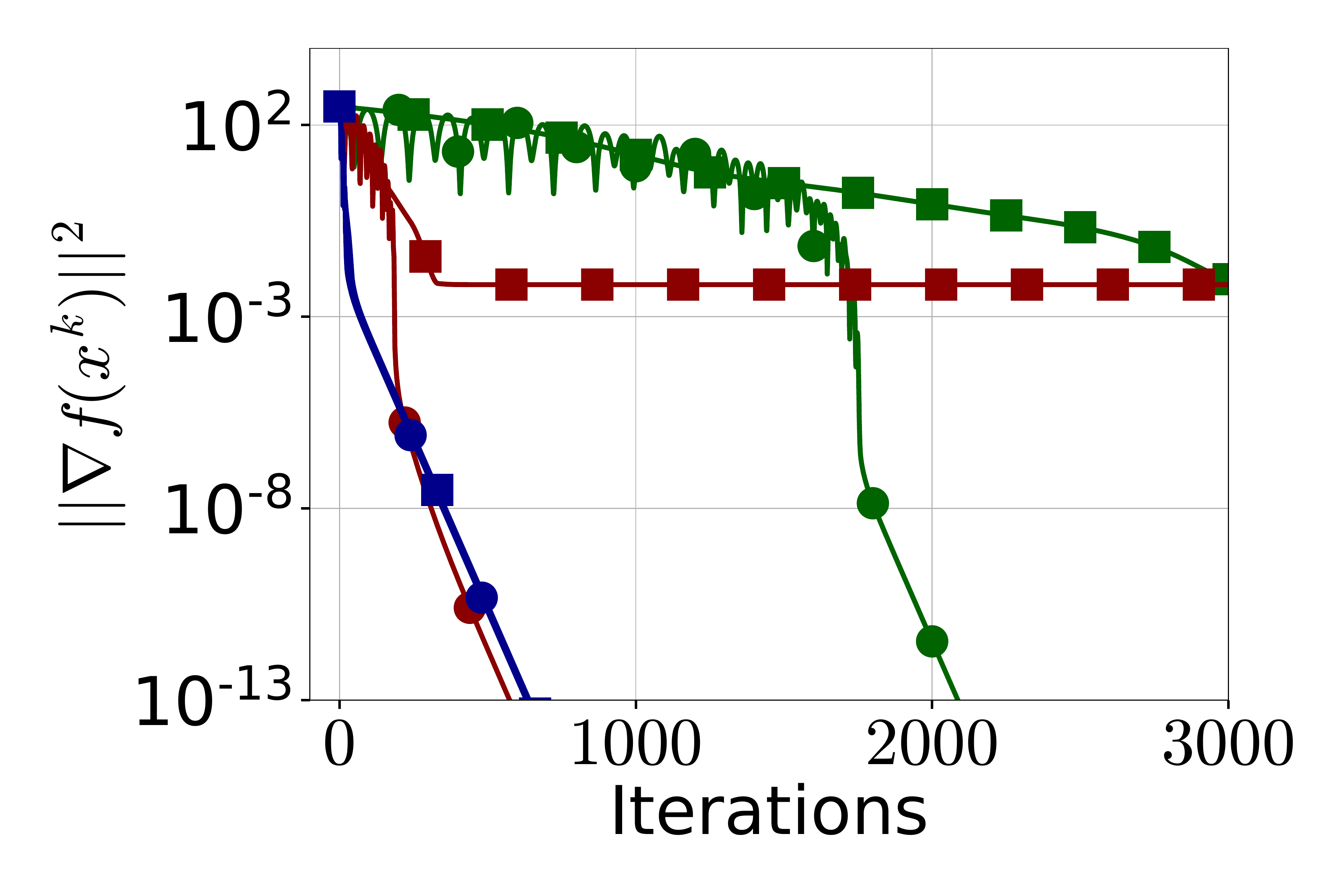} &
			\includegraphics[width=0.22\linewidth]{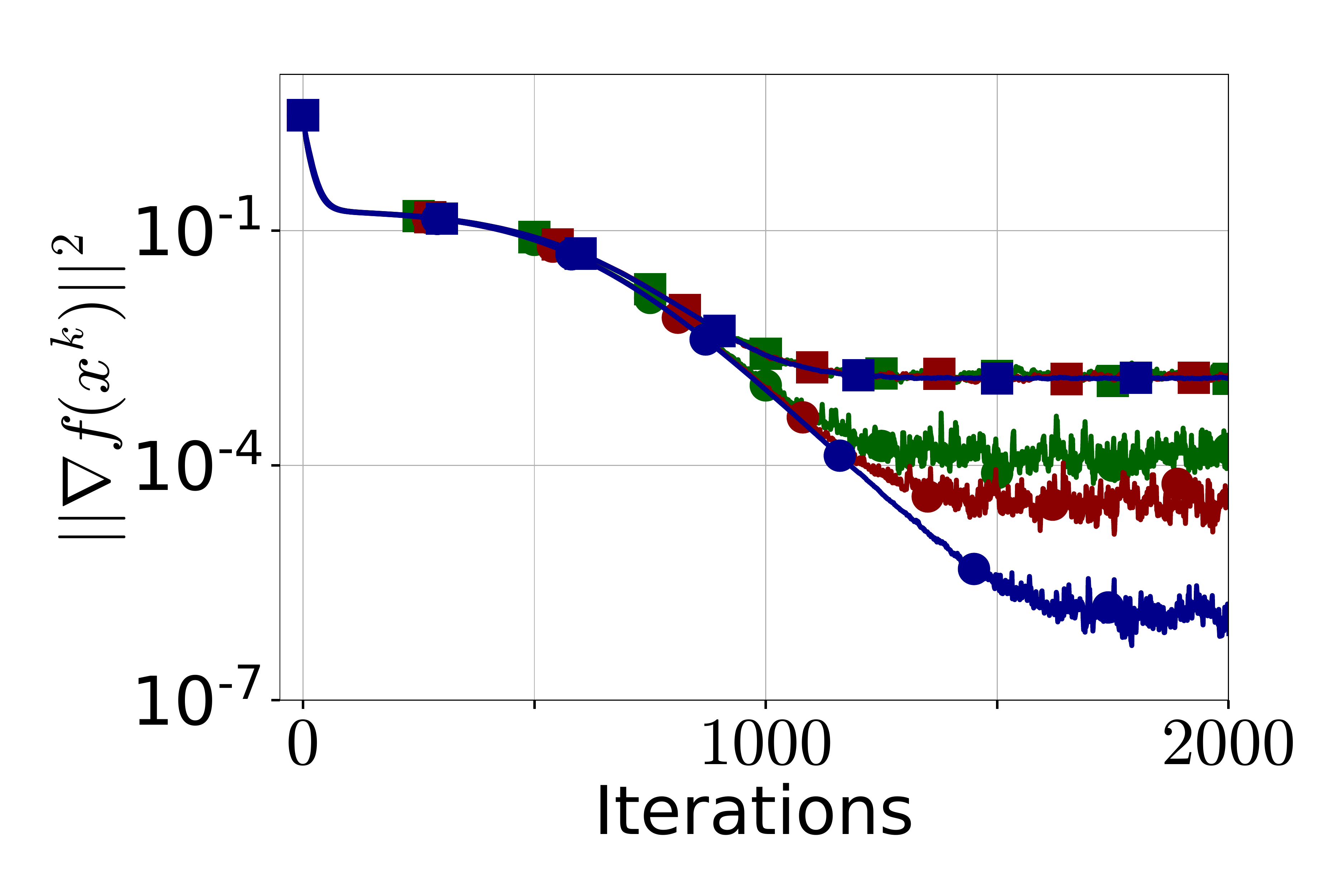} &
			\includegraphics[width=0.22\linewidth]{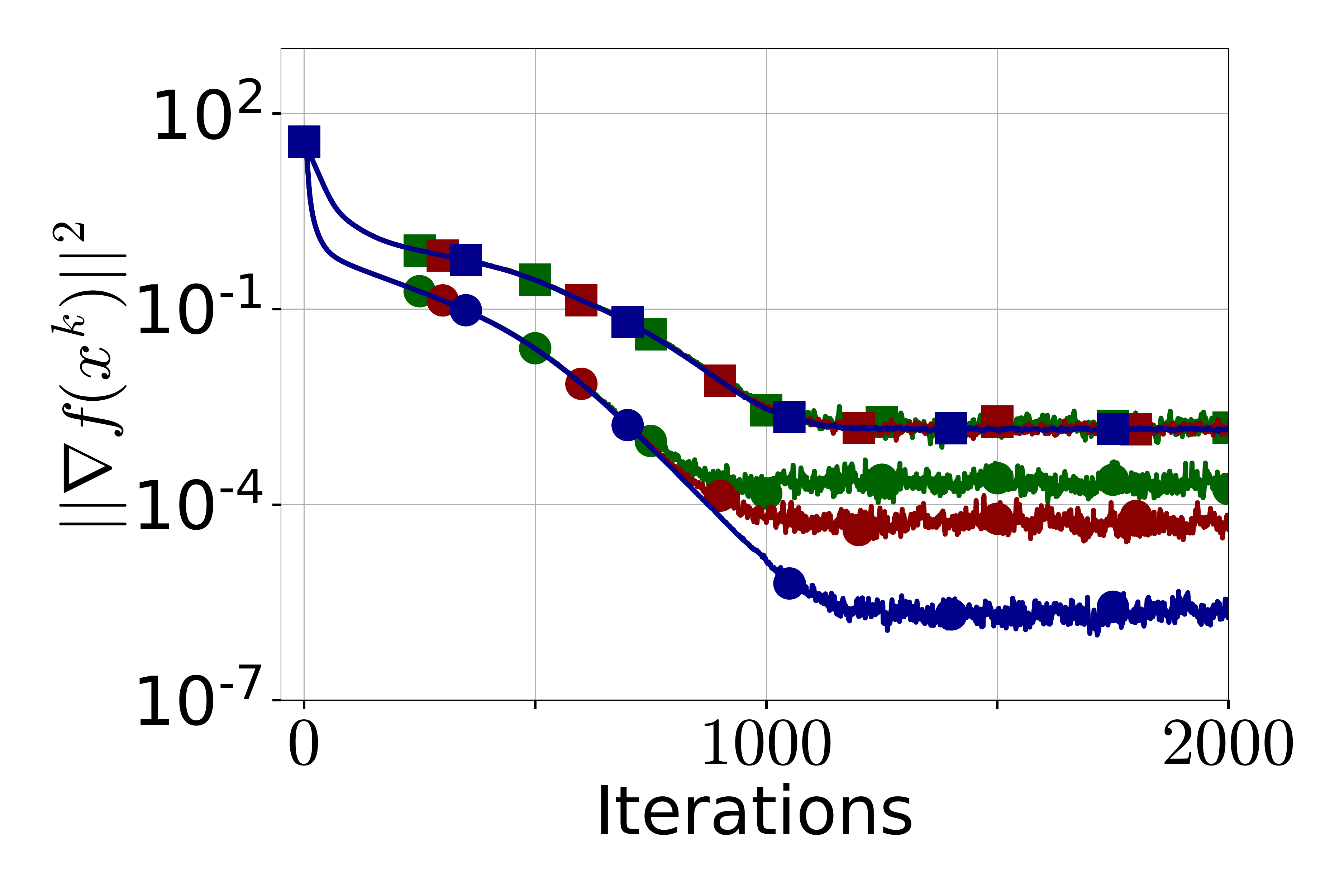}\\
			{\small (a) a1a} & 
			{\small (b) w8a } &
			{\small (c) phishing} &
			{\small (d) mushrooms}
		\end{tabular}
	\end{center}
	\caption{(a), (b) The comparison of \algname{Clip-EF21} and \algname{Clip-GD} varying clipping parameter $\tau$ on Linear Regression with nonconvex regularization. (c), (d) The comparison of \algname{DP-Clip21-GD} and \algname{DP-Clip-GD} varying noise level on Linear Regression with nonconvex regularization}
	\label{fig:linreg_dp_and_nondp}
\end{figure}

\subsection{Deep learning experiments}

We now showcase that \algname{Clip21} also outperforms \algname{Clip} for training the  VGG11 model \citep{vgg11} for multiclass classification problem on the CIFAR10 train dataset with $50000$ samples and $10$ classes ($5000$ samples for each class) \citep{cifar10}.
We modify \algname{Clip21-GD} and \algname{Clip-GD} by replacing the full local gradient $\nabla f_i(x_k)$  with its mini-batch stochastic estimator. We refer these methods as \algname{Clip21-SGD} and \algname{Clip-SGD}.

This data set is split into $10$ classes among $10$ workers according to the following rules:  
\begin{enumerate}
	\item $2500$ samples of the $i^{\text{th}}$ class are given to the $i^{\text{th}}$ client (which is in total $25000$ samples), and
	\item the rest of dataset is shuffled and partitioned randomly between workers.
\end{enumerate}
These preprocessing rules allow the $i^{\text{th}}$ worker to have most samples with the $i^{\text{th}}$ class, which makes the problem more heterogeneous. 
For each worker, its  local datasets are shuffled only once at the beginning, and  its stochastic gradient is computed for each iteration from randomly selected samples with the batch size $32$. 

We reported the best performance of \algname{Clip21-SGD} and \algname{Clip-SGD} in train loss and test accuracy from fine-tuning stepsizes (see Figure~\ref{fig:dl}). 
We observe that \algname{Clip21-SGD} outperforms \algname{Clip-SGD}  in both metrics for any values of $\tau$; see Figure~\ref{fig:dl}.
In particular, one can notice that for small value of clipping parameter ($\tau=10^{-4}$) the difference in train loss and test accuracy given by \algname{Clip21-SGD} and \algname{Clip-SGD} is significant, while for relatively large values ($\tau=10^{-2}$) the performance of algorithms becomes similar. Besides, \algname{Clip21-SGD} attains more than $3$ (in log scale) times lower train loss than \algname{Clip-SGD} at epoch $50$ for $\tau\in\{10^{-4}, 10^{-3},10^{-2}\}$. These encouraging experiments motivate us to investigate theoretical convergence guarantees for \algname{Clip21-SGD} as our future directions.  

Next, we consider DP-versions of algorithms applied on the same problem. Now workers compute mini-batches of size $512$. We add normally distributed noise to the updates varying its variance. The results of this set of experiments are presented in Figure~\ref{fig:dl_dp}. We observe that \algname{DP-Clip21-SGD} slightly outperforms \algname{DP-Clip-SGD} in both metrics.

\begin{figure}[ht]
	\begin{center}
		\begin{tabular}{ccc}
			\includegraphics[width=0.30\linewidth]{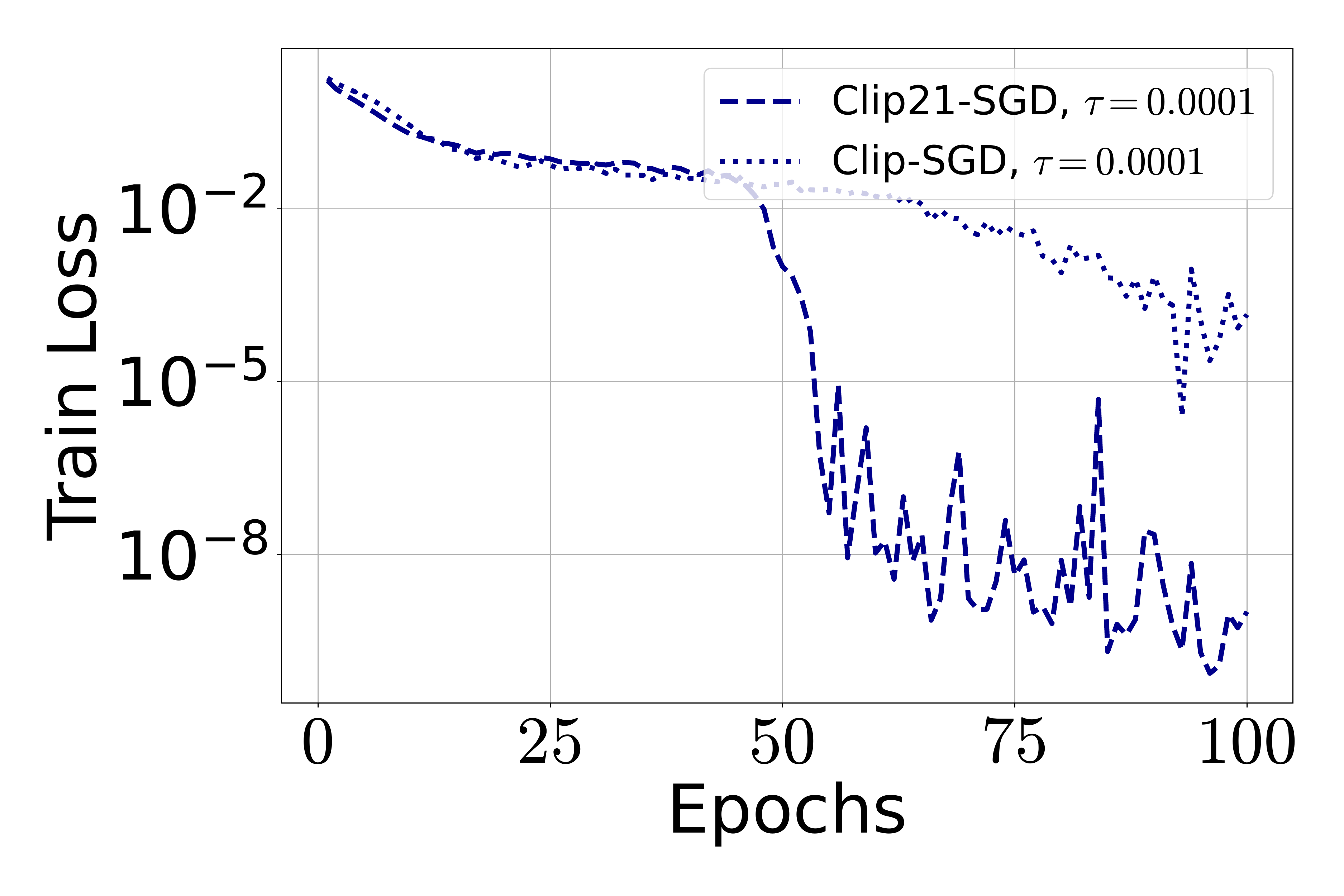} &
			\includegraphics[width=0.30\linewidth]{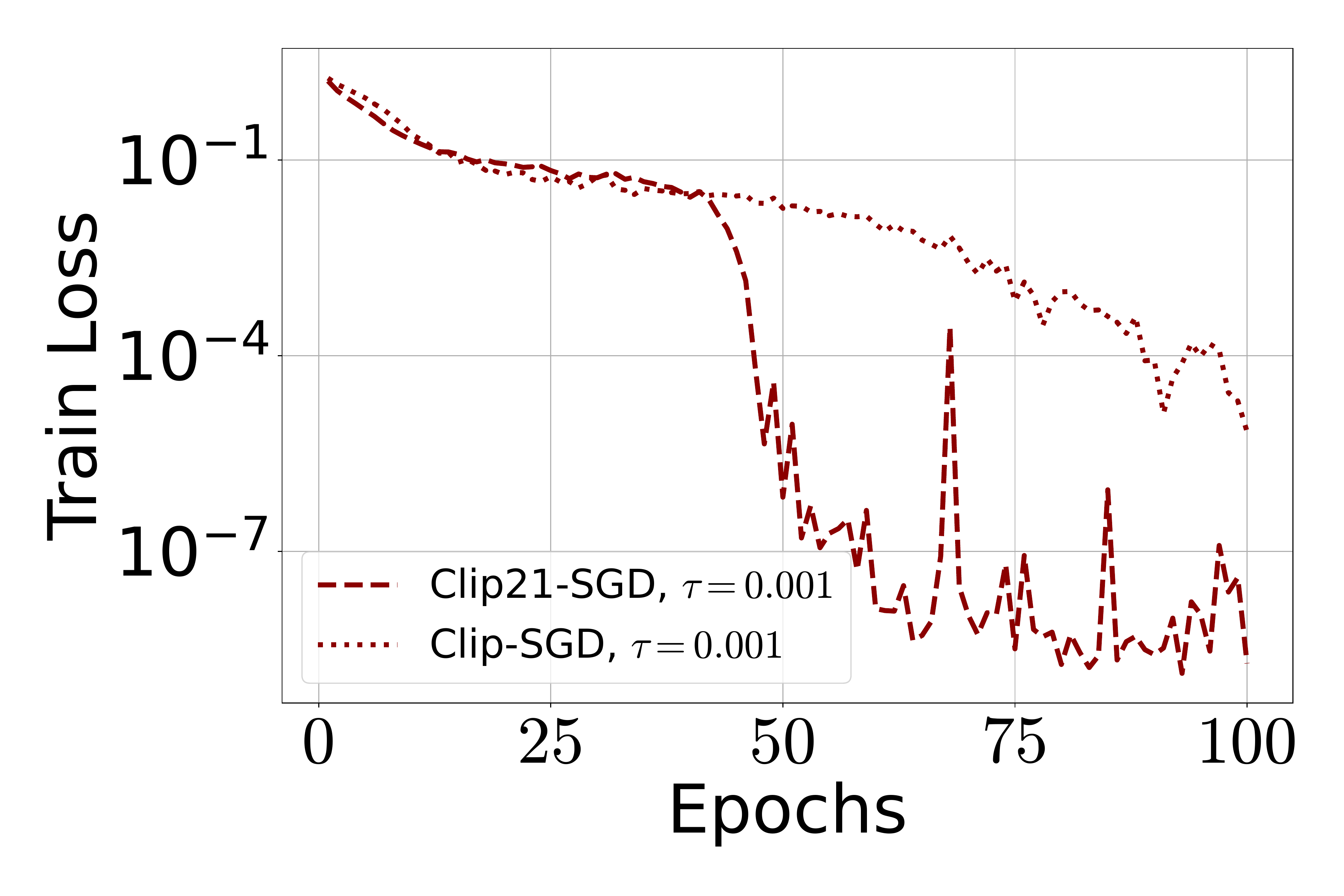} &
			\includegraphics[width=0.30\linewidth]{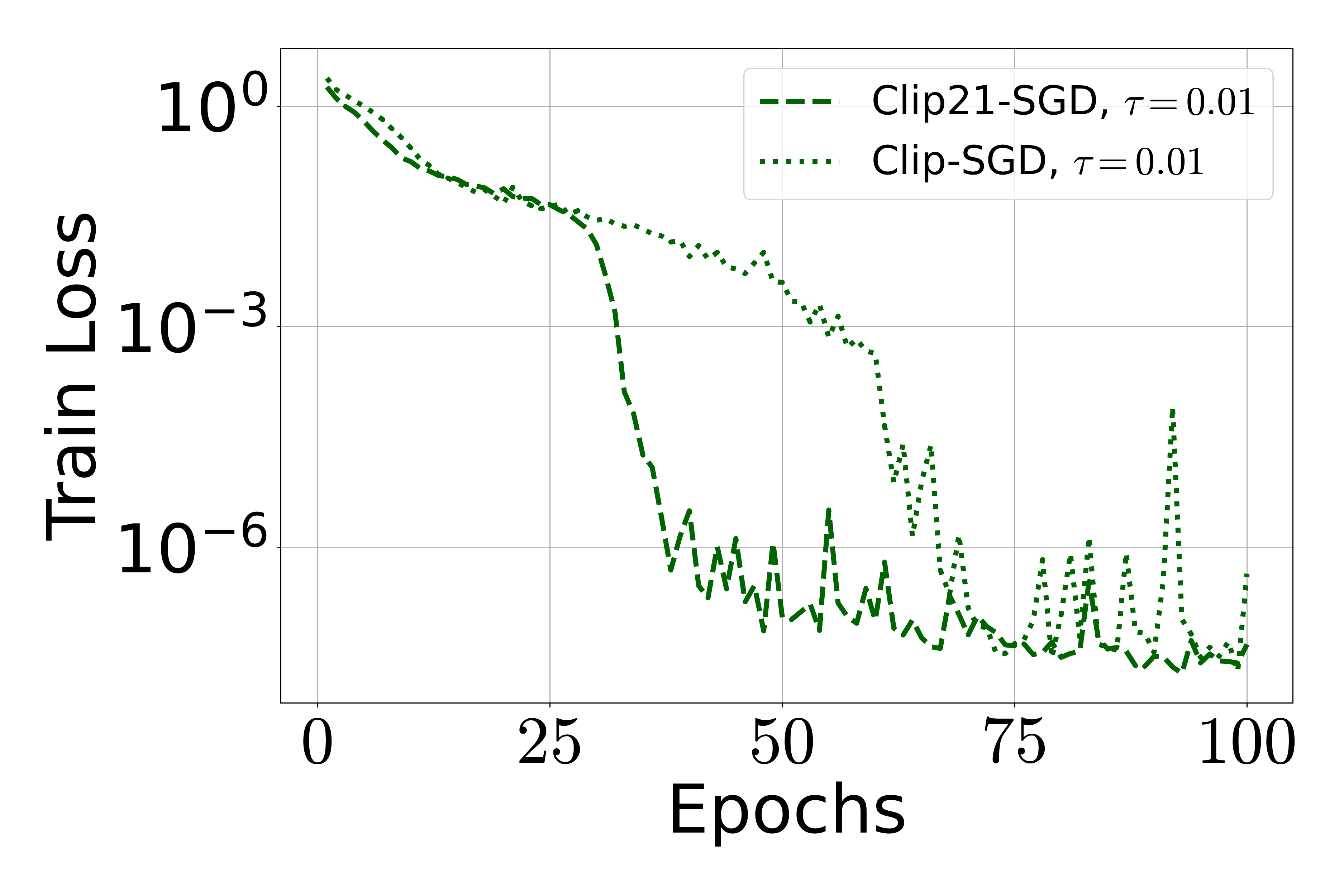}\\
			\includegraphics[width=0.30\linewidth]{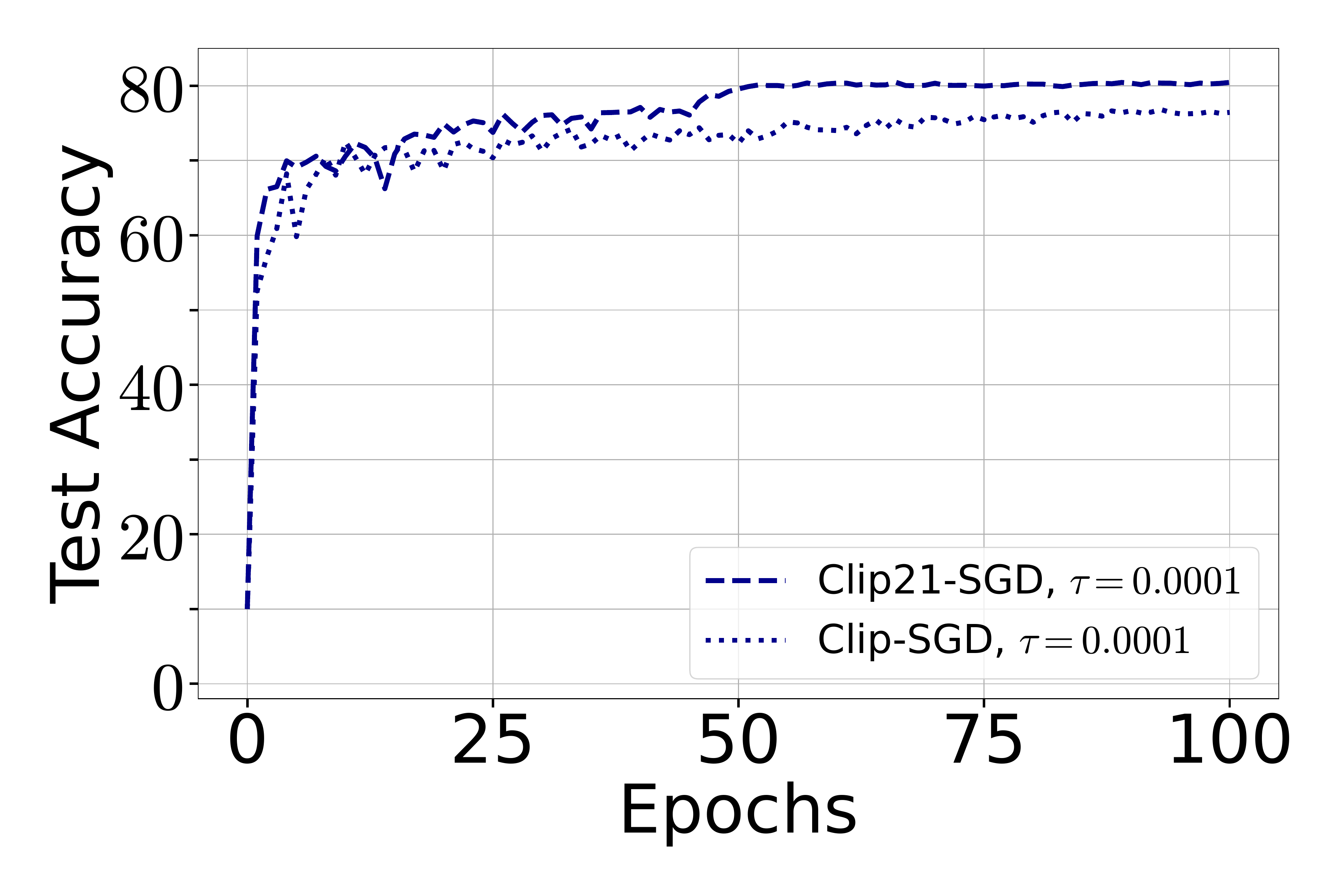} &
			\includegraphics[width=0.30\linewidth]{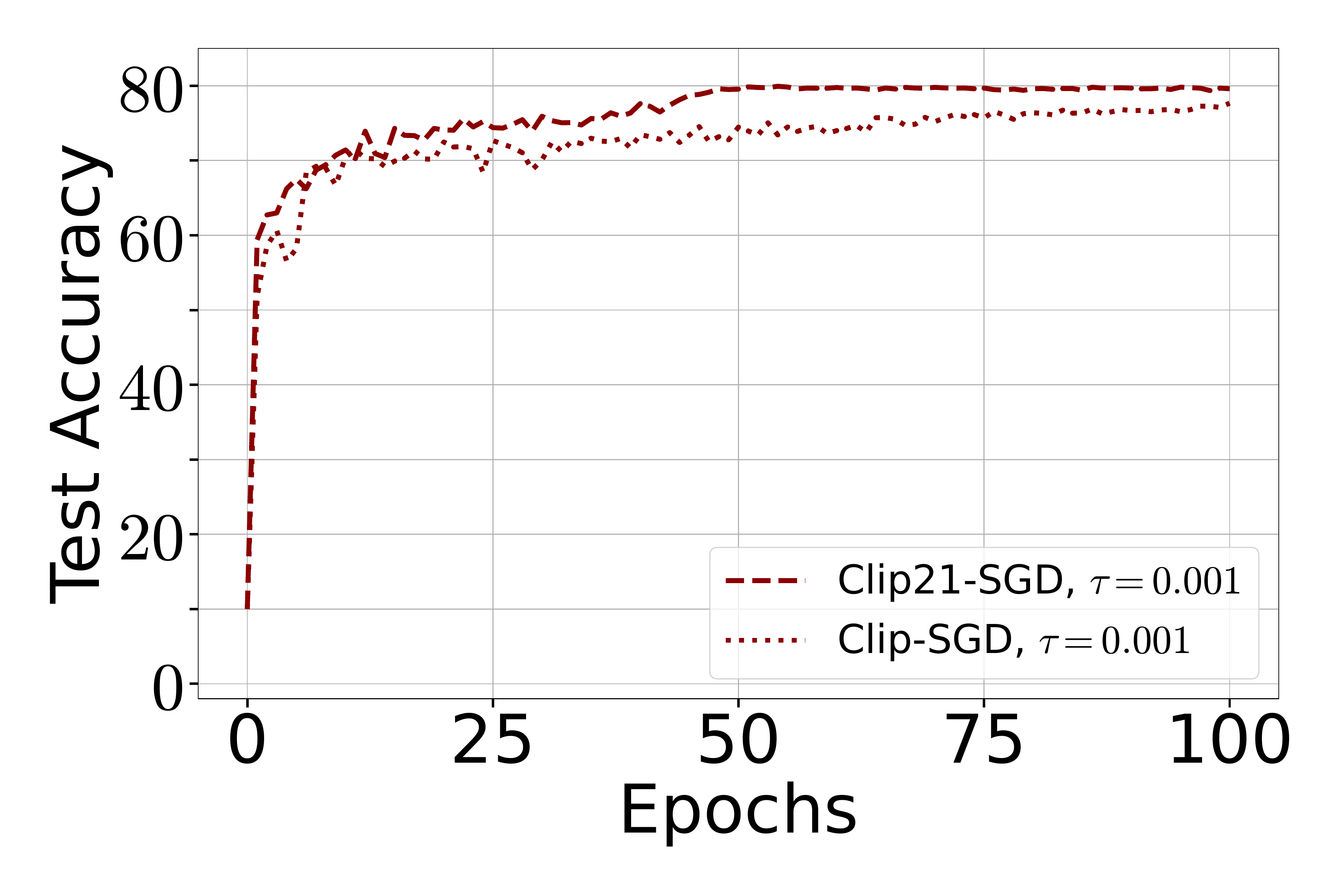} &
			\includegraphics[width=0.30\linewidth]{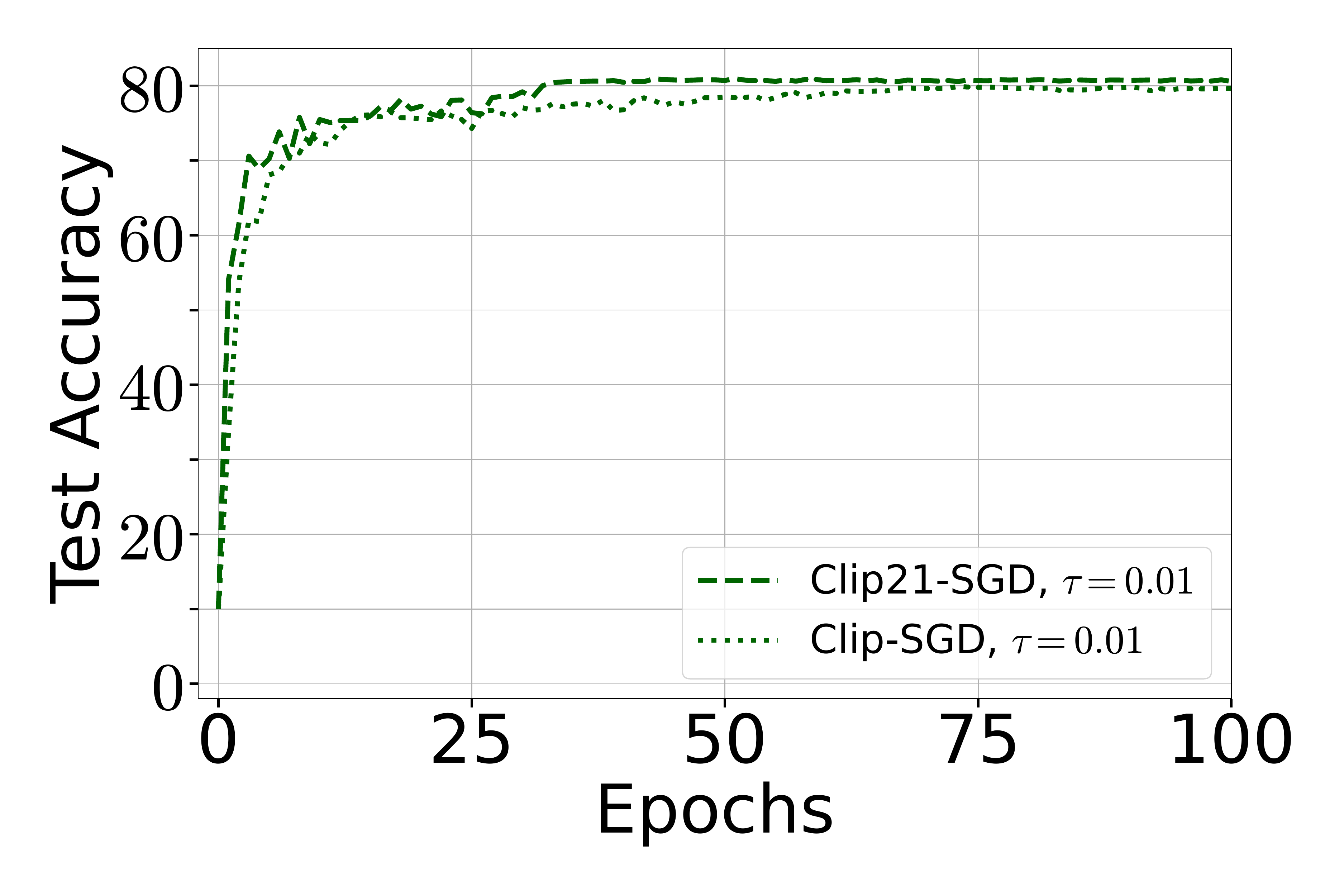}
		\end{tabular}
	\end{center}
	\caption{The performance of \algname{Clip21-SGD} and \algname{Clip-SGD} with fine-tuned stepsizes to  train the VGG11 model on the CIFAR10 dataset.}
	\label{fig:dl}
\end{figure}

\begin{figure}[ht]
	\begin{center}
		\begin{tabular}{cccc}
			\includegraphics[width=0.23\linewidth]{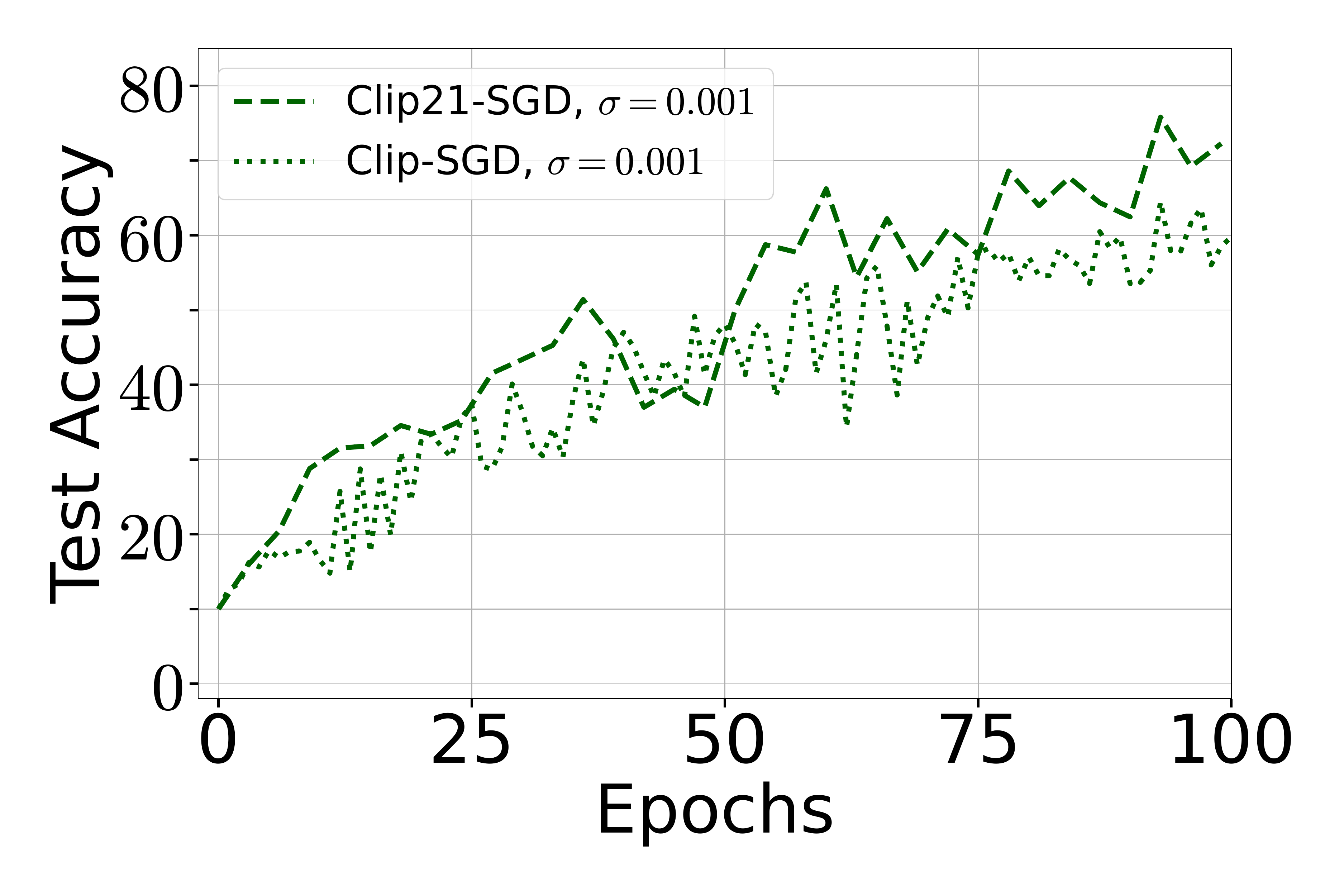} &
			\includegraphics[width=0.23\linewidth]{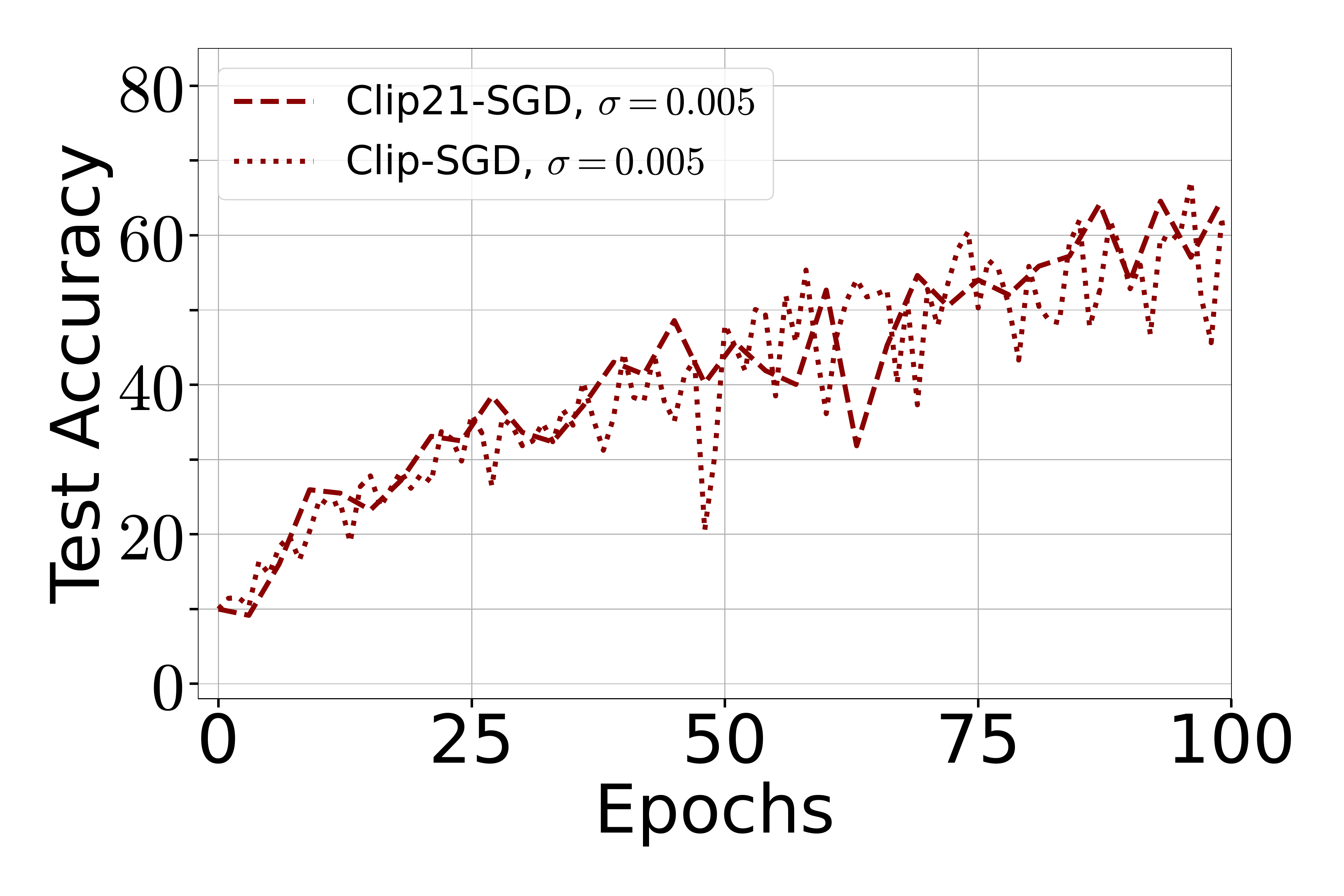} &
			\includegraphics[width=0.23\linewidth]{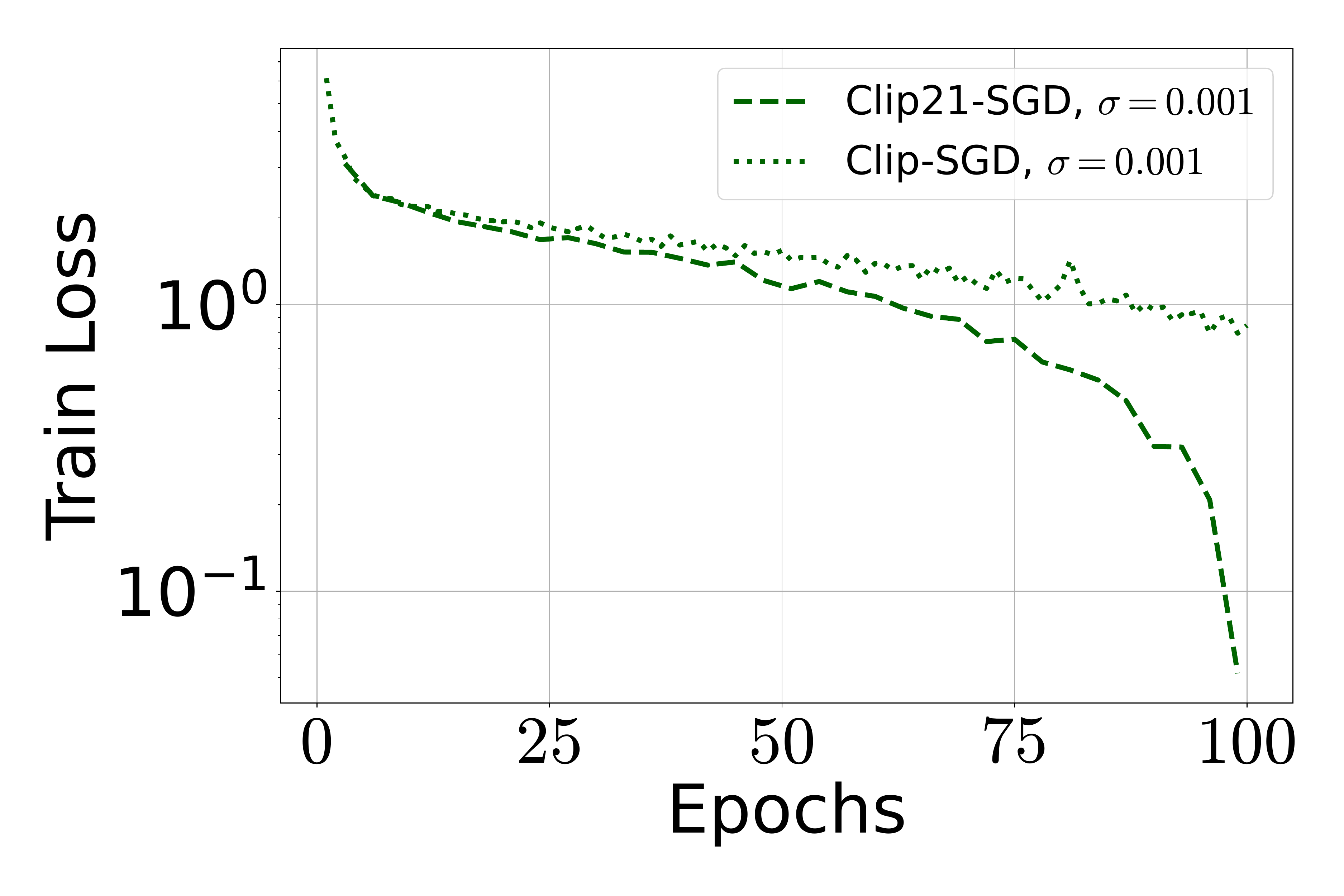}&
			\includegraphics[width=0.23\linewidth]{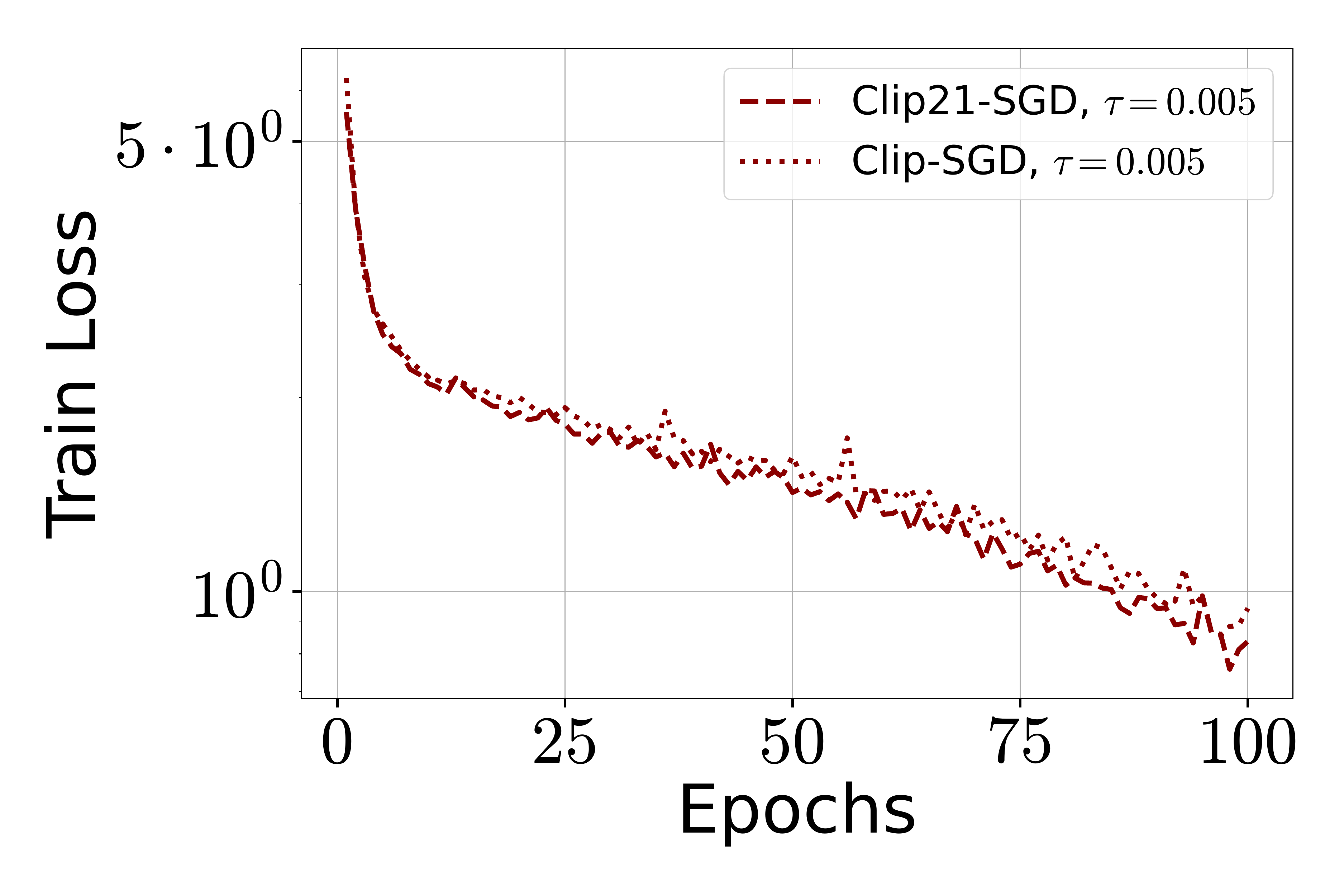}
		\end{tabular}
	\end{center}
	\caption{The performance of \algname{DP-Clip21-SGD} and \algname{DP-Clip-SGD} with fine-tuned stepsizes to  train the VGG11 model on the CIFAR10 dataset.}
	\label{fig:dl_dp}
\end{figure}

%%%%%%%%%%%%%%%%%%%%%%%%%%%%%%%%%%%%%%%%%%%%%%%%%%%%%%%%%%%%%%%%%%%%%%%%%%%%%%%
%XXX
%%%%%%%%%%%%%%%%%%%%%%%%%%%%%%%%%%%%%%%%%%%%%%%%%%%%%%%%%%%%%%%%%%%%%%%%%%%%%%%
\clearpage
\section{Loss of a Constant Factor when Generalizing to Arbitrary $n$}\label{sec:9}

When specializing our multi-node theory for \algname{Clip21-GD} to the $n=1$ case, and compare this to the theory from Section~\ref{subsec:singlenode_clip21}, we pay the price of a smaller maximum stepsize, which affects constants in the convergence rate.

We illustrate this by letting $\phi_0$ to be large, $L=L_{\max}$ and $n=1$ in Lemma~\ref{lem:multinode-clippedEF21}, Proposition~ \ref{prop:multinode-clippedEF21} and Theorem~ \ref{thm:multinode-clippedEF21}. 
Suppose that  $\frac{\norm{ \nabla f(x_0)}}{\tau}$ is close to $1$. Then, Lemma~\ref{lem:multinode-clippedEF21}, Proposition~ \ref{prop:multinode-clippedEF21} and Theorem~\ref{thm:multinode-clippedEF21} with $n=1$ recover the results from Section~\ref{subsec:singlenode_clip21} under the stepsize 
$$
	\gamma \leq \frac{1}{L}\cdot \min\left(  \frac{1-1/\sqrt{2}}{1+\sqrt{1+2\beta_1}} , \frac{\tau^2}{16L \left[ \sqrt{\FF} + \sqrt{\beta_2} \right]^2} \right),
$$
where $\eta,\beta_1,\beta_2,\FF,\GG$ are defined in Theorem \ref{thm:singlenode-clippedEF21}. 

This step-size is $4$ times smaller than that allowed by Theorem~\ref{thm:singlenode-clippedEF21} dedicated to  the single-node setting. The technical difficulty preventing us from  generalizing  the single-node theory to the multi-node version in a tighter manner is related to upper bounding the quantity $\norm{ v_k }$ in the analysis. 
In the multi-node setting, we have 
$$	\norm{v_k }  = \norm{ \nabla f(x_k) +  \frac{1}{n}\sum_{i=1}^n e_k^i  } \leq \norm{ \nabla f(x_k) } + \frac{1}{n}\sum_{i=1}^n { \norm{e_k^i} },
$$
where $$e_k^i = \clip_\tau(\nabla f_i(x_k) - v_{k-1}^i) - (\nabla f_i(x_k) - v_{k-1}^i).$$ This upper bound on $\norm{ v_k }$ is looser than that we could use in the single-node setting, which instead reads: 
$$
	\norm{ v_k } \leq (1-\eta_k)\norm{  v_{k-1} } + \eta_k \norm{  \nabla f(x_k) },
$$
where $$\eta_k = \min \left\{1, \frac{1}{ \norm{  \nabla f(x_k) - v_{k-1} }} \right\}.$$

\clearpage

\section{Basic Inequalities and Useful Lemmas}

\subsection{Basic Inequalities}

Triangle inequality:
\begin{equation}
\label{eq:triangle} \norm{x+y} \leq \norm{x} + \norm{y}, \quad \forall x,y\in \R^d.
\end{equation}

Subadditivity of the square root:
\begin{equation}
\label{eq:sqrt} \sqrt{a+b} \leq \sqrt{a} + \sqrt{b}, \quad \forall a,b\geq 0.
\end{equation}

Young's inequality:
\begin{equation}
\label{eq:Young}\norm{x + y}^2 \leq (1+\theta) \norm{x}^2 + (1 + \theta^{-1}) \norm{y}^2, \quad \forall x,y\in \R^d, \; \theta>0.
\end{equation} 

\subsection{Lemmas}
In this section, we introduce several lemmas that are instrumental to our analysis.

\begin{lemma} Let $f:\R^d\to \R$ be a function with $L$-Lipschitz gradient, with $f_{\inf} \in \R$ being a lower bound on $f$. Then
\begin{equation}
\label{eq:ssss} \frac{1}{2L} \norm{\nabla f(x)}^2 \leq f(x) - f_{\inf}, \quad \forall x\in \R^d.
\end{equation}
\end{lemma}

\begin{lemma} Let $f:\R^d\to \R$ be a function with $L$-Lipschitz gradient. Then
\begin{equation}
\label{eq:ssss_ineq} f(x) \leq f(y) + \langle \nabla f(y) , y-x \rangle + \frac{L}{2} \norm{x-y}^2, \quad \forall x\in \R^d.
\end{equation}
\end{lemma}

Lemma \ref{lemma:descent_li2021page} gives the bound on the function value after one step of a method of the type: $x_{k+1}\eqdef x_k-\gamma v_k$. We use this lemma to derive the descent inequality for \algname{Clip21-GD} and its variants.

\begin{lemma}[Lemma 2 from \citep{li2021page}]\label{lemma:descent_li2021page}
	Let $f:\R^d\to \R$ be a function with $L$-Lipschitz gradient and let $x_{k+1}\eqdef x_k-\gamma v_k$, where $\gamma >0$ and $v_k\in \R^d$ is any vector. Then
	\begin{align}\label{eq:descent_ineq}
		f(x_{k+1}) \leq f(x_k) - \frac{\gamma}{2}\norm{ \nabla f(x_k) }^2 - \left( \frac{1}{2\gamma} - \frac{L}{2}\right) \norm{ x_{k+1} - x_k }^2 + \frac{\gamma}{2}\norm{ \nabla f(x_k) - v_k }^2.
	\end{align}
\end{lemma}

Finally, we present Lemma \ref{lemma:trick_stepsize_EF21} and \ref{lemma:trick_stepsize_clippedEF21} that are useful for deriving the easy-to-write condition on the step-size that ensures the convergence of \algname{Clip21-GD} and its variants.

\begin{lemma}\label{lemma:trick_stepsize_EF21}
	If the stepsize is chosen to satisfy $$0 < \gamma \leq \frac{2}{L(\beta_1+\sqrt{\beta_1^2+4\beta_2})}$$ for some $\beta_1,\beta_2 > 0$, then $$\frac{1}{\gamma} - \beta_1 L - \gamma \beta_2 L^2 \geq 0.$$
\end{lemma}
\begin{proof}
	Let $\gamma = \nicefrac{1}{(\theta L)}$ where $\theta>0$. Then, $\gamma>0$. Also,  $\nicefrac{1}{\gamma} -\beta_1 L - \gamma \beta_2 L^2 \geq 0$ can be  expressed equivalently as 
	$\theta^2 - \beta_1 \theta - \beta_2 \geq 0.$
	This condition holds if $\theta \geq \nicefrac{(\beta_1+\sqrt{\beta_1^2+4\beta_2})}{2}$ or equivalently $0 < \gamma \leq \nicefrac{2}{[L(\beta_1+\sqrt{\beta_1^2+4\beta_2})]}$.
\end{proof}

\begin{lemma}\label{lemma:trick_stepsize_clippedEF21}
If $0< \gamma \leq \frac{4 c_3^2 \tau^2}{c_1^2 L^2 \left[ \sqrt{\FF} + \sqrt{\FF +  C\GG} \right]^2}$ for $C=\frac{4c_2c_3 \tau}{c_1^2 L}$ and for some $c_1,c_2,c_3,\FF,\GG,\tau >0$, then $c_1L\sqrt{\gamma\FF} + c_2 L \GG \gamma \leq c_3\tau$.
\end{lemma}
\begin{proof}
	Using Lemma \ref{lemma:trick_stepsize_EF21} with $\beta_1=c_1\sqrt{\FF}/(c_3\tau)$ and $\beta_2=c_2\GG/(c_3\tau L)$, we prove the result. 
\end{proof}

\clearpage
\section{Proof of Lemmas in Section~\ref{sec:average}}

\subsection{Proof of Lemma~\ref{lem:a^i}}
Subtracting $a^i$ from both sides of Step 3 of Algorithm~\ref{alg:Clip21-Avg}, and applying norms, we get
\begin{eqnarray*}
\norm{a^i - v_k^i} &= & \norm{v_{k-1}^i + \clip_\tau(a^i - v_{k-1}^i) - a^i} \\
&=& \norm{\clip_\tau(a^i - v_{k-1}^i) - ( a^i - v_{k-1}^i )}.
\end{eqnarray*}
In view of Lemma~\ref{lem:clip}, parts (ii) and (iii),  the last expression is 0 if 
$\norm{a^i - v_{k-1}^i} \leq \tau$, and is equal to $\norm{a^i - v_{k-1}^i}-\tau$ otherwise. The rest follows.

\subsection{Proof of Theorem~\ref{thm:Clip21-Avg}
}
By convexity of the norm, the quantity
$\norm{ v_k - a }  =  \norm{\frac{1}{n} \sum_{i=1}^n v^i_{k} - \frac{1}{n} \sum_{i=1}^n a^i } 
$ can be upper bounded by  $\frac{1}{n} \sum_{i=1}^n \norm{ v^i_{k} - a^i }$. It remains to apply Lemma~\ref{lem:a^i}.

%%%%%%%%%%%%%%%%%%%%%%%%%%%%%%%%%%%%%%%%%%%%%%%%%%%%%%%%%%%%%%%%%%%%%%%%%%%%%%%
% SINGLE NODE CASE
%%%%%%%%%%%%%%%%%%%%%%%%%%%%%%%%%%%%%%%%%%%%%%%%%%%%%%%%%%%%%%%%%%%%%%%%%%%%%%%
\clearpage
\section{Clip21-GD in the $n=1$ Regime}

We will now analyze the single-node  version of \algname{Clip21-GD} (Algorithm \ref{alg:Clip21-GD} for $n=1$), i.e., 
\begin{equation}\label{eq:8y98fd_08yfd}x_{k+1} = x_k - \gamma   v_k,\end{equation}
where \begin{equation}\label{eq:yf89gfdfd}v_{k} = v_{k-1} + \clip_\tau( \nabla f(x_k) - v_{k-1} ).\end{equation} Note that this can be rewritten into the simpler form
\begin{align}\label{eqn:ineq_v_k}
	v_k = (1-\eta_k)v_{k-1} + \eta_k \nabla f(x_k),
\end{align}
where 
\begin{equation}
\label{eq:98y98dg087fd+09ufd}
\eta_k \eqdef \min\left\{ 1, \frac{\tau}{\norm{\nabla f(x_k)-v_{k-1}}}  \right\}, \end{equation} 
and the ratio $\frac{\tau}{\norm{\nabla f(x_k)-v_{k-1}}} $ is interpreted as $+\infty$ if $\nabla f(x_k)=v_{k-1}$, which means that $\eta_k=1$ in that case. Note that this means that
\begin{equation}\label{eq:89hfd8f_98yfd9}
\norm{\nabla f(x_k) - v_k}^2 \overset{\eqref{eqn:ineq_v_k}}{=} (1-\eta_k)^2 \norm{\nabla f(x_k) - v_{k-1}}^2.
\end{equation}

Recall that the Lyapunov function was defined as
\begin{equation} \label{eq:Lyapunov-89d}\phi_k \eqdef f(x_k)-f_{\inf} +  \frac{\gamma A_\eta}{2} \norm{ \nabla f(x_k) - v_k }^2,\end{equation} where $A_\eta \eqdef \frac{1}{1-(1-\eta)\left(1-\frac{\eta}{2} \right) }$ and $\eta \eqdef \min \left\{1,\frac{\tau}{\norm{\nabla f(x_0)}} \right\}$. Also let $F_0\eqdef f(x_0)- f_{\inf}$ and $\GG \eqdef \max\left\{ 0, \norm{ \nabla f(x_0) } - \tau \right\}$.

\subsection{Claims}

We first establish several simple but helpful results.
\begin{claim}\label{claim:v_minus_1_n_1} Assume that $v_{-1}=0$. Then
\begin{equation}\label{eq:98ifd_8hfd9-1}v_0 = \clip_\tau( \nabla f(x_0)) , \end{equation}	
\begin{equation}\label{eq:98ifd_8hfd9-2}\norm{v_0} = \min \left\{ \tau, \norm{\nabla f(x_0)} \right\} ,
\end{equation} 
and
\begin{equation}\label{eq:98ifd_8hfd9-3}
	\norm{ \nabla f(x_0) - v_0 }  = \max\left\{0,\norm{ \nabla f(x_0) } - \tau \right\}.
\end{equation}	
If we additionally assume that  $\gamma \leq \frac{2}{L}$, then 
\begin{equation}\label{eq:98ifd_8hfd9-4}
	\norm{ v_0 }  \leq \sqrt{\frac{4}{\gamma} \phi_0}.
\end{equation}	
\end{claim}
%%%%
\begin{proof}
Equation \eqref{eq:98ifd_8hfd9-1} follows from
\begin{eqnarray*}
	v_0 & \overset{\eqref{eq:yf89gfdfd}}{=} &  v_{-1} + \clip_\tau( \nabla f(x_0) - v_{-1} ) \quad = \quad	\clip_\tau(\nabla f(x_0)).
\end{eqnarray*}
Equation \eqref{eq:98ifd_8hfd9-2} follows from \eqref{eq:98ifd_8hfd9-1}. Equation \eqref{eq:98ifd_8hfd9-3} follows from \eqref{eq:98ifd_8hfd9-1} and Lemma~\ref{lem:clip}:
 \[\norm{ \nabla f(x_0) - v_0 }  \overset{\eqref{eq:98ifd_8hfd9-1}}{=} \norm{ \nabla f(x_0) - \clip_\tau(\nabla f(x_0))  }  \overset{(\text{Lemma~\ref{lem:clip}(ii)-(iii)})}{=} \max\left\{0,\norm{ \nabla f(x_0) } - \tau\right\}.\]
 Finally, \eqref{eq:98ifd_8hfd9-4} follows from
  \begin{eqnarray*}
	\norm{v_0}&\overset{\eqref{eq:98ifd_8hfd9-2}}{=}& \min \left\{ \tau, \norm{\nabla f(x_0)} \right\}  \\
	&\leq & \norm{\nabla f(x_0)}  \\& \overset{\eqref{eq:ssss}}{\leq} & \sqrt{2L\left(f(x_0) - f_{\rm inf}\right)} \\
	& \overset{\eqref{eq:Lyapunov-89d}}{\leq} &  \sqrt{2L \phi_0} \\
	&\leq &  \sqrt{\frac{4}{\gamma} \phi_0},
\end{eqnarray*}
where the last inequality follows from
the assumption $\gamma \leq \frac{2}{L}$.
\end{proof}

\begin{claim} \label{claim:98u9x8u=09d} Fix $k\geq 0$. Then
\begin{equation}\label{eq:888999007567}
	\norm{ \nabla f(x_{k+1}) - v_k }  \leq   \max\left\{0,\norm{ \nabla f(x_k) - v_{k-1}} - \tau \right\} + L \gamma \norm{ v_k } .
\end{equation}
\end{claim}
%%%%
\begin{proof}
Indeed,
\begin{eqnarray*}
	 \norm{ \nabla f(x_{k+1}) - v_k } & 	 = 	 &\norm{\nabla f(x_k) - v_k  +  \nabla f(x_{k+1}) - \nabla f(x_k) }  \\
	& \overset{\eqref{eq:triangle}}{\leq} & 	 \norm{ \nabla f(x_k) - v_k } + \norm{ \nabla f(x_{k+1}) -\nabla f(x_k) }  \\
	& \overset{\eqref{eq:yf89gfdfd}}{=} & \norm{ (\nabla f(x_k) - v_{k-1}) - \clip_\tau(\nabla f(x_k) - v_{k-1}) } \\
 &&+ \norm{ \nabla f(x_{k+1}) -\nabla f(x_k) } 	 \\
	&\overset{(\text{Lemma~\ref{lem:clip}(ii)-(iii)})}{=} & \max\left\{0,\norm{ \nabla f(x_k) - v_{k-1}} - \tau \right\}  + \norm{ \nabla f(x_{k+1}) -\nabla f(x_k) } 	\\
	& \overset{\eqref{eq:L-smooth}}{\leq}   & \max\left\{0,\norm{ \nabla f(x_k) - v_{k-1}} - \tau \right\} + L \norm{ x_{k+1} - x_k } \\
	&\overset{\eqref{eq:8y98fd_08yfd}}{=}&  \max\left\{0,\norm{ \nabla f(x_k) - v_{k-1}} - \tau \right\}  + L \gamma \norm{ v_k }.
\end{eqnarray*}

\end{proof}

\begin{claim}\label{claim:fu9o} Fix $k\geq 1$. If $\frac{\gamma}{2}\norm{\nabla f(x_{k-1}) }^2 \leq \phi_0$,  $ \frac{\gamma}{4}\| v_{k-1}\|^2 \leq \phi_0$ and $\gamma \leq \frac{1- \frac{1}{\sqrt{2}}}{L}$, then
\begin{equation}
\label{eq:v_k-bound_09u0fd9uf}
	\norm{v_{k}} \leq  \sqrt{\frac{4}{\gamma}\phi_0}.
\end{equation}
\end{claim}
\begin{proof}
\begin{eqnarray}
	\norm{v_{k}} &\overset{\eqref{eqn:ineq_v_k}+\eqref{eq:triangle}}{\leq} & (1-\eta_{k})\|v_{k-1}\| + \eta_k\|\nabla f(x_{k})\| \notag \\
	& \overset{\eqref{eq:triangle} }{\leq} & (1-\eta_{k})\|v_{k-1}\| +  \eta_k \|\nabla f(x_{k})-\nabla f(x_{k-1})\| +\eta_k \|\nabla f(x_{k-1})\|  \notag \\
	& \overset{\eqref{eq:L-smooth}}{\leq}   &  (1-\eta_{k})\|v_{k-1}\| + \eta_k L  \norm{x_k - x_{k-1}}  +  \eta_k \norm{\nabla f(x_{k-1})} \notag  \\	& \overset{\eqref{eq:8y98fd_08yfd}}{=}   &  (1-\eta_{k})\|v_{k-1}\| + \eta_k L \gamma \norm{v_{k-1}}  +  \eta_k \|\nabla f(x_{k-1})\| \notag  \\	
	& = & (1-\eta_k + \eta_k L\gamma) \norm{ v_{k-1} } + \eta_k \norm{ \nabla f(x_{k-1}) } \notag \\
	& \overset{(*)}{\leq} & (1-\eta_k + \eta_k L\gamma) \sqrt{\frac{4}{\gamma}\phi_0} + \eta_k \sqrt{\frac{2}{\gamma}\phi_0} \notag \\
	& = & \left(1-\eta_k + \eta_k L\gamma + \frac{\eta_k}{\sqrt{2}} \right) \sqrt{\frac{4}{\gamma}\phi_0} \notag \\
	& \overset{(**)}{\leq} &  \sqrt{\frac{4}{\gamma}\phi_0}, \label{eq:09u09ud099fdfd}
\end{eqnarray}
where in (*) we have used the first two assumptions, and in (**) we have used the bound on $\gamma$.
\end{proof}

\begin{claim} \label{claim:080fd-=9090} Assume that $v_{-1}=0$. If 
\begin{align}\label{eqn:stepsize_clippedEF21_singlenode_on_closedform}
	0 \leq \gamma \leq \frac{\tau^2}{4L^2 \left[ \sqrt{\FF} + \sqrt{\FF + \frac{\tau\GG}{\sqrt{2\eta} L}} \right]^2},
\end{align}	
where $\eta\eqdef \min\left\{1,\frac{\tau}{\norm{\nabla f(x_0)}} \right\}$, $\FF \eqdef f(x_0)-f_{\rm inf}$ and $\GG \eqdef \max \left\{ 0,\norm{ \nabla f(x_0) } - \tau \right \}$, 
then $$2 L  \sqrt{\gamma\phi_0} \leq \frac{\tau}{2}.$$
\end{claim}
\begin{proof}
This takes a bit of effort since $\phi_0$ depends on $\gamma$ as well. In particular,
\begin{eqnarray*}
	2L \sqrt{\gamma \phi_0}
	& \overset{\eqref{eq:Lyapunov-89d}}{=} & 2L \sqrt{\gamma \left(f(x_0)-f_{\inf}\right) + \frac{\gamma^2 A_\eta}{2} \norm{ \nabla f(x_0) - v_0 }^2 } \\
	& \overset{\eqref{eq:98ifd_8hfd9-3}}{=}  & 2L \sqrt{\gamma \FF + \frac{\gamma^2 A_\eta}{2} \GG^2 } \\
	& \leq  & 2L \sqrt{\gamma \FF + \frac{\gamma^2}{2\eta}  \GG^2 },
\end{eqnarray*}	
where we reach the last inequality by the fact that $A_\eta = \frac{1}{1-(1-\eta)\left(1- \frac{\eta}{2} \right)} \leq \frac{1}{1-(1-\eta)}=\frac{1}{\eta}$. By subadditivity of $t\mapsto \sqrt{t}$, we therefore get 
\begin{eqnarray*}
	2L \sqrt{\gamma \phi_0}
	& \overset{\eqref{eq:sqrt}}{\leq} & 2L \sqrt{\FF} \sqrt{\gamma} + \sqrt{\frac{2}{\eta}}L \GG \gamma.
\end{eqnarray*}
Hence, any $\gamma >0$ satisfying
\begin{align}\label{eqn:stepsize_clippedEF21_singlenode_on}
	2L \sqrt{\FF} \sqrt{\gamma} + \sqrt{\frac{2}{\eta}} L \GG \gamma \leq \frac{\tau}{2}
\end{align}	
also satisfies $2 L  \sqrt{\gamma\phi_0} \leq \frac{\tau}{2}$. The condition \eqref{eqn:stepsize_clippedEF21_singlenode_on} can be re-written as 
\begin{align*}
	\frac{1}{\sqrt{\gamma}} - \frac{4}{\tau}\sqrt{\FF} \cdot L - \sqrt{\gamma} \cdot \frac{2\sqrt{2}}{\tau \sqrt{\eta}} \frac{\GG}{L} \cdot L^2 \geq 0.
\end{align*}
It remains to apply Lemma~\ref{lemma:trick_stepsize_clippedEF21} with $c_1 = 2$, $c_2=\sqrt{\frac{2}{\eta}}$ and $c_3 = \frac{1}{2}$.

\end{proof}

%%%%%%%%%%%%%%%%%%
\begin{claim} \label{eq:98fy9d_09} Pick any $k\geq 0$ and assume that $\norm{\nabla f(x_{k+1})-v_k} \leq \norm{\nabla f(x_{0})}$. Also let $\eta = \min\left\{ 1, \frac{\tau}{\norm{\nabla f(x_0)}} \right\}$. Then 
\begin{equation}\label{eqn:0980d9_09ufd}
	\norm{ \nabla f(x_{k+1}) - v_{k+1} }^2 
	 \leq  (1-\eta)\left(1-\frac{\eta}{2}\right) \norm{ \nabla f(x_{k}) - v_{k}  }^2 + \left(1+\frac{2}{\eta}\right)(1-\eta)^2 L^2 \norm{ x_{k+1} - x_{k}  }^2,
\end{equation}
and
\begin{equation}\label{eq:8908fd_-0-fdi}
	\phi_{k+1}  
	\leq  \phi_k - \frac{\gamma}{2}\norm{ \nabla f(x_k) }^2 - \frac{1}{2\gamma}\left(1 - \gamma L - \gamma^2 A_\eta \left(1+ \frac{2}{\eta} \right)(1-\eta)^2  L^2\right)\norm{ x_{k+1} - x_k }^2 .
	\end{equation}
If moreover the step-size  satisfies 
\begin{equation}
	0 < \gamma \leq \frac{1}{L \left(1 + \sqrt{1 +  2 A_\eta (1-\eta)^2 \left(1+ \frac{2}{\eta} \right)}\right)} , \label{eq:stepsize_bound_6}
\end{equation}
then
\begin{eqnarray}
	\phi_{k+1} \leq  \phi_k  - \frac{\gamma}{2}\norm{ \nabla f(x_k) }^2 - \frac{\gamma}{4}\norm{ v_k  }^2.	\label{eq:descent-claim-xx}
\end{eqnarray}	
\end{claim}
%%%%%%%%%%%%%%%%%%
\begin{proof}
Since $\norm{\nabla f(x_{k+1})-v_{k}} \leq \norm{\nabla f(x_0)}$, 
\begin{equation}\label{eq:hp87yfd0}
\eta_{k+1} \overset{\eqref{eq:98y98dg087fd+09ufd}}{=} \min\left\{ 1, \frac{\tau}{\norm{\nabla f(x_{k+1})-v_{k}}}  \right\} \geq \min\left\{ 1, \frac{\tau}{\norm{\nabla f(x_0)}}  \right\} = \eta. 
\end{equation} 
Therefore,
\begin{eqnarray*}\norm{\nabla f(x_{k+1}) - v_{k+1}}^2 & \overset{\eqref{eq:89hfd8f_98yfd9}}{=} & (1-\eta_{k+1})^2 \norm{\nabla f(x_{k+1}) - v_{k}}^2 \\
&\overset{\eqref{eq:hp87yfd0}}{\leq }&(1-\eta)^2 \norm{\nabla f(x_{k+1}) - v_{k}}^2 \\
&=&(1-\eta)^2 \norm{\nabla f(x_{k})- v_{k} + \nabla f(x_{k+1}) -\nabla f(x_{k}) }^2 \\
& \overset{\eqref{eq:Young}}{\leq} & (1+\theta)(1-\eta)^2 \norm{ \nabla f(x_{k}) - v_{k}  }^2 \\
&& + \left(1+ \theta^{-1}\right)(1-\eta)^2 \norm{ \nabla f(x_{k+1}) - \nabla f (x_{k})  }^2\\
&\leq &(1+\theta)(1-\eta)^2 \norm{ \nabla f(x_{k}) - v_{k}  }^2 + \left(1+ \theta^{-1}\right)(1-\eta)^2 L^2 \norm{ x_{k+1}  -x_{k}  }^2,
\end{eqnarray*}
where we have the freedom to choose $\theta>0$.
To obtain \eqref{eqn:0980d9_09ufd}, it remains to choose $\theta = \frac{\eta}{2}$ and apply the inequality $(1-\eta)(1+\frac{\eta}{2}) \leq 1-\frac{\eta}{2}$ (which holds for any $\eta \in \R$).

Furthermore, by combining \eqref{eqn:0980d9_09ufd} with Lemma \ref{lemma:descent_li2021page}, we get \begin{eqnarray}
	\phi_{k+1}  
	& \overset{\eqref{eq:Lyapunov-89d}}{=} &f(x_{k+1}) - f_{\rm inf} + \frac{\gamma  A_\eta}{2} \norm{ \nabla f(x_{k+1}) - v_{k+1} }^2  \notag \\
	& \overset{\eqref{eq:descent_ineq}}{\leq}& f(x_k)- f_{\inf}- \frac{\gamma}{2}\norm{ \nabla f(x_k) }^2 - \left( \frac{1}{2\gamma} - \frac{L}{2} \right)\norm{ x_{k+1} - x_k }^2 + \frac{\gamma}{2}\norm{ \nabla f(x_k) - v_{k}   }^2 \notag \\ 
	&& + \frac{\gamma  A_\eta}{2} \norm{ \nabla f(x_{k+1}) - v_{k+1} }^2 \notag \\
	&\overset{\eqref{eq:Lyapunov-89d}}{=}& \phi_k - \frac{\gamma}{2}\norm{ \nabla f(x_k) }^2 - \left( \frac{1}{2\gamma} - \frac{L}{2} \right)\norm{ x_{k+1} - x_k }^2  + \left( \frac{\gamma }{2} - \frac{\gamma  A_\eta}{2} \right)\norm{ \nabla f(x_k) - v_{k}   }^2 \notag \\
	&& + \frac{\gamma  A_\eta}{2} \norm{ \nabla f(x_{k+1}) - v_{k+1} }^2  \notag  \\
	& \overset{\eqref{eqn:0980d9_09ufd}}{\leq}& \phi_k - \frac{\gamma}{2}\norm{ \nabla f(x_k) }^2 - \left( \frac{1}{2\gamma} - \frac{L}{2} -  \frac{\gamma  A_\eta}{2}  \left( 1 + \frac{2}{\eta}\right) (1-\eta)^2 L^2 \right)\norm{ x_{k+1} - x_k }^2 \notag \\
	&& \hspace{0.5cm} + \left( \frac{\gamma}{2} + \frac{\gamma  A_\eta}{2} (1-\eta)\left(1- \frac{\eta}{2} \right) - \frac{\gamma  A_\eta}{2} \right)\norm{ \nabla f(x_k) - v_{k}   }^2 \notag \\
&=& \phi_k - \frac{\gamma}{2}\norm{ \nabla f(x_k) }^2 - \frac{1}{2\gamma}\left(1 - \gamma L - \gamma^2 A_\eta \left(1+ \frac{2}{\eta} \right)(1-\eta)^2  L^2\right)\norm{ x_{k+1} - x_k }^2,	\label{eq:988y=-oidf9u}
	\end{eqnarray}
	where in the last step the term corresponding to $\norm{ \nabla f(x_k) - v_{k}   }^2$ vanished because  $A_\eta \eqdef \frac{1}{1-(1-\eta)\left(1-\frac{\eta}{2}\right)}$. Recall that  $\eta \eqdef \min \left\{1,\frac{\tau}{\norm{\nabla f(x_0)}}\right\}$. 
		
Finally, if the step-size $\gamma$ satisfies \eqref{eq:stepsize_bound_6},
then from  Lemma \ref{lemma:trick_stepsize_EF21} with $\beta_1=2$ and $\beta_2 = 2 A_{\eta} (1-\eta)^2 \left(1+\frac{2}{\eta}\right)$, this condition implies $1 - \gamma L - \gamma^2 A_{\eta} (1-\eta)^2 \left(1+ \frac{2}{\eta}\right)  L^2  \geq \frac{1}{2}$
and thus
\begin{eqnarray*}
	\phi_{k+1} & \overset{\eqref{eq:988y=-oidf9u}}{\leq} & \phi_k - \frac{\gamma}{2}\norm{ \nabla f(x_k) }^2 - \frac{1}{4\gamma} \norm{ x_{k+1} - x_k }^2 \\
&\overset{\eqref{eq:8y98fd_08yfd}}{=}& \phi_k  - \frac{\gamma}{2}\norm{ \nabla f(x_k) }^2 - \frac{\gamma}{4}\norm{ v_k  }^2.	
\end{eqnarray*}	
\end{proof}

\subsection{Proof of Lemma~\ref{lem:singlenode-clippedEF21}}

We will derive the descent inequality 
\begin{align*}
	\phi_{k+1} & \leq \phi_k - \frac{\gamma}{2}\norm{ \nabla f(x_k) }^2 .
\end{align*}
 We consider two cases: (1) when $\tau < \| \nabla f(x_{k}) - v_{k-1}\|$ and (2) when $\| \nabla f(x_{k}) - v_{k-1}\| \leq \tau$.

For notational convenience, we assume that $\nabla f(x_{-1}) = v_{-1} = 0$ and $\phi_{-1} = \phi_0$.
 
\subsection{Case (1): $\tau < \| \nabla f(x_{k}) - v_{k-1}\|$}
Assume that \begin{equation}\label{eq:case1}\tau \leq \norm{\nabla f(x_{k}) - v_{k-1}}.\end{equation} To derive the descent inequality we will show by induction the stronger result: $\| \nabla f(x_{k}) - v_{k-1}\| \leq B - \frac{k\tau}{2}$, where $B \eqdef \norm{\nabla f(x_0)}$ and
\begin{equation}
	\phi_{k} 
	 \leq  \phi_{k-1}  - \frac{\gamma}{2}\norm{\nabla f(x_{k-1}) }^2 - \frac{\gamma}{4}\| v_{k-1}\|^2 \label{eq:detailed_descent_inequality}
\end{equation}
for any $k \geq 0$. The base of the induction is trivial: when $k = 0$ we have   $\| \nabla f(x_{k}) - v_{k-1}\| = \|\nabla f(x_0) - v_{-1}\| = \|\nabla f(x_0)\| := B$ and \eqref{eq:detailed_descent_inequality} holds by definition since $\phi_0=\phi_{-1}$. Next, we assume that for some $k\geq 0$ inequalities $\norm{\nabla f(x_t) - v_{t-1}} \leq B$ and \eqref{eq:detailed_descent_inequality} hold for $t = 0,1,\ldots, k$.

Let $0 \leq \phi_k \leq \phi_0$ for all $k\geq 0$. If $\gamma \leq 1/L$, then  
\begin{eqnarray*}
    \norm{\nabla f(x_{k-1})}^2 
    & \mathop{\leq}\limits^{\eqref{eq:ssss}} & 2L [f(x_{k-1})-f_{\inf}] \\ 
    & \mathop{\leq}\limits^{\eqref{eq:Lyapunov-89d}}  & 2L \phi_{k-1} \\ 
    & \leq & 2L \phi_0  \\ 
    & \leq & \frac{2}{\gamma}\phi_0.
\end{eqnarray*}
If $\gamma \leq \frac{1-\frac{1}{\sqrt{2}}}{L}$ and $v_{-1}=0$, then we have from Claim~\ref{claim:v_minus_1_n_1} and~\ref{claim:fu9o}  that for $k \geq 0$
\begin{eqnarray*}
	\norm{v_{k}} \leq  \sqrt{\frac{4}{\gamma}\phi_0}.    
\end{eqnarray*}

Next, from Claim~\ref{claim:98u9x8u=09d} with \eqref{eq:case1}, 
\begin{eqnarray*}
    \norm{\nabla f(x_{k+1}) - v_k} 
    & \leq & \norm{\nabla f(x_k) - v_{k-1}} - \tau + L\gamma \norm{v_k} \\ 
    & \leq & \norm{\nabla f(x_k) - v_{k-1}} - \tau + 2L \sqrt{\gamma\phi_0}. 
\end{eqnarray*}

{\bf STEP: Small stepsize.}

If $\gamma>0$ satisfies  \eqref{eqn:stepsize_clippedEF21_singlenode_on_closedform}, then from Claim~\ref{claim:080fd-=9090} we have $2 L  \sqrt{\gamma\phi_0} \leq \frac{\tau}{2}$. Hence, the above inequality and the inductive assumption imply
\begin{align}
	\norm{ \nabla f(x_{k+1}) - v_{k} }
	& \leq \norm{ \nabla f(x_{k}) - v_{k-1}  } - \frac{\tau}{2} \leq \cdots \leq \norm{ \nabla f(x_0) }- \frac{(k+1)\tau}{2}. \label{eq:decrease_of_the_norm_diff_single_node}
\end{align}
In conclusion, Eq. \eqref{eq:decrease_of_the_norm_diff_single_node} implies that $\norm{ \nabla f(x_k)-v_{k-1}} \leq \norm{\nabla f(x_0)}:= B$ for any $k\geq 0$.

{\bf STEP: Descent inequality.}

If the step-size $\gamma>0$ satisfies \eqref{eq:stepsize_bound_6}, then by the assumption that $\norm{ \nabla f(x_k)-v_{k-1}} \leq B$ for any $k\geq 0$ and  from Claim~\ref{eq:98fy9d_09} we obtain
\begin{eqnarray*}
	\phi_{k+1} \leq  \phi_k  - \frac{\gamma}{2}\norm{ \nabla f(x_k) }^2 - \frac{\gamma}{4}\norm{ v_k  }^2.
\end{eqnarray*}	
This  concludes the proof in case (1).

\subsection{Case (2): $\| \nabla f(x_{k}) - v_{k-1}\| \leq \tau$} 
Suppose that  $\| \nabla f(x_k) - v_{k-1}\| \leq \tau$. Then, by using \eqref{eqn:ineq_v_k} and by the fact that $\eta_k = \eta = 1$, we have 
\begin{align}\label{eqn:v_k_off_centralized}
    v_k = \nabla f(x_k).
\end{align}
 Therefore, single-node \algname{Clip21-GD} described in Algorithm \ref{alg:Clip21-GD} with $n=1$ reduces to classical gradient descent at step $k$. 
 From the definition of $\phi_{k}$ and Lemma \ref{lemma:descent_li2021page},
\begin{eqnarray*}
	\phi_{k+1} 
	& = & f(x_{k+1}) - f_{\rm inf} \\
	& \overset{\eqref{eq:descent_ineq}}{\leq} & f(x_k)- f_{\rm inf}- \frac{\gamma}{2}\norm{ \nabla f(x_k) }^2 - \left( \frac{1}{2\gamma} - \frac{L}{2} \right)\| x_{k+1} - x_k \|^2 + \frac{\gamma}{2}\norm{ \nabla f(x_k) - v_{k}   }^2 \\
	& {\leq} & \phi_k - \frac{\gamma}{2}\norm{ \nabla f(x_k) }^2 - \left( \frac{1}{2\gamma} - \frac{L}{2} \right)\| x_{k+1} - x_k \|^2.
\end{eqnarray*}
If $\gamma \leq 1/L$, then 
\begin{align}\label{eqn:MainIneq_case1_singlenode}
 \phi_{k+1}
	& \leq \phi_k - \frac{\gamma}{2}\norm{ \nabla f(x_k) }^2.
\end{align}
From \eqref{eqn:MainIneq_case1_singlenode}, 
we get $0\leq \phi_k \leq \phi_0$ and
\begin{align*}
	\norm{ \nabla f(x_k) } 
	& \leq  \sqrt{\norm{ \nabla f(x_k) }^2}  \leq  \sqrt{\frac{2}{\gamma}[\phi_k - \phi_{k+1}]}  \leq \sqrt{\frac{2}{\gamma}\phi_k} \leq \sqrt{\frac{2}{\gamma}\phi_0}.
\end{align*}

Finally, we will show that $\| \nabla f(x_k) - v_{k-1}\| \leq \tau$ implies $\| \nabla f(x_{k+1}) - v_{k}\| \leq \tau$. 
In this case, 
\begin{eqnarray*}
	\| \nabla f(x_{k+1}) - v_{k} \|
	& \leq & L \norm{ x_{k+1} - x_{k}    } \\ 
	& \mathop{=}\limits^{\eqref{eqn:v_k_off_centralized}}  & L \gamma \norm{  \nabla f(x_{k})    }\\
	& \mathop{\leq}\limits^{\eqref{eq:triangle} }  & L\gamma\norm{ \nabla f(x_{k}) - v_{k-1}  } + L \gamma \norm{ v_{k-1}  } \\
	& \leq  & L\gamma \tau + L \gamma \norm{ v_{k-1}  } .
\end{eqnarray*}

If $\gamma \leq \frac{1-1/\sqrt{2}}{L \left(1 + \sqrt{1 +  2\beta_1} \right)}$ where
$\beta_1 = \frac{ (1-\eta)^2(1+2/\eta)}{[1-(1-\eta)(1-\eta/2)]}$,  then  
from Claim~\ref{claim:fu9o} (clipping active at $k-1$) and from \eqref{eqn:MainIneq_case1_singlenode} (clipping inactive at $k-1$)
\begin{align*}
	\| v_{k-1}\| \leq \max \left( \sqrt{\frac{4}{\gamma}\phi_0} , \sqrt{\frac{2}{\gamma}\phi_0} \right) \leq \sqrt{\frac{4}{\gamma}\phi_0}.
\end{align*}
Note that for $k = 0$, we have $\norm{v_0} \leq \sqrt{4\phi_0/\gamma}$ due to $v_{-1} = 0$. Therefore, 
\begin{align}
	\| \nabla f(x_{k+1}) - v_{k} \|
	&  \leq L\gamma \tau+ 2L  \sqrt{ \gamma \phi_0}. \label{eq:derivation_for_proposition_2}
\end{align}
In conclusion,  $\| \nabla f(x_k) - v_{k-1}\| \leq \tau$ implies $\| \nabla f(x_{k+1}) - v_{k}\| \leq \tau$ if $\gamma$ satisfies 
\begin{align}\label{eqn:stepsize_singlenode_case1_Lemma1}
	L\gamma\tau \leq \frac{\tau}{2} \quad \text{and} \quad  2L  \sqrt{ \gamma \phi_0} \leq \frac{\tau}{2}.
\end{align}
From  Claim~\ref{claim:080fd-=9090}, 
we can express  step-size conditions  \eqref{eqn:stepsize_singlenode_case1_Lemma1}
equivalently as:
\begin{align*}
	\gamma \leq \frac{1}{2L} \quad \text{and} \quad \gamma \leq \frac{\tau^2}{4L^2 \left[ \sqrt{\FF} + \sqrt{\FF + \frac{\tau\GG}{\sqrt{2\eta} L}} \right]^2},
\end{align*}	
where $\eta=\frac{\tau}{\| \nabla f(x_0) \|}$, $\FF = f(x_0)-f_{\rm inf}$ and $\GG = \vert \| \nabla f(x_0) \| - \tau \vert$.
Putting all the conditions on $\gamma$ together, we obtain the results.

\clearpage
\section{Proof of Proposition \ref{prop:singlenode-clippedEF21}}

If $\| \nabla f(x_0) \| \leq \tau$, then from \eqref{eq:derivation_for_proposition_2} and step-size condition from Theorem \ref{thm:singlenode-clippedEF21}, we prove that the clipping operator is always turned off for all $k \geq 0$, i.e., $\| \nabla f(x_k)-v_{k-1}\|\leq \tau$ implying that $v_k = \nabla f(x_k)$ for all $k\geq 0$.

If $\| \nabla f(x_0) \| > \tau$, then the clipping operator is turned on at the beginning. Moreover, for all $k \geq 0$ such that $\| \nabla f(x_k)-v_{k-1}\|>\tau$ (clipping is turned on), we have from the derivation of \eqref{eq:decrease_of_the_norm_diff_single_node} and the step-size condition of Theorem \ref{thm:singlenode-clippedEF21} 
\begin{align*}
	\| \nabla f(x_{k}) - v_{k-1} \|
	& \leq  \| \nabla f(x_{k-1}) - v_{k-2} \| - \frac{\tau}{2} \leq \ldots \leq \| \nabla f(x_0) \| - k \frac{\tau}{2}.
\end{align*}
Therefore, the situation when $ \tau <	\| \nabla f(x_{k}) - v_{k-1} \|$ is possible only for $0\leq k < k^\star$ with $k^\star=\frac{2}{\tau}(\| \nabla f(x_0) \|-\tau)+1$. After that, the clipping operator always turns off, i.e.,  $\| \nabla f(x_{k}) - v_{k-1} \| \leq \tau$ for $k \geq k^\star$.

\clearpage
\section{Proof of Theorem \ref{thm:singlenode-clippedEF21}}
Let $\hat x_K$ be selected uniformly at random from $\{x_0,x_1,\ldots,x_{K-1}\}$. Then,
\begin{align*}
\Exp{\| \nabla f(\hat x_K) \|^2} & = \frac{1}{K}\sum_{k=0}^{K-1} \norm{ \nabla f(x_k) }^2.
\end{align*}
From Theorem \ref{thm:singlenode-clippedEF21}, 
\begin{align*}
\Exp{\| \nabla f(\hat x_K) \|^2} & \leq \frac{2}{\gamma}\frac{1}{K}\sum_{k=0}^{K-1} (\phi_k - \phi_{k+1}) = \frac{2(\phi_0 - \phi_K)}{\gamma K} \leq \frac{2\phi_0}{\gamma K}.
\end{align*}

Next, we consider the case when
\begin{equation*}
 \gamma =	\min\left( \frac{1}{L(1 + \sqrt{1 + 2\beta_1})} , \frac{\tau^2}{4L^2 \left[ \sqrt{\FF} + \sqrt{\beta_2} \right]^2} \right).
\end{equation*}
First, we have
\begin{align*}
	\beta_1 & = \frac{ (1-\eta)^2(1+\nicefrac{2}{\eta})}{1-(1-\eta)(1-\nicefrac{2}{\eta})} \leq \frac{1 + \nicefrac{2}{\eta}}{1 - (1-\eta)} = \frac{2}{\eta^2} + \frac{1}{\eta} = \cO\left(\frac{1}{\eta^2}\right) = \cO\left(1 + \frac{\|\nabla f(x_0)\|^2}{\tau^2}\right),\\
	\beta_2 & = \FF + \frac{\tau\GG}{\sqrt{2\eta} L} \leq \FF + \frac{\tau\max(\|\nabla f(x_0)\|, \tau)}{\sqrt{2\max\left(1, \frac{\tau}{\|\nabla f(x_0)\|}\right)} L} = \FF + \frac{\tau \|\nabla f(x_0)\|\sqrt{\max\left(1, \frac{\tau}{\|\nabla f(x_0)\|}\right)}}{\sqrt{2}L} \\
	&= \cO\left(\FF + \frac{\tau\|\nabla f(x_0)\|}{L} + \frac{\tau^{\nicefrac{3}{2}}\sqrt{\|\nabla f(x_0)\|}}{L}\right).
\end{align*}
Using this, we estimate $\nicefrac{1}{\gamma}$ as
\begin{align*}
	\frac{1}{\gamma} &= \max\left(L(1+\sqrt{1+2\beta_1}), \frac{4L^2\left[\sqrt{\FF} + \sqrt{\beta_2}\right]^2}{\tau^2}\right)\\
	&= \cO\left( L\left(1 + \frac{\|\nabla f(x_0)\|}{\tau}\right) + \frac{L^2\FF}{\tau^2} + \frac{L\|\nabla f(x_0)\|}{\tau} + \frac{L \sqrt{\|\nabla f(x_0)\|}}{\sqrt{\tau}}\right)\\
	&= \cO\left( L\left(1 + \frac{\|\nabla f(x_0)\|}{\tau}\right) + \frac{L^2\FF}{\tau^2}\right).
\end{align*}
Therefore, since
\begin{align*}
	\frac{A}{\gamma} = \frac{1}{2[1-(1-\eta)(1-\eta/2)]} = \cO\left(\frac{1}{\eta}\right) = \cO\left(1 + \frac{\|\nabla f(x_0)\|}{\tau}\right)
\end{align*}
and
\begin{align*}
	\|\nabla f(x_0) - v_{-1}\|^2 = \|\nabla f(x_0)\|^2, %\leq 2L\left(f(x_0) - f_{\rm inf}\right) = \cO(L\FF),
\end{align*}
we have
\begin{align*}
	\Exp{\| \nabla f(\hat x_K) \|^2} & \leq \frac{2\phi_0}{\gamma K} = \frac{2\left(f(x_0) - f_{\rm inf} + A \|\nabla f(x_0) - v_{-1}\|^2\right)}{\gamma K}\\
	& = \cO\left(\frac{\left(1 + \frac{\|\nabla f(x_0)\|}{\tau}\right)\max(L\FF, \| \nabla f(x_0) \|^2 )+ \frac{L^2(\FF)^2}{\tau^2}}{K}\right),
\end{align*}
which concludes the proof.

%%%%%%%%%%%%%%%%%%%%%%%%%%%%%%%%%%%%%%%%%%%%%%%%%%%%%%%%%%%%%%%%%%%%%%%%%%%%%%%
% MULTI NODE CASE
%%%%%%%%%%%%%%%%%%%%%%%%%%%%%%%%%%%%%%%%%%%%%%%%%%%%%%%%%%%%%%%%%%%%%%%%%%%%%%%
\section{Multi-node \algname{Clip21-GD}}
In this section we analyze the convergence for multi-node \algname{Clip21-GD} described in Algorithm \ref{alg:Clip21-GD}. Its update can be expressed as \eqref{eq:8y98fd_08yfd}, where $v_k =\frac{1}{n}\sum_{i=1}^n v_k^i$ and 
\begin{align}\label{eqn:v_k_distributed_equi}
	v_k^i = (1-\eta_{k}^i) v_{k-1}^i + \eta_k^i \nabla f_i(x_k).
\end{align}
Here, $\eta_k^i = \min\left( 1, \frac{\tau}{\norm{ \nabla f_i(x_k) - v_{k-1}^i} } \right)$. 
To facilitate our analysis, 
denote $\mathcal{I}_k$ as the subset from $\{1,2,\ldots,n\}$ such that $\norm{   \nabla f_i(x_k) - v_{k-1}^i }>\tau$.
Recall that the Lyapunov function is 
\begin{align}\label{eqn:phi_multinode_noDPnoComp}
    \phi_k = f(x_k)-f^{\text{inf}} + A \frac{1}{n}\sum_{i=1}^n \norm{ \nabla f_i(x_k) - v_k^i  }^2,
\end{align}
where $A= \frac{\gamma}{2[1-(1-\eta)(1-\eta/2)]}$.

\subsection{Claims}

We first establish several simple but helpful results.
\begin{claim}\label{claim:bound_diff_case_on_multinode}
Let each $f_i$ have $L_i$-Lipschitz gradient. Then,
for $ k \geq 0$, 
\begin{align}
    \norm{\nabla f_i(x_{k+1})-v^i_k} \leq \max\left\{ 0 , \norm{\nabla f_i(x_k)-v^i_{k-1}} - \tau \right\} +  L_{\max} \gamma \norm{   v_k }.
\end{align}
\end{claim}
\begin{proof}
From the definition of the Euclidean norm, 
\begin{eqnarray*}
    \norm{\nabla f_i(x_{k+1}) - v_k^i}
    & \mathop{\leq}\limits^{\eqref{eq:triangle} } & \norm{  \nabla f_i(x_k) - v_k^i   }  + \norm{\nabla f_i(x_{k+1}) - \nabla f_i(x_k)} \\ 
    &  \mathop{=}\limits^{ \eqref{eqn:v_k_distributed_equi} } & \norm{   \nabla f_i(x_k) - v_{k-1}^i - \clip_\tau(\nabla f_i(x_k) - v_{k-1}^i)    }  
 \\
 &&+ \norm{   \nabla f_i(x_{k+1}) -  \nabla f_i(x_k)    } \\ 
    & \overset{(\text{Lemma~\ref{lem:clip}(ii)-(iii)})}{\leq} &  \max\{ 0 , \norm{\nabla f_i(x_k) - v_{k-1}^i} -\tau \} +  \norm{   \nabla f_i(x_{k+1}) -  \nabla f_i(x_k)    } \\
    & \overset{\eqref{eq:L-smooth}}{\leq}  &  \max\{ 0 , \norm{\nabla f_i(x_k) - v_{k-1}^i} -\tau \} +  L_{\max}\norm{   x_{k+1} -  x_k    } \\
    & \overset{\eqref{eq:8y98fd_08yfd}}{=}  & \max\{ 0 , \norm{\nabla f_i(x_k) - v_{k-1}^i} -\tau \} +  L_{\max} \gamma \norm{   v_k }.
\end{eqnarray*}
\end{proof}

\begin{claim}\label{claim:v_0_multinode}
Let $v_{-1}^i=0$ for all $i$, $\max_i\norm{\nabla f_i(x_0)} := B > \tau$ and $\gamma \leq 2/L$. Then,  
\begin{eqnarray}
    \norm{v_0}  \leq  \sqrt{\frac{4}{\gamma}\phi_0} + 2(B-\tau).
\end{eqnarray}
\end{claim}
\begin{proof}
By the fact that $v_0 = \frac{1}{n}\sum_{i=1}^n v^i_0 = \frac{1}{n}\sum_{i=1}^n  \clip_\tau( \nabla f_i(x_0))$, 
\begin{eqnarray*}
    \norm{v_0} 
    & \overset{\eqref{eq:triangle}}{\leq} & \norm{ \nabla f(x_0)} + \norm{\frac{1}{n}\sum_{i=1}^n \clip_\tau( \nabla f_i(x_0))- \nabla f(x_0)} \\
    & \overset{\eqref{eq:triangle}}{\leq} &  \norm{ \nabla f(x_0)} + \frac{1}{n}\sum_{i=1}^n \norm{\clip_\tau( \nabla f_i(x_0))- \nabla f_i(x_0)} \\ 
    & \overset{(\text{Lemma~\ref{lem:clip}(ii)-(iii)})}{\leq} & \norm{ \nabla f(x_0)} + \frac{1}{n}\sum_{i=1}^n \max\{ 0, \norm{ \nabla f_i(x_0)} - \tau  \}. 
\end{eqnarray*}    
If $\max_i\norm{\nabla f_i(x_0)} > \tau$ and $\gamma \leq 2/L$, then 
\begin{eqnarray*}
    \norm{v_0}  &\leq& \norm{ \nabla f(x_0)} + B - \tau \\
    &\leq& \norm{ \nabla f(x_0)} + 2(B - \tau) \\
    & \overset{\eqref{eq:ssss}}{\leq} & \sqrt{2L[f(x_0)-f^{\inf}]} + 2(B - \tau) \\ 
    & \overset{\eqref{eqn:phi_multinode_noDPnoComp}}{\leq}  & \sqrt{2L\phi_0} + 2(B - \tau) \\ 
    & \leq & \sqrt{\frac{4}{\gamma}\phi_0} + 2(B-\tau).
\end{eqnarray*}
\end{proof}

\begin{claim}\label{claim:v_k_multinode}
Fix $k\geq 1$. Let $f$ have $L$-Lipschitz gradient. Also suppose that  $\norm{\nabla f_i(x_k) - v_{k-1}^i} \leq {\max}_i \norm{\nabla f_i(x_0)} := B > \tau$ for $i\in\mathcal{I}_k$,  $\norm{v_{k-1}}\leq \sqrt{\frac{4}{\gamma}\phi_0} + 2(B-\tau)$, $\norm{\nabla f(x_{k-1})} \leq \sqrt{\frac{2}{\gamma}\phi_0}$, and $\gamma \leq \frac{1-\frac{1}{\sqrt{2}}}{L}$. Then, 
\begin{eqnarray}
    \norm{v_k} 
    & \leq & \sqrt{\frac{4}{\gamma}\phi_0} + 2(B-\tau).
\end{eqnarray}
\end{claim}
\begin{proof}
From the definition of the Euclidean norm and by the fact that $v_k = \frac{1}{n}\sum_{i=1}^n v^i_k$,  
\begin{eqnarray*}
	\norm{v_k}
	& \overset{\eqref{eqn:v_k_distributed_equi}}{=} &  \norm{ \frac{1}{n}\sum_{i=1}^n v_{k-1}^i + \clip_\tau(\nabla f_i(x_k) - v_{k-1}^i) } \\
	& = &\norm{ \frac{1}{n}\sum_{i=1}^n \nabla f_i(x_k) + [\clip_\tau(\nabla f_i(x_k) - v_{k-1}^i) - (\nabla f_i(x_k) - v_{k-1}^i)] } \\
	& \overset{\eqref{eq:triangle}}{\leq} &\norm{ \nabla f(x_k) } + \frac{1}{n}\sum_{i=1}^n \norm{ \clip_\tau(\nabla f_i(x_k) - v_{k-1}^i) - (\nabla f_i(x_k) - v_{k-1}^i)} \\
	%& \leq & \norm{  \nabla f(x_k)  } + \frac{1}{n}\sum_{i=1}^n [0\cdot \mathbbm{1}(i\in\mathcal{I}_k^\prime) + (\norm{   \nabla f_i(x_k) - v_{k-1}^i } - \tau) \cdot \mathbbm{1}(i\in\mathcal{I}_k)] \\
	& \leq &\norm{  \nabla f(x_k)  } + \frac{1}{n}\sum_{i=1}^n \max\{ 0, \norm{\nabla f_i(x_k) - v_{k-1}^i} - \tau \}.
\end{eqnarray*}

If $\norm{\nabla f_i(x_k) - v_{k-1}^i} \leq {\max}_i \norm{\nabla f_i(x_0)} := B > \tau$ for $i\in\mathcal{I}_k$, then 
\begin{eqnarray*}
    \norm{v_k} 
    & \leq &  \norm{\nabla f(x_k)} + B-\tau \\
    & \overset{\eqref{eq:triangle}}{\leq} & \norm{\nabla f(x_{k-1})} + \norm{\nabla f(x_k)-\nabla f(x_{k-1})} + B-\tau \\ 
    & \overset{\eqref{eq:L-smooth}}{\leq} & \norm{\nabla f(x_{k-1})} + L\norm{x_k-x_{k-1}} + B-\tau\\ 
    & \overset{\eqref{eq:8y98fd_08yfd}}{=} & \norm{\nabla f(x_{k-1}))} + L\gamma\norm{v_{k-1}} + B-\tau.
\end{eqnarray*}

If $\norm{v_{k-1}}\leq \sqrt{\frac{4}{\gamma}\phi_0} + 2(B-\tau)$ and $\norm{\nabla f(x_{k-1})} \leq \sqrt{\frac{2}{\gamma}\phi_0}$, then 
\begin{eqnarray*}
    \norm{v_k} 
    & \leq & (L\gamma + 1/\sqrt{2})\sqrt{\frac{4}{\gamma}\phi_0} + (2L\gamma + 1)(B-\tau).
\end{eqnarray*}

If $\gamma \leq \frac{1-\frac{1}{\sqrt{2}}}{L}$, then $\gamma \leq 1/(2L)$ and 
\begin{eqnarray*}
    \norm{v_k} 
    & \leq & \sqrt{\frac{4}{\gamma}\phi_0} + 2(B-\tau).
\end{eqnarray*}
\end{proof}

\begin{claim}\label{claim:stepsize_multinode_on}
If 
\begin{align}\label{eqn:stepsize_multinode_on}
     0 < \gamma \leq \frac{\tau^2}{16L_{\max}^2(\sqrt{F_0} + \sqrt{F_0 + \frac{G_0 \tau}{2\sqrt{2\eta} L_{\max}}})^2},  
\end{align}
where $\eta\eqdef \min\left\{1,\frac{\tau}{\max_i\norm{\nabla f_i(x_0)}} \right\}$, $F_0:=f(x_0)-f_{\inf}$ and $G_0:= \sqrt{\frac{1}{n}\sum_{i=1}^n (\norm{\nabla f_i(x_0)} - \tau)^2}$,
then 
 \begin{eqnarray}
     4L_{\max} \sqrt{\gamma\phi_0} \leq \tau/2.
 \end{eqnarray}   
\end{claim}
\begin{proof}
By the definition of $\phi_0$, 
\begin{eqnarray*}
     2L_{\max} \sqrt{\gamma\phi_0} 
     & = &   2L_{\max} \sqrt{\gamma [f(x_0)-f_{\inf}] + \gamma A \frac{1}{n}\sum_{i=1}^n\norm{  \nabla f_i(x_0) - v^i_0  }^2 } \\
     & \leq & 2L_{\max} \sqrt{\gamma [f(x_0)-f_{\inf}] + \frac{\gamma^2}{2\eta}\frac{1}{n}\sum_{i=1}^n\norm{  \nabla f_i(x_0) - v^i_0  }^2 }.
\end{eqnarray*}
Since $v_{-1}^i=0$ for all $i$, we have $v^i_0=\clip_\tau(\nabla f_i(x_0))$ and 
\begin{eqnarray*}
     2L_{\max} \sqrt{\gamma\phi_0} 
     & \overset{(\text{Lemma~\ref{lem:clip}(ii)-(iii)})}{\leq} &  2L_{\max} \sqrt{\gamma [f(x_0)-f_{\inf}] + \frac{\gamma^2}{2\eta}\frac{1}{n}\sum_{i=1}^n \max\{0, \norm{\nabla f_i(x_0)} - \tau\}^2 } \\
     & \overset{\eqref{eq:sqrt}}{\leq}   & 2L_{\max} \sqrt{\gamma}\sqrt{F_0} + \gamma L_{\max} \sqrt{\frac{2}{\eta}} G_0,
\end{eqnarray*}
where $F_0:=f(x_0)-f_{\inf}$ and $G_0:= \sqrt{\frac{1}{n}\sum_{i=1}^n (\norm{\nabla f_i(x_0)} - \tau)^2}$.

Hence, any $\gamma>0$ satisfying
\begin{eqnarray}\label{eqn:stepsize_choice_multinode}
    4L_{\max} \sqrt{\gamma}\sqrt{F_0} + \gamma 2L_{\max} \sqrt{\frac{2}{\eta}} G_0 \leq \tau/2
\end{eqnarray}
also satisfies $ 4L_{\max} \sqrt{\gamma\phi_0} \leq B/2$. Condition \eqref{eqn:stepsize_choice_multinode} can be rewritten as:
\begin{align*}
    \frac{1}{\sqrt{\gamma}} - \frac{8\sqrt{F_0}}{\tau}L_{\max} - \sqrt{\gamma} \frac{4\sqrt{2}}{\sqrt{\eta}}\frac{G_0}{L_{\max} \tau} L_{\max}^2 \geq 0.
\end{align*}
Finally, by Lemma~\ref{lemma:trick_stepsize_EF21} with $L = L_{\max}$, $\beta_1 = \frac{8\sqrt{F_0}}{\tau}$ and $\beta_2 = \frac{4\sqrt{2}}{\sqrt{\eta}}\frac{G_0}{L_{\max} \tau}$, we have 
\begin{align*}
    0 < \sqrt{\gamma} \leq \frac{\tau}{4L_{\max}(\sqrt{F_0} + \sqrt{F_0 + \frac{G_0 \tau}{2\sqrt{2\eta} L_{\max}}})}.  
\end{align*}
By taking the square, we complete our proof. 
\end{proof}

\subsection{Proof of Theorem \ref{thm:multinode-clippedEF21}}

 We then derive the descent inequality 
\begin{align*}
	\phi_{k+1} \leq \phi_k - \frac{\gamma}{2}\norm{\nabla f(x_k)}^2,
\end{align*}
for two possible cases: (1) when $\vert \mathcal{I}_k \vert >0$ and (2) when $\vert \mathcal{I}_k \vert =0$.

\paragraph{Case (1):  $\vert \mathcal{I}_k \vert >0$.}
To derive the descent inequality we will show by induction the stronger result: $\norm{ \nabla f_i(x_{k}) - v^i_{k-1}  } \leq B - \frac{k\tau}{2}$ for $i\in\mathcal{I}_k$, where $B \eqdef \max_i\norm{ \nabla f_i(x_0) }$ and
\begin{equation}
	\phi_{k} 
	\leq  \phi_{k-1}  - \frac{\gamma}{2}\norm{ \nabla f(x_{k-1})  }^2 - \frac{\gamma}{4}\norm{  v_{k-1}  }^2 \label{eq:detailed_descent_inequality_multinode}
\end{equation}
for any $k \geq 0$, where for notational convenience we assume that $\nabla f_i(x_{-1}) = v^i_{-1} = 0$ and $\phi_{-1} = \phi_0$. The base of the induction is trivial: when $k = 0$ we have   $\norm{ \nabla f_i(x_{k}) - v^i_{k-1}  } =\norm{  \nabla f_i(x_0) - v^i_{-1} } = \norm{ \nabla f_i(x_0) } \leq \max_i \norm{\nabla f_i(x_0)} = B$ where $B>\tau$ for $i\in\mathcal{I}_{-1}$ and \eqref{eq:detailed_descent_inequality_multinode} holds by definition. Next, we assume that for some $k\geq 0$ inequalities $\norm{  \nabla f_i(x_t) - v^i_{t-1} } \leq B$ for $i\in\mathcal{I}_{t-1}$ and \eqref{eq:detailed_descent_inequality_multinode} hold for $t = 0,1,\ldots, k$. 

Let $0 \leq \phi_k \leq \phi_0$ for all $k\geq 0$. If $\gamma \leq 1/L$, then  
\begin{eqnarray*}
    \norm{\nabla f(x_{k-1})}^2 
    & \mathop{\leq}\limits^{\eqref{eq:ssss}} & 2L [f(x_{k-1})-f_{\inf}] \\ 
    & \mathop{\leq}\limits^{\eqref{eq:Lyapunov-89d}}  & 2L \phi_{k-1} \\ 
    & \leq & 2L \phi_0  \\ 
    & \leq & \frac{2}{\gamma}\phi_0.
\end{eqnarray*}
If $\gamma \leq (1-1/\sqrt{2})/L$, then by using the above inequality, and also  Claim~\ref{claim:v_0_multinode} and~\ref{claim:v_k_multinode}, for $k \geq 0$ 
\begin{align*}
    \norm{v_k} \leq  \sqrt{\frac{4}{\gamma}\phi_0} + 2(B-\tau).
\end{align*}

Next, from the above inequality and from Claim \ref{claim:bound_diff_case_on_multinode}, 
for $i \in \mathcal{I}_k$
\begin{eqnarray*}
	\norm{  \nabla f_i(x_{k+1}) - v_k^i  }
 &\leq&  \norm{   \nabla f_i(x_k) - v_{k-1}^i } - \tau + L_{\max} \gamma \norm{  v_k     } \\ 
 &\leq&  \norm{   \nabla f_i(x_k) - v_{k-1}^i } - \tau + L_{\max} \gamma  \sqrt{\frac{4}{\gamma}\phi_0} + 2 L_{\max} \gamma (B-\tau) \\ 
 & = & \norm{   \nabla f_i(x_k) - v_{k-1}^i } - \tau + 2L_{\max}   \sqrt{\gamma\phi_0} + 2 L_{\max} \gamma (B-\tau). 
\end{eqnarray*}

{\bf STEP: Small step-size}

The above inequality and the inductive assumption imply for $i\in\mathcal{I}_{k}$
\begin{align}
	\norm{ \nabla f_i(x_{k+1}) - v^i_{k} }
	\leq \norm{  \nabla f_i(x_{k}) - v^i_{k-1}  }- \frac{\tau}{2} 
  \leq B - \frac{(k+1)\tau}{2} \label{eq:decrease_of_the_norm_diff_multi_node}
\end{align}

if the step-size $\gamma>0$ satisfies 
\begin{align*}
	2 L_{\max} \gamma (B-\tau) \leq  2L_{\max}   \sqrt{\gamma\phi_0}  \quad \text{and} \quad	4L_{\max} \sqrt{\gamma\phi_0} \leq \tau/2.
\end{align*}
By Claim~\ref{claim:stepsize_multinode_on}, this condition can be expressed equivalently as: 
\begin{align*}
	\gamma \leq \frac{\phi_0}{(B-\tau)^2} \quad \text{and} \quad	  \gamma \leq \frac{\tau^2}{16L_{\max}^2(\sqrt{F_0} + \sqrt{F_0 + \frac{G_0 \tau}{2\sqrt{2\eta} L_{\max}}})^2}.
\end{align*}

In conclusion, under this step-size condition, $\norm{ \nabla f_i(x_{k}) - v^i_{k-1} } \leq B$ for $i \in \mathcal{I}_{k-1}$ and $k \geq 0$.
In addition, $\mathcal{I}_{k+1} \subseteq \mathcal{I}_{k}$.

{\bf STEP: Descent inequality}

It remains to prove the descent inequality. 
By the inductive assumption proved above, 
we then have for $i\in\mathcal{I}_{k+1}$,
\begin{align}\label{eqn:eta_multinode_case_on}
    \eta^i_{k+1} = \frac{\tau}{\norm{ \nabla f_i(x_{k+1}) - v^i_{k} }}  \geq \frac{\tau}{B} = \eta. 
\end{align}
Therefore, 
\begin{eqnarray*}
	\norm{   \nabla f_i(x_{k+1}) - v^i_{k+1}    }^2
	& \overset{\eqref{eqn:v_k_distributed_equi}}{=} & 0 \cdot  \mathbbm{1}(i\in\mathcal{I}_{k+1}^\prime) +  (1-\eta^i_{k+1})^2\norm{ \nabla f_i(x_{k+1}) - v^i_{k} }^2 \cdot  \mathbbm{1}(i\in\mathcal{I}_{k+1}) \\
	&  \overset{\eqref{eqn:eta_multinode_case_on}}{\leq}  & 0 \cdot  \mathbbm{1}(i\in\mathcal{I}_{k+1}^\prime) +  (1-\eta)^2\norm{ \nabla f_i(x_{k+1}) - v^i_{k} }^2 \cdot  \mathbbm{1}(i\in\mathcal{I}_{k+1})\\
	& \leq  &   (1-\eta)^2\norm{ \nabla f_i(x_{k+1}) - v^i_{k} }^2 \\
	& \overset{\eqref{eq:Young}}{\leq}  & (1+\theta)(1-\eta)^2\norm{  \nabla f_i(x_{k}) - v^i_{k}   }^2 \\
 &&+ (1+ \nicefrac{1}{\theta})(1-\eta)^2\norm{  \nabla f_i(x_{k+1}) - \nabla f_i (x_{k})  }^2 \\
       & \overset{\eqref{eq:L_i-smooth}}{\leq}   & (1+\theta)(1-\eta)^2\norm{  \nabla f_i(x_{k}) - v^i_{k}   }^2 \\
    &&+ (1+ \nicefrac{1}{\theta})(1-\eta)^2 L_{\max}^2\norm{  x_{k+1} - x_k }^2 
\end{eqnarray*}
where 
$\theta > 0$. Taking $\theta = \nicefrac{\eta}{2}$ and applying inequality $(1-\eta)(1+\nicefrac{\eta}{2}) \leq 1-\nicefrac{\eta}{2}$, we get 
\begin{eqnarray}
	\norm{   \nabla f_i(x_{k+1}) - v^i_{k+1}    }^2
	& \leq &(1-\eta)(1-\nicefrac{\eta}{2})\norm{  \nabla f_i(x_{k}) - v^i_{k}   }^2 \nonumber \\
 &&+ (1+\nicefrac{2}{\eta})(1-\eta)^2 L_{\max}^2 \norm{  x_{k+1} - x_{k} }^2.  \label{eqn:TrickFromEF21_multinode}
\end{eqnarray}
Next, we combine the above inequality with Lemma \ref{lemma:descent_li2021page}:
\begin{eqnarray*}
	\phi_{k+1}  
	& = & f(x_{k+1}) - f_{\inf} + A \frac{1}{n}\sum_{i=1}^n \| \nabla f_i(x_{k+1}) - v^i_{k+1} \|^2\\
	& \mathop{\leq}\limits^{\eqref{eq:descent_ineq}} & f(x_k)- f_{\inf} - \frac{\gamma}{2}\norm{ \nabla f(x_k)  }^2 - \left( \frac{1}{2\gamma} - \frac{L}{2} \right)\norm{ x_{k+1} - x_k }^2  \\
	&& + \frac{\gamma}{2}\frac{1}{n}\sum_{i=1}^n\norm{ \nabla f_i(x_k) - v^i_{k} }^2 + A\frac{1}{n}\sum_{i=1}^n\norm{   \nabla f_i(x_{k+1}) - v^i_{k+1}    }^2  \\
	& = & \phi_k - \frac{\gamma}{2}\norm{ \nabla f(x_k)  }^2 - \left( \frac{1}{2\gamma} - \frac{L}{2} \right)\norm{ x_{k+1} - x_k }^2  \\
	&& + \left(\frac{\gamma}{2} -A\right)\frac{1}{n}\sum_{i=1}^n\norm{ \nabla f_i(x_k) - v^i_{k} }^2 + A\frac{1}{n}\sum_{i=1}^n\norm{   \nabla f_i(x_{k+1}) - v^i_{k+1}    }^2  \\
	& \mathop{\leq}\limits^{\eqref{eqn:TrickFromEF21_multinode}} & \phi_k - \frac{\gamma}{2}\norm{ \nabla f(x_k)  }^2 - \left( \frac{1}{2\gamma} - \frac{L}{2} -  A  \left( 1 + \frac{2}{\eta}\right) (1-\eta)^2 (L_{\max})^2 \right)\norm{ x_{k+1} - x_k }^2 \\
	&& + \left( \frac{\gamma}{2} + A(1-\eta)(1-\eta/2) - A\right)\frac{1}{n}\sum_{i=1}^n\norm{ \nabla f_i(x_{k}) - v^i_{k}   }^2  .
\end{eqnarray*}

Since $A = \frac{\gamma}{2[1-(1-\eta)(1-\eta/2)]}$, we get
\begin{align*}
	\phi_{k+1} \leq \phi_k - \frac{\gamma}{2}\norm{ \nabla f(x_k)  }^2 - \frac{1}{2\gamma}\left(1 - \gamma L - \gamma^2 \cdot \frac{ (1+2/\eta)(1-\eta)^2}{[1-(1-\eta)(1-\eta/2)]}  (L_{\max})^2\right)\norm{ x_{k+1} - x_k }^2.
\end{align*}

If the step-size $\gamma$ satisfies 
\begin{align*}
	0 < \gamma \leq \frac{1}{L \left(1 + \sqrt{1 +  2\frac{ (1-\eta)^2(1+2/\eta)}{[1-(1-\eta)(1-\eta/2)]} \left( \frac{L_{\max}}{L} \right)^2 } \right)},
\end{align*}
then from  Lemma \ref{lemma:trick_stepsize_EF21} with $L=1$, $\beta_1=2L$ and $\beta_2 = 2\frac{ (1-\eta)^2(1+2/\eta)}{[1-(1-\eta)(1-\eta/2)]}(L_{\max})^2$, this condition implies $1 - \gamma L - \gamma^2 \cdot \frac{ (1-\eta)^2(1+2/\eta)}{[1-(1-\eta)(1-\eta/2)]}  (L_{\max})^2  \geq \frac{1}{2}$
and thus
\begin{align*}
	\phi_{k+1} \leq \phi_k - \frac{\gamma}{2}\norm{ \nabla f(x_k)  }^2 - \frac{1}{4\gamma}\norm{ x_{k+1} - x_k }^2.
\end{align*}
Since $x_{k+1}-x_k = -\gamma v_k$,
\begin{align}\label{eqn:descentIneq_EF21_multinode}
	\phi_{k+1} 
	& \leq  \phi_k  - \frac{\gamma}{2}\norm{ \nabla f(x_k)  }^2 - \frac{\gamma}{4}\norm{ v_k  }^2. 
\end{align}
This concludes the proof in the case (1).

\paragraph{Case (2): $\vert \mathcal{I}_k \vert=0$.}
Suppose $\vert \mathcal{I}_k \vert=0$. Then, $\norm{ \nabla f_i(x_k) - v^i_{k-1}  } \leq \tau$ for all $i$.
Then, by using \eqref{eqn:v_k_distributed_equi} and by the fact that $\eta_k^i = \eta = 1$, we have $v_k^i = \nabla f_i(x_k)$. Therefore, \algname{Clip21-GD} described in Algorithm \ref{alg:Clip21-GD} reduces to classical gradient descent at step $k$. From the definition of $\phi_{k}$ and Lemma \ref{lemma:descent_li2021page},
\begin{align*}
	\phi_{k+1} 
	& = f(x_{k+1}) - f_{\inf} \\
	& \mathop{\leq}\limits^{\eqref{eq:descent_ineq}}  f(x_k)- f_{\inf}- \frac{\gamma}{2}\norm{ \nabla f(x_k)  }^2 - \left( \frac{1}{2\gamma} - \frac{L}{2} \right)\norm{ x_{k+1} - x_k }^2 + \frac{\gamma}{2}\norm{ \nabla f(x_k) - v_{k}   }^2 \\
	& \leq \phi_k - \frac{\gamma}{2}\norm{ \nabla f(x_k)  }^2 - \left( \frac{1}{2\gamma} - \frac{L}{2} \right)\norm{ x_{k+1} - x_k }^2.
\end{align*}
If $\gamma \leq 1/L$, then 
\begin{align}\label{eqn:MainIneq_case1_multinode}
	\phi_{k+1}
	& \leq \phi_k - \frac{\gamma}{2}\norm{ \nabla f(x_k)  }^2.
\end{align}
Therefore, $0\leq \phi_k \leq \phi_0$ and
\begin{align*}
	\norm{ \nabla f(x_k) } 
	& \leq  \sqrt{\norm{ \nabla f(x_k)  }^2}  \leq  \sqrt{\frac{2}{\gamma}[\phi_k - \phi_{k+1}]}  \leq \sqrt{\frac{2}{\gamma}\phi_k} \leq \sqrt{\frac{2}{\gamma}\phi_0}.
\end{align*}

Finally, we will show that $\norm{ \nabla f_i(x_k) - v^i_{k-1}  } \leq \tau$ for all $i$ implies $\| \nabla f_i(x_{k+1}) - v^i_{k}\| \leq \tau$ for all $i$. Indeed, in this case, we have $v^i_k = \nabla f_i(x_k)$ and 
\begin{eqnarray*}
	\norm{ \nabla f_i(x_{k+1}) - v^i_{k} }
	& \overset{\eqref{eq:L_i-smooth}}{\leq} &  L_{\max} \norm{ x_{k+1} - x_{k}    } \\ 
	%& = L_{\max} \gamma \norm{ v_k    } \\
	& = & L_{\max} \gamma \norm{  \nabla f(x_{k})    }\\
	& \overset{\eqref{eq:triangle}}{\leq} & L_{\max} \gamma\norm{ \nabla f(x_{k}) - v_{k-1}  } + L_{\max} \gamma \norm{ v_{k-1}  } \\
	& \overset{\eqref{eq:triangle}}{\leq} & L_{\max}\gamma \frac{1}{n} \sum_{i=1}^n \norm{ \nabla f_i(x_k) - v^i_{k-1} } + L_{\max}\gamma \|v_{k-1}\| \\
	& \leq & L_{\max} \gamma \tau +L_{\max}\gamma \norm{ v_{k-1}  } .
\end{eqnarray*}

If $\gamma \leq\frac{1}{L (1 + \sqrt{1 +  2\beta_1 )}}$ where
$\beta_1 = \frac{ (1-\eta)^2(1+2/\eta)}{[1-(1-\eta)(1-\eta/2)]} \left( \frac{L_{\max}}{L} \right)^2$,  then  from Claim~\ref{claim:v_k_multinode} and \eqref{eqn:MainIneq_case1_multinode},
\begin{align*}
	\norm{v_{k-1}}  \leq \sqrt{\frac{4}{\gamma}\phi_0} + 2(B - \tau).
\end{align*}
Note that for $k = 0$ the above inequality also holds due to $v_{-1} = 0$ (see Claim~\ref{claim:v_0_multinode}). Hence, for all $i$
\begin{eqnarray}
	\norm{ \nabla f_i(x_{k+1}) - v^i_{k} }
	&  \leq & L_{\max} \gamma \tau+ 2L_{\max} \sqrt{ \gamma \phi_0} + 2L_{\max}\gamma(B-\tau) \nonumber\\
 & \leq & 2L_{\max} \sqrt{ \gamma \phi_0}  + 2L_{\max}\gamma B \label{eq:derivation_for_proposition_2_multinode}
\end{eqnarray}
In conclusion, $ \norm{ \nabla f_i(x_{k+1}) - v^i_{k} } \leq \tau$ is true for all $i$ if $\gamma$ satisfies 
\begin{align}\label{eqn:stepsize_multinode_case1_Lemma1}
	2L_{\max}\gamma B\leq \frac{\tau}{2} \quad \text{and} \quad  2L_{\max} \sqrt{ \gamma \phi_0} \leq \frac{\tau}{2}.
\end{align}
From Claim~\ref{claim:stepsize_multinode_on}, the condition \eqref{eqn:stepsize_multinode_case1_Lemma1} is hence satisfied when 
\begin{align*}
	\gamma \leq \frac{\tau}{4BL_{\max}} \quad \text{and} \quad \gamma \leq  \frac{\tau^2}{16L_{\max}^2(\sqrt{F_0} + \sqrt{F_0 + \frac{G_0 \tau}{2\sqrt{2\eta} L_{\max}}})^2},
\end{align*}	
where $\eta\eqdef \min\left\{1,\frac{\tau}{\max_i\norm{\nabla f_i(x_0)}} \right\}$, $F_0:=f(x_0)-f_{\inf}$ and $G_0:= \sqrt{\frac{1}{n}\sum_{i=1}^n (\norm{\nabla f_i(x_0)} - \tau)^2}$.
Putting all the conditions on $\gamma$ together, we obtain the results.

\newpage 
\section{Proof of Proposition \ref{prop:multinode-clippedEF21}}

 If $\| \nabla f_i(x_0) \| \leq \tau$, then from \eqref{eq:derivation_for_proposition_2_multinode} and step-size condition from Lemma~\ref{lem:multinode-clippedEF21}, we prove that the clipping operator is always turned off for all $k \geq 0$, i.e., $\| \nabla f_i(x_k)-v^i_{k-1}\|\leq \tau$ implying that $v^i_k = \nabla f_i(x_k)$ for all $k\geq 0$.

If $\| \nabla f_i(x_0) \| > \tau$, then the clipping operator is turned on at the beginning. Moreover, for all $k \geq 0$ such that $\| \nabla f_i(x_k)-v^i_{k-1}\|>\tau$ (clipping is turned on), we have from the derivation of \eqref{eq:decrease_of_the_norm_diff_multi_node} and the step-size condition of Lemma~\ref{lem:multinode-clippedEF21} 
\begin{align*}
	\| \nabla f_i(x_{k}) - v^i_{k-1} \|
	& \leq  \| \nabla f_i(x_{k-1}) - v^i_{k-2} \| - \tau/2 \leq \ldots \leq \| \nabla f_i(x_0) \| - k\tau/2.
\end{align*}
Therefore, the situation when $ \tau <	\| \nabla f_i(x_{k}) - v^i_{k-1} \|$ is possible only for $0\leq k < k^\star$ with $k^\star=\frac{2}{\tau}(\| \nabla f_i(x_0) \|-\tau)+1$. After that, the clipping operator always turns off, i.e.,  $\| \nabla f_i(x_{k}) - v^i_{k-1} \| \leq \tau$ for $k \geq k^\star$.

\newpage
\section{Proof of Theorem \ref{thm:multinode-clippedEF21}}
Let $\hat x_K$ be selected uniformly at random from $\{x_0,x_1,\ldots,x_{K-1}\}$. Then,
\begin{align*}
	\Exp{ \| \nabla f(\hat x_K) \|^2 } & = \frac{1}{K}\sum_{k=0}^{K-1} \norm{ \nabla f(x_k) }^2.
\end{align*}
From Lemma~\ref{lem:multinode-clippedEF21}, 
\begin{align*}
	\Exp{ \| \nabla f(\hat x_K) \|^2 } & \leq \frac{2}{\gamma}\frac{1}{K}\sum_{k=0}^{K-1} (\phi_k - \phi_{k+1}) = \frac{2(\phi_0 - \phi_K)}{\gamma K} \leq \frac{2\phi_0}{\gamma K}.
\end{align*}

Next, we consider the case when $\phi_0$ is very large. Then, 
\begin{equation*}
	\gamma \leq \min\left( \frac{1}{L}\frac{1}{1+\sqrt{1+2\beta_1}} , \frac{\tau^2}{16L_{\max}^2 \left[ \sqrt{\FF} + \sqrt{\beta_2} \right]^2} \right).	
\end{equation*}
First, we have $\frac{1}{\eta} = \mathcal{O}( 1 + \frac{\max_i \|\nabla f_i(x_0)\|}{\tau})$ and 
\begin{align*}
	\beta_1 & \leq  \frac{ (1-\eta)^2(1+\nicefrac{2}{\eta})}{1-(1-\eta)(1-\nicefrac{2}{\eta})} \leq \frac{1 + \nicefrac{2}{\eta}}{1 - (1-\eta)} = \frac{2}{\eta^2} + \frac{1}{\eta} = \cO\left(\frac{1}{\eta^2}\right) = \cO\left(1 + \frac{(\max_i \|\nabla f_i(x_0)\|)^2}{\tau^2}\right),\\
	\beta_2 & = \FF +  \frac{\tau\GG}{\sqrt{2\eta} L_{\max}} \leq \FF + \frac{1}{\sqrt{2}}\frac{\tau\max(\tau, \max_i\|\nabla f_i(x_0)\|)}{\sqrt{\max\left(1, \frac{\tau}{\max_i\|\nabla f_i(x_0)\|}\right)} L_{\max}} \\ & = \FF + \frac{1}{\sqrt{2}} \frac{ \tau \max_i\|\nabla f_i(x_0)\|\sqrt{\max\left(1, \frac{\tau}{\max_i \|\nabla f_i(x_0)\|}\right)}}{L_{\max}} \\
	&= \cO\left(\FF + \frac{\tau\max_i \|\nabla f_i(x_0)\|}{L_{\max}} + \frac{\tau^{3/2}{\max_i \|\nabla f_i(x_0)\|}^{1/2}}{L_{\max}}\right).
\end{align*}
Using this, we estimate $\nicefrac{1}{\gamma}$ as
\begin{align*}
	& \frac{1}{\gamma} =  \max\left( L(1+\sqrt{1+2\beta_1}), \frac{16(L_{\max})^2\left[\sqrt{\FF} + \sqrt{\beta_2}\right]^2}{\tau^2}\right)\\
	&= \cO\left( L\left(1 + \frac{\max_i \|\nabla f_i(x_0)\|}{\tau}\right) + \frac{(L_{\max})^2\FF}{\tau^2} + T\right)\\
	&= \cO\left( \max(L,L_{\max})\left(1 + \sqrt{\tau \max_i \norm{\nabla f_i(x_0)}}\right) + \frac{(L_{\max})^2\FF}{\tau^2}\right),
\end{align*}
where $T=\frac{L_{\max}\max_i \|\nabla f_i(x_0)\|}{\tau} +  L_{\max} \sqrt{\tau \max_i \norm{\nabla f_i(x_0) }}$.
Therefore, since
\begin{align*}
	\frac{A}{\gamma} = \frac{1}{2[1-(1-\eta)(1-\eta/2)]} = \cO\left(\frac{1}{\eta}\right) = \cO\left(1 + \frac{\max_i\|\nabla f_i(x_0)\|}{\tau}\right)
\end{align*}
and
\begin{align*}
	\frac{1}{n}\sum_{i=1}^n \|\nabla f_i(x_0) - v^i_{-1}\|^2 = \frac{1}{n} \sum_{i=1}^n  \|\nabla f_i(x_0)\|^2 = \cO( \max_i \| \nabla f_i(x_0) \|^2 ),
\end{align*}
we have
\begin{align*}
	\Exp{ \| \nabla f(\hat x_K) \|^2 } & \leq \frac{2\phi_0}{\gamma K} = \frac{2\left(f(x_0) - f_{\rm inf} + A \frac{1}{n}\sum_{i=1}^n \|\nabla f_i(x_0) - v^i_{-1}\|^2\right)}{\gamma K}\\
	& = \cO\left(\frac{\left(1 + \frac{\max_i \|\nabla f_i(x_0)\|}{\tau}\right) \max( \FF \max(L_{\max},L) , \max_i \| \nabla f_i(x_0) \|^2) + \frac{(L_{\max})^2(\FF)^2}{\tau^2}}{K}\right).
\end{align*}

%%%%%%%%%%%%%%%%%%%%%%%%%%%%%%%%%%%%%%%%%%%%%%%%%%%%%%%%%%%%%%%%%%%%%%%%%%%%%%%
% MULTI NODE CASE + DP
%%%%%%%%%%%%%%%%%%%%%%%%%%%%%%%%%%%%%%%%%%%%%%%%%%%%%%%%%%%%%%%%%%%%%%%%%%%%%%%

%%%%%%%%%%%%%%%%%%%%%%%%%%%%%%%%%%%%%%%%%%%%%%%%%%%%%%%%%%%%%%%%%%%%%%%%%%%%%%%
%%%%%%%%%%%%%%%%%%%%%%%%%%%%%%%%%%%%%%%%%%%%%%%%%%%%%%%%%%%%%%%%%%%%%%%%%%%%%%%
\clearpage

%%%%%%%%%%%%%%%%%%%%%%%%%%%%%%%%%%%%%%%%%%%%%%%%%%%%%%%%%%%%%%%%%%%%%%%%%%%%%%%
%%%%%%%%%%%%%%%%%%%%%%%%%%%%%%%%%%%%%%%%%%%%%%%%%%%%%%%%%%%%%%%%%%%%%%%%%%%%%%%
\section{Adding Gaussian Noise for DP Guarantees}\label{sec:DP}
%%%%%%%%%%%%%%%%%%%%%%%%%%%%%%%%%%%%%%%%%%%%%%%%%%%%%%%%%%%%%%%%%%%%%%%%%%%%%%%
%%%%%%%%%%%%%%%%%%%%%%%%%%%%%%%%%%%%%%%%%%%%%%%%%%%%%%%%%%%%%%%%%%%%%%%%%%%%%%%

\begin{algorithm}[t]
	\centering
	\caption{\algname{DP-Clip21-GD} (Error Feedback for  DP Optimization with Clipping)}\label{alg:DP-Clip21-GD}
	\begin{algorithmic}[1]
		\STATE \textbf{Input:} initial iterate $x_{0} \in \R^d$; learning rate $\gamma>0$; initial gradient shifts $v^1_{-1},\dots,v^n_{-1}\in \R^d$; clipping threshold $\tau>0$; variance $\sigma^2>0$; variance bound $\nu>0$
		\FOR{$k=0,1, 2, \dots, K-1 $}
		\STATE Broadcast $x_k$ to all workers 
		\FOR{each worker $i=1,\ldots,n$ in parallel}
		\STATE Sample $\zeta_{k-1}^i\sim \mathcal{N}(0,\sigma^2)$ 
		\STATE Set $z_{k-1}^i =  \clip_{\nu}(\zeta_{k-1}^i)$
			\STATE Compute 	$g_k^i = \clip_\tau( \nabla f_i(x_{k}) - v^i_{k-1} ) + z_{k-1}^i$
			\STATE Update $v^i_{k} = v^i_{k-1} + g_k^i$
		\ENDFOR
		\STATE  $v_{k} = v_{k-1} + \frac{1}{n}\sum_{i=1}^n g_k^i$
		\STATE		  $x_{k+1} = x_k -  \gamma\frac{1 }{n}\sum_{i=1}^n v^i_k$ 
		
		\ENDFOR
	\end{algorithmic}
\end{algorithm}

In this section, we extend \algname{Clip21-GD} to solve the problem under target privacy budget. We call this new method \algname{DP-Clip21-GD} -- it is described in detail as Algorithm \ref{alg:DP-Clip21-GD}.
Notice that \algname{DP-Clip21-GD} with $\sigma^2=0$ reduces to  \algname{Clip21-GD}. 

Given a certain condition on the Gaussian noise, we first derive a privacy guarantee by Theorem 3.4 of \citep{chen2022bounded}, which gives an $\epsilon$-DP guarantee for the bounded Gaussian mechanism, and by the advanced composition theorem in Corollary 3.21 of \citep{dwork2014algorithmic}.
\begin{theorem}[Privacy guarantee]\label{thm:DPclipEF21_privacy}
Let $0<\epsilon<1$, $\delta>0$, $0<\alpha<1$, and $\tau \geq 6\nu \geq 6\sigma$ satisfy 
$$\squeeze \delta \geq \exp\left( -\frac{1}{2K} \left[ \frac{\epsilon}{\max(P_1,P_2)} \right]^2 \right),$$ 
where  $P_1 \eqdef \frac{2\ln(\Delta C(\tau/6, c^\star))}{\alpha}$,  $P_2\eqdef  \frac{72\tau^2}{(1-\alpha)\tau}$,  $K$ is the number of steps  and  $c^\star$ is the solution to the following problem: 
\begin{align*}
	\squeeze	\max_c &  \squeeze \quad	 \Delta C(\sqrt{\min(\sigma^2,\nu^2)},c) \eqdef \frac{C(a,\sqrt{\min(\sigma^2,\nu^2)})}{C(a+c,\sqrt{\min(\sigma^2,\nu^2)})}  \\
	\text{subject to} &\squeeze 	 \quad  0 \leq c \leq b-a \quad \text{and} \quad  \| c \| \leq 2\tau.
\end{align*}
Here, each element of $a\in\mathbb{R}^d$ is $-\nicefrac{\tau}{\sqrt{d}}$, each element of $b\in\mathbb{R}^d$ is $\nicefrac{\tau}{\sqrt{d}}$, and $$\squeeze C(y,\sigma) \eqdef \left(\int_a^b \exp\left(- \frac{\norm{ x-y }^2}{\sigma^2} \right) \;dx \right)^{-1}.$$
Then, \algname{DP-Clip21-GD} is $(\epsilon,\delta)$-differentially private for
\begin{equation}\label{eq:sigma_min}\squeeze	\min(\nu^2,\sigma^2) \geq \frac{12\tau^2\sqrt{2K \ln(1/\delta)}}{(1-\alpha)\epsilon} \eqdef \sigma^2_{\min}(K).
\end{equation}
\end{theorem}

The utility guarantee, presented next, can thus be obtained by substituting $\sigma^2$ from Theorem \ref{thm:DPclipEF21_privacy} into our convergence theorem, which we present in Theorem~\ref{thm:DP-clippedEF21}.

\begin{theorem}[Utility guarantee]\label{thm:DPclipEF21_utility}
Consider the problem of solving 
\eqref{eq:main}. Suppose that each $f_i$ is $L_i$-Lipschitz gradient, and that $f$ is $L$-Lipschitz gradient and satisfies the PŁ condition, i.e., there exists $\mu>0$ such that $$\squeeze f(x)-f_\star\leq \frac{1}{2\mu}\norm{\nabla f(x)}^2, \; \forall x\in \R^d,$$ where $f_\star \eqdef \min_x f(x)$.
Choose $\nu\leq \frac{\tau}{2(\sqrt{2LC}+2)}$   and let $v_{-1}^i=0$ for all $i\in [n]$, $\eta\eqdef\min\left\{1,\frac{\tau}{\max_i \| \nabla f_i(x_0)\|}\right\}$, $\FF\eqdef f(x_0)-f_\star$ and $\GG \eqdef \sqrt{\frac{1}{n}\sum_{i=1}^n (\vert \nabla f_i(x_0) - \tau \vert + \nu)^2 }$.
Choose the privacy variance $\sigma^2$ according to Theorem \ref{thm:DPclipEF21_privacy},
\begin{align} \label{eq:stepsize_DP_clipped_EF21}
\squeeze	\gamma \leq \min\left(  \frac{\eta }{4\mu} , \frac{2\mu}{L_{\max}^2} , \frac{\phi_0}{(B-\tau/2)^2} , \frac{1-1/\sqrt{2}}{L(1+\sqrt{1+8\beta_1})} , \frac{\tau^2}{64 L_{\max}^2 \left[ \sqrt{\FF} + \sqrt{\beta_2} \right]^2}  \right)	
\end{align}
where $\beta_1\eqdef \frac{(1+2/\eta)(1-\eta)(1-\eta/2)}{\eta}\left( \frac{L_{\max}}{L}\right)^2$, and $\beta_2 \eqdef \FF + \frac{\tau\GG}{2\sqrt{2\eta}L_{\max}}$. Then, \algname{DP-Clip21-GD} (described in Algorithm~\ref{alg:DP-Clip21-GD}) satisfies 
\begin{align} %\label{eq:descentIneq_DP_clipped_EF21_utility}
\squeeze	\Exp{\phi_{K}} \leq (1-\gamma\mu)^K\phi_0 + A_3 \sigma^2_{\min}(K),
\end{align}
where $\sigma^2_{\min}(K)$ is as in \eqref{eq:sigma_min},
 $A_3\eqdef \frac{1}{\eta\mu}{2(1+ \frac{2}{\eta})}$ and $$\squeeze\phi_k \eqdef f(x_k)-f_{\star}+\frac{2\gamma}{\eta}\frac{1}{n}\sum \limits_{i=1}^n\norm{ \nabla f_i(x_k) - v_k^i }^2.$$ 
 Here, $0 \leq \phi_k \leq \phi_0 + C \nu^2$ for some $C>0$.
\end{theorem}

\section{Convergence theorem for   \algname{DP-Clip21-GD}}\label{sec_app:DP}

We derive the convergence theorem for  \algname{DP-Clip21-GD} in Section~\ref{sec:DP}.
\begin{theorem}\label{thm:DP-clippedEF21}
	Consider the problem of minimizing $f(x)=\frac{1}{n}\sum_{i=1}^n f_i(x)$.
	Suppose that each $f_i(x)$ is $L_i$-Lipschitz gradient, and that the whole objective $f(x)$ is $L$-Lipschitz gradient and  satisfies the PŁ condition, i.e. there exists $\mu>0$ such that $f(x)-f_\star\leq \norm{\nabla f(x)}^2/(2\mu)$ where $f_\star = \min_x f(x)$.
	Let $v_{-1}^i=0$ for all $i$, $\eta\eqdef\min\left(1,\frac{\tau}{\max_i \| \nabla f_i(x_0)\|}\right)$, $\FF\eqdef f(x_0)-f_\star$ and $\GG \eqdef \sqrt{\frac{1}{n}\sum_{i=1}^n (\vert \nabla f_i(x_0) - \tau \vert + \nu)^2 }$.
	Then,  \algname{DP-Clip21-GD} with $\nu\leq \frac{\tau}{2(\sqrt{2LC}+2)}$ and 
	\begin{align} %\label{eq:stepsize_DP_clipped_EF21}
		\gamma \leq \min\left(  \frac{\eta }{4\mu} , \frac{2\mu}{L_{\max}^2} , \frac{\phi_0}{(B-\tau/2)^2} , \frac{1-1/\sqrt{2}}{L(1+\sqrt{1+8\beta_1})} , \frac{\tau^2}{64 L_{\max}^2 \left[ \sqrt{\FF} + \sqrt{\beta_2} \right]^2}  \right)	
	\end{align}
	where $\beta_1\eqdef \frac{(1+2/\eta)(1-\eta)(1-\eta/2)}{\eta}\left( \frac{L_{\max}}{L}\right)^2$, and $\beta_2 \eqdef \FF + \frac{\tau\GG}{2\sqrt{2\eta}L_{\max}}$ satisfies 
	\begin{align} \label{eq:descentIneq_DP_clipped_EF21}
		\Exp{\phi_{k+1}} \leq (1-\gamma\mu)\Exp{\phi_k} + \frac{2(1+2/\eta)}{\eta}\gamma \min(\nu^2,\sigma^2),
	\end{align}
	where $\phi_k \eqdef f(x_k)-f^{\star}+\frac{2\gamma}{\eta}\frac{1}{n}\sum_{i=1}^n\| \nabla f_i(x_k) - v_k^i \|^2$ and $0\leq \phi_k \leq \phi_0 +C\nu^2$ for some $C>0$.
\end{theorem}

\section{Proof of Theorem~\ref{thm:DP-clippedEF21}}
	The update of $v_k^i$ in \algname{DP-Clip21-GD}	can be equivalently expressed as \eqref{eq:8y98fd_08yfd}, where $v_k =\frac{1}{n}\sum_{i=1}^n v_k^i$ and 
	\begin{align}\label{eqn:v_k_DP_equi}
		v_k^i = (1-\eta_{k}^i) v_{k-1}^i + \eta_k^i \nabla f_i(x_k) + {z}_{k-1}^i,
	\end{align}
	where $\eta_k^i = \min\left( 1, \frac{\tau}{\| \nabla f_i(x_k) - v_{k-1}^i \|} \right)$ and $\norm{z_{k-1}^i} \leq \nu$.
 Recall that the Lyapunov function for analyzing the result is
 \begin{eqnarray}
\phi_k = f(x_k)-f^{\star} + A \frac{1}{n}\sum_{i=1}^n \| \nabla f_i(x_k) - v_k^i \|^2,
 \end{eqnarray}
where $A= \frac{2\gamma}{\eta}$.

\subsection{Basic Claims}

\begin{claim}\label{claim:diff_v_f_DPClip}
Let each $f_i$ have $L_i$-Lipschitz gradient. Then, for $k\geq 0$, 
    \begin{eqnarray}
		\norm{ \nabla f_i(x_{k+1}) - v_k^i } 
		 \leq \max\{ 0, \norm{ \nabla f_i(x_k) - v_{k-1}^i} -\tau \} + \nu + L_{\max} \gamma \norm{v_k}.
	\end{eqnarray}
\end{claim}
\begin{proof}
 From the definition of $v_k^i$, 
 	\begin{eqnarray*}
		\norm{ \nabla f_i(x_{k+1}) - v_k^i } 
		& \overset{\eqref{eq:triangle}}{\leq} & \norm{ \nabla f_i(x_k) - v_k^i } + \norm{ \nabla f_i(x_{k+1}) -  \nabla f_i(x_k)  }  \\
		& = & \norm{ \nabla f_i(x_k) - v_{k-1}^i - \clip_\tau(\nabla f_i(x_k) - v_{k-1}^i) + z_{k-1}^i }  \\
  && +  \norm{ \nabla f_i(x_{k+1}) -  \nabla f_i(x_k)  }  \\
		& \overset{\eqref{eq:triangle}}{\leq} & \norm{ \nabla f_i(x_k) - v_{k-1}^i - \clip_\tau(\nabla f_i(x_k) - v_{k-1}^i)  } \\
  && + \norm{z_{k-1}^i} + \norm{ \nabla f_i(x_{k+1}) -  \nabla f_i(x_k)  } \\
            &  \overset{(\text{Lemma~\ref{lem:clip}(ii)-(iii)})}{\leq} & \max\{ 0, \norm{ \nabla f_i(x_k) - v_{k-1}^i} -\tau \} \\
            && + \norm{z_{k-1}^i} + \norm{ \nabla f_i(x_{k+1}) -  \nabla f_i(x_k)  } \\
            & \overset{\eqref{eq:L_i-smooth} + \eqref{eq:8y98fd_08yfd}}{\leq} &\max\{ 0, \norm{ \nabla f_i(x_k) - v_{k-1}^i} -\tau \} + \nu + L_{\max} \gamma \norm{v_k}.
	\end{eqnarray*}
\end{proof}

\begin{claim}\label{claim:v_0_DPCLip}
Let $v_{-1}^i=0$ for all $i$, $\max_i \norm{\nabla f_i(x_0)} := B > \tau$, and $\gamma \leq 2/L$. Then, 
\begin{eqnarray}
    \norm{v_0}
  \leq \sqrt{\frac{4}{\gamma}\phi_0} + 2(B+\nu-\tau).
\end{eqnarray}
\end{claim}
\begin{proof}
By the fact that $v_0 = \frac{1}{n}\sum_{i=1}^n v_0^i = \frac{1}{n}\sum_{i=1}^n [\clip_\tau(\nabla f_i(x_0)) + z_{-1}^i]$, 
\begin{eqnarray*}
    \norm{v_0}
    & \overset{\eqref{eq:triangle}}{\leq} & \norm{\nabla f(x_0)} + \frac{1}{n}\sum_{i=1}^n \norm{\clip_\tau(\nabla f_i(x_0)) - \nabla f_i(x_0)} + \frac{1}{n}\sum_{i=1}^n \norm{z_{-1}^i} \\ 
    & \overset{\text{(Lemma~\ref{lem:clip}(ii)-(iii))}}{\leq} & \norm{\nabla f(x_0)} + \frac{1}{n}\sum_{i=1}^n \max\{0,\norm{\nabla f_i(x_0)}-\tau\} + \frac{1}{n}\sum_{i=1}^n \norm{z_{-1}^i} \\
    & \leq & \norm{\nabla f(x_0)} + \frac{1}{n}\sum_{i=1}^n \max\{0,\norm{\nabla f_i(x_0)}-\tau\} + \nu. 
\end{eqnarray*}

If  $\norm{\nabla f_i(x_k)-v_{k-1}^i} \leq \max_i \norm{\nabla f_i(x_0)} \eqdef B > \tau$ and $\gamma \leq 2/L$, then  
\begin{eqnarray*}
    \norm{v_0}
    & \leq & \norm{\nabla f(x_0)} + B + \nu - \tau \\ 
    & \overset{\eqref{eq:ssss}}{\leq} & \sqrt{2L (f(x_0)-f_{\inf})} + B + \nu -\tau \\ 
    & \leq & \sqrt{2L \phi_0} + B + \nu -\tau \\ 
    & \leq & \sqrt{\frac{4}{\gamma}\phi_0} + 2(B+\nu-\tau).
\end{eqnarray*}
\end{proof}

\begin{claim}\label{claim:v_k_DPCLip}
Fix $k \geq 1$. Let $f$ have $L$-Lipschitz gradient. Also suppose that  $\norm{\nabla f_i(x_k)-v_{k-1}^i} \leq \max_i \norm{\nabla f_i(x_0)} \eqdef B > \tau$, $\norm{v_{k-1}} \leq \sqrt{\frac{4}{\gamma}\phi_0} + 2(B+\nu-\tau)$, $\norm{\nabla f(x_{k-1})} \leq \sqrt{\frac{2}{\gamma}\phi_0} + \sqrt{2L C} \nu$ for some $C>0$, $\gamma \leq (1 - 1/\sqrt{2})/L$ and $\nu \leq \frac{\tau}{2(\sqrt{2LC} + 2)}$. Then, 
 	\begin{eqnarray}
		\norm{v_k  } \leq  \sqrt{\frac{4}{\gamma}\phi_0}+  2(B -\tau/2).
	\end{eqnarray}    
\end{claim}
\begin{proof}
 From the definition of $v_k$, 
 \begin{eqnarray*}
v_k 
& = & \frac{1}{n}\sum_{i=1}^n [v_{k-1}^i + \clip_\tau(\nabla f_i(x_k)-v_{k-1}^i) + z_{k-1}^i] \\
& = & \nabla f(x_k) + \frac{1}{n}\sum_{i=1}^n [\clip_\tau(\nabla f_i(x_k)-v_{k-1}^i)- (\nabla f_i(x_k)-v_{k-1}^i)  +  z_{k-1}^i ].
 \end{eqnarray*}
 Therefore, 
 	\begin{eqnarray*}
		\norm{v_k  } 
		& \overset{\eqref{eq:triangle}}{\leq} & \norm{ \nabla f(x_k)} + \frac{1}{n}\sum_{i=1}^n \norm{\clip_\tau(\nabla f_i(x_k)-v_{k-1}^i)- (\nabla f_i(x_k)-v_{k-1}^i)} + \frac{1}{n}\sum_{i=1}^n \norm{z_{k-1}^i}.
	\end{eqnarray*}
By the fact that  $\norm{z_{k-1}^i} \leq \nu$ and by  {Lemma~\ref{lem:clip}(ii)-(iii)}, 
	\begin{eqnarray*}
		\norm{v_k  } 
		& \leq & \norm{\nabla f(x_k)} + \nu + \frac{1}{n}\sum_{i=1}^n \max\{0, \norm{\nabla f_i(x_k)-v_{k-1}^i)} - \tau \} \\
            & \overset{\eqref{eq:triangle}}{\leq} & \norm{\nabla f(x_{k-1})} + \norm{\nabla f(x_{k}) - \nabla f(x_{k-1})} + \nu + \frac{1}{n}\sum_{i=1}^n \max\{0, \norm{\nabla f_i(x_k)-v_{k-1}^i)} - \tau \} \\ 
            & \overset{\eqref{eq:8y98fd_08yfd} + \eqref{eq:L-smooth}}{\leq}  & \norm{\nabla f(x_{k-1})} + L \gamma \norm{ v_{k-1} } + \nu + \frac{1}{n}\sum_{i=1}^n \max\{0, \norm{\nabla f_i(x_k)-v_{k-1}^i)} - \tau \}. 
	\end{eqnarray*}

If $\norm{\nabla f_i(x_k)-v_{k-1}^i} \leq \max_i \norm{\nabla f_i(x_0)} \eqdef B > \tau$, then  
	\begin{eqnarray*}
		\norm{v_k  } 
		 \leq  \norm{\nabla f(x_{k-1})} + L \gamma \norm{ v_{k-1} } + \nu + B -\tau. 
	\end{eqnarray*}

If $\norm{v_{k-1}} \leq \sqrt{\frac{4}{\gamma}\phi_0} + 2(B+\nu-\tau)$, $\norm{\nabla f(x_{k-1})} \leq \sqrt{\frac{2}{\gamma}\phi_0} + \sqrt{2LC}\nu$ for some $C>0$, and $\gamma \leq (1 - 1/\sqrt{2})/L$, then we have $\gamma \leq 1/(2L)$ and 
 	\begin{eqnarray*}
		\norm{v_k  } 
		 & \leq &   (L\gamma + 1/\sqrt{2})  \sqrt{\frac{4}{\gamma}\phi_0}+  (2L\gamma + 1)(B  -\tau) + (2L\gamma + \sqrt{2LC} + 1)\nu \\
            & \leq & \sqrt{\frac{4}{\gamma}\phi_0}+  2(B -\tau)  + (\sqrt{2LC} + 2)\nu.
	\end{eqnarray*}

 If $\nu \leq \frac{\tau}{2(\sqrt{2LC} + 2)}$, then $\norm{v_{k-1}} \leq \sqrt{\frac{4}{\gamma}\phi_0} + 2(B+\nu-\tau) \leq  \sqrt{\frac{4}{\gamma}\phi_0}+  2(B -\tau/2)$ and 
 	\begin{eqnarray*}
		\norm{v_k  } 
		 &  \leq & \sqrt{\frac{4}{\gamma}\phi_0}+  2(B -\tau/2).
	\end{eqnarray*}
\end{proof}

\begin{claim}\label{claim:stepsize_range_DPClip}
If 
\begin{align}
    0 \leq \gamma \leq \frac{\tau^2}{64 L_{\max}^2 \left(\sqrt{F_0} + \sqrt{F_0 + \frac{G_0 \tau}{2\sqrt{2\eta}L_{\max}}} \right)^2},
\end{align}
where $\eta\eqdef \min\left\{1,\frac{\tau}{\max_i\norm{\nabla f_i(x_0)}} \right\}$, $F_0:=f(x_0)-f(x^\star)$ and $G_0:=  \sqrt{\frac{1}{n}\sum_{i=1}^n (\vert \nabla f_i(x_0) - \tau \vert + \nu)^2 }$, then
\begin{eqnarray}
    4L_{\max} \sqrt{\gamma\phi_0} \leq \frac{\tau}{4}.
\end{eqnarray}    
\end{claim}
\begin{proof}
From the definition of $\phi_0$ and by the fact that $v_{-1}^i=0$ and that $\norm{z^i_k} \leq \nu$,
\begin{eqnarray*}
   4L_{\max} \sqrt{\gamma\phi_0} 
   & = & 4 L_{\max} \sqrt{\gamma F_0 + \frac{2\gamma^2}{\eta} \frac{1}{n}\sum_{i=1}^n \norm{\nabla f_i(x_0) - v^i_0}^2} \\ 
   & = &  4 L_{\max} \sqrt{\gamma F_0 + \frac{2\gamma^2}{\eta} \frac{1}{n}\sum_{i=1}^n \norm{\nabla f_i(x_0) - \clip_\tau(\nabla f_i(x_0)) - z^i_{-1} }^2},
\end{eqnarray*}
where $F_0 = f(x_0)-f(x^\star)$.
Since 
\begin{eqnarray*}
     \norm{\nabla f_i(x_0) - \clip_\tau(\nabla f_i(x_0)) - z^i_{-1} } & \overset{\eqref{eq:triangle}}{\leq} & \norm{\nabla f_i(x_0) - \clip_\tau(\nabla f_i(x_0))} + \norm{z^i_{k-1}} \\ 
     & \overset{\text{(Lemma~\ref{lem:clip}(ii)-(iii))}}{\leq}& \max\{ 0 , \norm{\nabla f_i(x_0)} - \tau\}+ \nu \\
     & \leq & \vert \nabla f_i(x_0) - \tau \vert + \nu,
\end{eqnarray*}
we have 
\begin{eqnarray*}
   4L_{\max} \sqrt{\gamma\phi_0} \overset{\eqref{eq:sqrt}}{\leq} 4 L_{\max}\sqrt{\gamma}\sqrt{F_0} + 4L_{\max} \gamma \sqrt{\frac{2}{\eta}} G_0,
\end{eqnarray*}   
where $G_0 = \sqrt{\frac{1}{n}\sum_{i=1}^n (\vert \nabla f_i(x_0) - \tau \vert + \nu)^2 }$.

Here, any $\gamma > 0$ satisfying
\begin{eqnarray*}
    4 L_{\max}\sqrt{\gamma}\sqrt{F_0} + 4L_{\max} \gamma \sqrt{\frac{2}{\eta}} G_0 \leq \frac{\tau}{4}
\end{eqnarray*}
also satisfies $ 4L_{\max} \sqrt{\gamma\phi_0} \leq \frac{\tau}{4}$. This condition can be expressed equivalently as:
\begin{eqnarray*}
    \frac{1}{\sqrt{\gamma}}  - \frac{16 \sqrt{F_0}}{\tau} L_{\max} - \sqrt{\gamma} \cdot \frac{16\sqrt{2} G_0}{\tau \sqrt{\eta} L_{\max}} L_{\max}^2 \geq 0.
\end{eqnarray*}
Finally, by Lemma~\ref{lemma:trick_stepsize_EF21} with $L = L_{\max}$, $\beta_1 = \frac{16\sqrt{F_0}}{\tau}$ and $\beta_2 = \frac{16\sqrt{2}}{\sqrt{\eta}}\frac{G_0}{L_{\max} \tau}$, we have 
\begin{align*}
    0 \leq \sqrt{\gamma} \leq \frac{\tau}{8 L_{\max} \left(\sqrt{F_0} + \sqrt{F_0 + \frac{G_0 \tau}{2\sqrt{2\eta}L_{\max}}} \right)}.
\end{align*}
Taking the square, we obtain the final result. 
\end{proof}

\subsection{Proof of Theorem~\ref{thm:DP-clippedEF21}}
	
	Let $\mathcal{I}_k$ be the subset from $\{1,2,\ldots,n\}$ such that $\| \nabla f_i(x_k) - v_{k-1}^i \|>\tau$. 
	We then derive the descent inequality 
	\begin{align*}
		\phi_{k+1} \leq \phi_k - \frac{\gamma}{2}\| \nabla f(x_k)\|^2 + \frac{2\gamma}{\eta} \left( 1 + \frac{2}{\eta}\right) \nu^2,
	\end{align*}
	  for two possible cases: (1) when $\vert \mathcal{I}_k \vert >0$ and (2) when $\vert \mathcal{I}_k \vert =0$.

	\paragraph{Case (1):  $\vert \mathcal{I}_k \vert >0$.}
	To derive the descent inequality we will show by induction the stronger result: $\| \nabla f_i(x_{k}) - v^i_{k-1}\| \leq B - \frac{k\tau}{2}$ for $i\in\mathcal{I}_k$, where $B \eqdef \max_i\|\nabla f_i(x_0)\|$ and
	\begin{equation}
		\phi_{k+1} \leq (1-\gamma\mu)\phi_k   - \frac{\gamma}{4}\norm{ v_k  }^2 + \frac{2\gamma}{\eta} \left( 1 + \frac{2}{\eta}\right)\nu^2 \label{eq:detailed_descent_inequality_DP}
	\end{equation}
	for any $k \geq 0$, where for notational convenience we assume that $\nabla f_i(x_{-1}) = v^i_{-1}= 0$ and $\phi_{-1} = \phi_0$. The base of the induction is trivial: when $k = 0$ we have   $\| \nabla f_i(x_{k}) - v^i_{k-1}\| = \|\nabla f_i(x_0) - v^i_{-1}\| = \|\nabla f_i(x_0)\| \leq \max_i \|\nabla f_i(x_0)\| = B$ for $i\in\mathcal{I}_k$ and \eqref{eq:detailed_descent_inequality_DP} holds by definition. Next, we assume that for some $k\geq 0$ inequalities $\|\nabla f_i(x_t) - v^i_{t-1}\| \leq B$ for $i\in\mathcal{I}_k$ and \eqref{eq:detailed_descent_inequality_DP} hold for $t = 0,1,\ldots, k$. 

Let $0 \leq \phi_k \leq \phi_0 + C \nu^2$ for some $C>0$. If $\gamma \leq (1-1/\sqrt{2})/L$, then we have $\gamma \leq 1/L$ and 
\begin{eqnarray*}
    \norm{\nabla f(x_{k-1})}^2 
    & \overset{\eqref{eq:ssss}}{\leq} & 2L[f(x_{k-1})-f(x^\star)] \\ 
    & \leq & 2L \phi_{k-1} \\ 
    & \leq & 2L \phi_0 + 2L C \nu^2 \\ 
    & \leq & \frac{2}{\gamma}\phi_0 + 2L C \nu^2. 
\end{eqnarray*}

If $\nu \leq \frac{\tau}{2(\sqrt{2LC} + 2)}$, then we have $\nu \leq \frac{\tau}{2}$. From Claim~\ref{claim:v_0_DPCLip}, we have $\norm{v_0} \leq \sqrt{\frac{4}{\gamma}\phi_0} + 2(B-\tau/2)$. From Claim~\ref{claim:v_k_DPCLip}, we have $\norm{v_k} \leq \sqrt{\frac{4}{\gamma}\phi_0} + 2(B-\tau/2)$ for $k \geq 1$. Hence, for $k \geq 0$, 
\begin{align*}
    \norm{v_k} \leq \sqrt{\frac{4}{\gamma}\phi_0} + 2(B-\tau/2).
\end{align*}
Next, from Claim~\ref{claim:diff_v_f_DPClip}, for $i\in\mathcal{I}_k$,
\begin{eqnarray*}
		\norm{ \nabla f_i(x_{k+1}) - v_k^i } 	
   &\leq& \norm{ \nabla f_i(x_k) - v_{k-1}^i} - \frac{\tau}{2}+ L_{\max} \gamma \norm{v_k} \\ 
   & \leq &\norm{ \nabla f_i(x_k) - v_{k-1}^i} - \frac{\tau}{2} + 2L_{\max} \sqrt{\gamma\phi_0} + 2L_{\max}\gamma(B-\tau/2). 
\end{eqnarray*}
 
{\bf STEP: Small step-size}

We can conclude that for $i\in\mathcal{I}_{k}$
	\begin{align*}
		\| \nabla f_i(x_{k+1}) - v_k^i \| \leq 	\| \nabla f_i(x_{k}) - v_{k-1}^i \| - \frac{\tau}{4} \leq \ldots \leq \| \nabla f_i(x_0) \| - \frac{(k+1)\tau}{4}
	\end{align*}
	if the following step-size conditions hold:
	\begin{align}\label{eqn:parameter_condition_DP_case_on}
	 2L_{\max}\gamma(B-\tau/2) \leq  2L_{\max} \sqrt{\gamma\phi_0}, \quad \text{and} \quad 4L_{\max} \sqrt{\gamma\phi_0} \leq \frac{\tau}{4}.
	\end{align}
From Claim~\ref{claim:stepsize_range_DPClip}, the above condition can be expressed equivalently as: 
\begin{eqnarray*}
    \gamma \leq \frac{\phi_0}{(B-\tau/2)^2} \quad \text{and} \quad  \gamma \leq \frac{\tau^2}{64 L_{\max}^2 \left(\sqrt{F_0} + \sqrt{F_0 + \frac{G_0 \tau}{2\sqrt{2\eta}L_{\max}}} \right)^2}.
\end{eqnarray*}

In conclusion, under this step-size condition, $\norm{\nabla f_i(x_k) - v^i_{k-1}} \leq B$ for $i \in \mathcal{I}_{k-1}$ and $k \geq 0$. In addition, $\mathcal{I}_{k+1} \subseteq \mathcal{I}_k$.

 {\bf STEP: Descent inequality}
	
	It remains to prove the descent inequality. 
 From the inductive assumption above, for $i\in\mathcal{I}_{k+1}$
 \begin{eqnarray}\label{eq:eta_DPClip}
     \eta^i_{k+1} = \frac{\tau}{\| \nabla f_i(x_{k+1}) - v^i_{k} \|} \geq  \frac{\tau}{B} := \eta. 
 \end{eqnarray}
Therefore, 
\begin{eqnarray*}
    \norm{ \nabla f_i(x_{k+1}) - v^i_{k+1}}^2
    & \overset{\eqref{eqn:v_k_DP_equi}}{=} & \norm{ \nabla f_i(x_{k+1}) - v_k^i - \clip_\tau(\nabla f_i(x_{k+1}) - v_k^i) - z_{k}^i}^2 \\
    & = & \norm{ z_{k}^i}^2 \cdot  \mathbbm{1}(i\in\mathcal{I}_{k+1}^\prime) \\
    && + \norm{ (1-\eta_{k+1}^i)(\nabla f_i(x_{k+1}) - v_k^i) - z_{k}^i}^2\cdot  \mathbbm{1}(i\in\mathcal{I}_{k+1}) \\
    & \overset{\eqref{eq:Young} + \eqref{eq:eta_DPClip}}{\leq} &  \norm{ z_{k}^i}^2 \cdot  \mathbbm{1}(i\in\mathcal{I}_{k+1}^\prime) +  (1+1/\theta_1) \norm{z_{k}^i}^2 \cdot  \mathbbm{1}(i\in\mathcal{I}_{k+1}) \\
    && + (1-\eta)^2 (1+\theta_1)\norm{\nabla f_i(x_{k+1}) - v_k^i}^2 \mathbbm{1}(i\in\mathcal{I}_{k+1}) \\
    & \leq & (1+1/\theta_1) \norm{z_k^i}^2 + (1-\eta)^2 (1+\theta_1)\norm{\nabla f_i(x_{k+1}) - v_k^i}^2,
\end{eqnarray*}
for $\theta>0$. By \eqref{eq:Young}, 
\begin{eqnarray*}
    \norm{ \nabla f_i(x_{k+1}) - v^i_{k+1}}^2
    &\leq& (1+\theta_1)(1+\theta_2)(1-\eta)^2\norm{\nabla f_i(x_{k}) - v_k^i}^2 \\
    && + (1+\theta_1)(1+1/\theta_2)\norm{\nabla f_i(x_{k+1}) - \nabla f _i(x_k)}^2 + (1+1/\theta_1) \norm{z_k^i}^2 \\ 
    & \overset{\eqref{eq:L_i-smooth}}{=} & (1+\theta_1)(1+\theta_2)(1-\eta)^2\norm{\nabla f_i(x_{k}) - v_k^i}^2 \\
    && + (1+\theta_1)(1+1/\theta_2)L_{\max}^2 \norm{x_{k+1} - x_k}^2 + (1+1/\theta_1) \norm{z_k^i}^2.
\end{eqnarray*}    
 Taking $\theta_1 = \theta_2 = \nicefrac{\eta}{2}$ and applying the inequality $(1-\eta)(1+\nicefrac{\eta}{2}) \leq 1-\nicefrac{\eta}{2}$, we get
	\begin{eqnarray}\label{eqn:TrickFromEF21_DP}
		\| \nabla f_i(x_{k+1}) - v^i_{k+1} \|^2 
		& \leq&  (1-\nicefrac{\eta}{2})^2 \| \nabla f_i(x_{k}) - v^i_{k}  \|^2   + (1+\nicefrac{2}{\eta})\| z_k^i\|^2 \notag \\
		&& + (1+\nicefrac{2}{\eta})(1-\eta)(1-\eta/2) (L_{\max})^2 \| x_{k+1} - x_{k}  \|^2. 
	\end{eqnarray}

	Next, we combine the above inequality with  Lemma \ref{lemma:descent_li2021page}, 
	\begin{eqnarray*}
		\phi_{k+1}  
		& = & f(x_{k+1}) - f_\star + \frac{2\gamma}{\eta} \frac{1}{n}\sum_{i=1}^n \|  \nabla f_i(x_{k+1}) - v_{k+1}^i\|^2 \\
		& \overset{\eqref{eq:descent_ineq}}{\leq} & f(x_k)- f_{\star} - \frac{\gamma}{2}\norm{ \nabla f(x_k) }^2 - \left( \frac{1}{2\gamma} - \frac{L}{2} \right)\| x_{k+1} - x_k \|^2  \\
		&&+ \frac{\gamma}{2}\frac{1}{n}\sum_{i=1}^n\norm{ \nabla f(x_k) - v_{k}   }^2 + \frac{2\gamma}{\eta}\frac{1}{n}\sum_{i=1}^n\| \nabla f_i(x_{k+1}) - v^i_{k+1} \|^2   \\
		& \overset{\eqref{eqn:TrickFromEF21_DP}}{\leq} & f(x_k)- f_{\star} - \frac{\gamma}{2}\norm{ \nabla f(x_k) }^2  \\
  && - \left( \frac{1}{2\gamma} - \frac{L}{2} -  \frac{2\gamma}{\eta} \left( 1 + \frac{2}{\eta}\right) (1-\eta)(1-\eta/2) (L_{\max})^2 \right)\| x_{k+1} - x_k \|^2 \\
		&& + \left(\frac{\gamma}{2} + \frac{2\gamma}{\eta}(1-\eta/2)^2 \right)\frac{1}{n}\sum_{i=1}^n\norm{ \nabla f(x_k) - v_{k}   }^2 + \frac{2\gamma}{\eta} \left( 1 + \frac{2}{\eta}\right) \frac{1}{n}\sum_{i=1}^n \| z_k^i\|^2.
	\end{eqnarray*}
	By the fact that $F$ satisfies the PŁ condition with $\mu>0$ and that $(1-\eta/2)^2 \leq 1-\eta/2$, 
	\begin{eqnarray*}
		\phi_{k+1} &\leq& (1-\gamma\mu)( f(x_k)- f_{\star}) + \left( 1- \frac{\eta}{4} \right)   \frac{2\gamma}{\eta} \frac{1}{n}\sum_{i=1}^n\norm{ \nabla f(x_k) - v_{k}   }^2 + \frac{2\gamma}{\eta} \left( 1 + \frac{2}{\eta}\right) \frac{1}{n}\sum_{i=1}^n \| z_k^i\|^2\\
		& & - \left( \frac{1}{2\gamma} - \frac{L}{2} -  \frac{2\gamma}{\eta} \left( 1 + \frac{2}{\eta}\right) (1-\eta)(1-\eta/2) (L_{\max})^2 \right)\| x_{k+1} - x_k \|^2.  
	\end{eqnarray*}
	Therefore, 
	\begin{align*}
		\phi_{k+1} &\leq (1-\gamma\mu)\phi_k  - \frac{1}{4\gamma}\| x_{k+1} - x_k\|^2+ \frac{2\gamma}{\eta} \left( 1 + \frac{2}{\eta}\right) \frac{1}{n}\sum_{i=1}^n \| z_k^i\|^2
	\end{align*}
	if the step-size $\gamma>0$ satisfies 
	\begin{align*}
		1-\eta/4 \leq 1-\gamma\mu \quad \text{and} \quad  \frac{1}{2\gamma} - \frac{L}{2} -  \frac{2\gamma}{\eta} \left( 1 + \frac{2}{\eta}\right) (1-\eta)(1-\eta/2) (L_{\max})^2 \geq \frac{1}{4\gamma}.
	\end{align*}
	From Lemma \ref{lemma:trick_stepsize_EF21} with $L=1$, $\beta_1=2L$ and $\beta_2 =\frac{8(1+2/\eta)(1-\eta)(1-\eta/2)}{\eta}(L_{\max})^2/L^2$, the above condition can be equivalently expressed as: 
	\begin{align*}
		0 <	\gamma \leq \min\left( \frac{\eta}{4\mu} , \frac{1}{L\left[ 1+\sqrt{1+ \frac{8(1+2/\eta)(1-\eta)(1-\eta/2)}{\eta}\left( \frac{L_{\max}}{L} \right)^2}\right]} \right).
	\end{align*}
	Since $x_{k+1}-x_k = -\gamma v_k$,
	\begin{align*}
		\phi_{k+1} 
		& \leq  (1-\gamma\mu)\phi_k  - \frac{\gamma}{4}\norm{ v_k  }^2 + \frac{2\gamma}{\eta} \left( 1 + \frac{2}{\eta}\right) \frac{1}{n}\sum_{i=1}^n \| z_k^i\|^2. 
	\end{align*}
	Since $\| z_k^i \|^2 \leq \nu^2$, we get 
	\begin{align}\label{eq:DP_DescentIneq_DP_Case1}
		\phi_{k+1} \leq (1-\gamma\mu)\phi_k  + \frac{2\gamma}{\eta} \left( 1 + \frac{2}{\eta}\right)\nu^2.
	\end{align}
	Since $\| z_k^i \|^2 \leq \nu^2$ and $\Exp{\norm{z_k^i}^2} \leq \sigma^2$, we have $\Exp{\norm{z_k^i}^2} \leq \min(\nu^2,\sigma^2)$ and
	\begin{align}\label{eq:DP_DescentIneq_final_Case1}
		\Exp{\phi_{k+1}} \leq (1-\gamma\mu)\Exp{\phi_k}  - \frac{\gamma}{4}\norm{ v_k  }^2 + \frac{2\gamma}{\eta} \left( 1 + \frac{2}{\eta}\right)\min(\nu^2,\sigma^2).
	\end{align}
	This concludes the proof in the case (1).

	\paragraph{Case (2):  $\vert \mathcal{I}_k \vert =0$.}
	Suppose $\vert \mathcal{I}_k \vert=0$. Then, $\norm{ \nabla f_i(x_k) - v^i_{k-1}  } \leq \tau$ for all $i$.
	Then, by using \eqref{eqn:v_k_DP_equi} and by the fact that $\eta_k = \eta = 1$, we have $v^i_k = \nabla f_i(x_k)+z_{k-1}^i$. Therefore, \algname{DP-Clip21-GD} 
	can be expressed equivalently as:
	\begin{align*}
		x_{k+1} = x_k - \gamma v_k,
	\end{align*}	
	where $v_k = \nabla f(x_k) + (1/n)\sum_{i=1}^n {z}_{k-1}^i$. From the definition of $\phi_{k}$ and Lemma \ref{lemma:descent_li2021page}, by the fact that $f(x)$ satisfies  the PŁ condition with $\mu>0$, and by  letting $\gamma \leq 1/L$,
	\begin{align*}
		\phi_{k+1} 
		& = f(x_{k+1}) - f_{\star} \\
		& \leq f(x_k)- f_{\star}- \frac{\gamma}{2}\norm{ \nabla f(x_k) }^2 - \left( \frac{1}{2\gamma} - \frac{L}{2} \right)\| x_{k+1} - x_k \|^2 + \frac{\gamma}{2}\norm{ \nabla f(x_k) - v_{k}   }^2 \\
		& \leq (1-\gamma\mu)[f(x_k)-f_\star] + \frac{\gamma}{2} \frac{1}{n}\sum_{i=1}^n \| {z}^i_{k-1} \|^2 \\
		& \leq  (1-\gamma\mu)\phi_k  + \frac{\gamma}{2} \frac{1}{n}\sum_{i=1}^n \| {z}^i_{k-1} \|^2.
	\end{align*}
	Since $\| z_k^i \|^2 \leq \nu^2$, we get 
\begin{align}\label{eq:DP_DescentIneq_DP_Off}
	\phi_{k+1} \leq (1-\gamma\mu)\phi_k  + \frac{\gamma}{2}\nu^2.
\end{align}
	Since $\| z_k^i \|^2 \leq \nu^2$ and $\Exp{\| z^i_{k-1} \|^2} \leq \sigma^2$, we have $\mathbf{E}\| z_k^i \|^2 \leq \min(\nu^2,\sigma^2)$ and 
	\begin{align}\label{eq:DP_DescentIneq_final_Off}
		\Exp {\phi_{k+1}} \leq (1-\gamma\mu)\Exp{\phi_k} + \frac{\gamma}{2} \min(\nu^2, \sigma^2).
	\end{align}	

Finally, we will show that $\norm{ \nabla f_i(x_k) - v^i_{k-1}  } \leq \tau$ for all $i$ implies $\| \nabla f_i(x_{k+1}) - v^i_{k}\| \leq \tau$ for all $i$. Indeed, in this case, we have $v^i_k = \nabla f_i(x_k) + z^i_{k-1}$ and
\begin{eqnarray*}
    \norm{\nabla f_i(x_{k+1}) - v_k^i}
    & = &  \norm{\nabla f_i(x_{k+1}) - \nabla f_i(x_k) - z_{k-1}^i} \\ 
    & \overset{\eqref{eq:triangle}}{\leq} & \norm{\nabla f_i(x_{k+1}) - \nabla f_i(x_k)} + \norm{z_{k-1}^i} \\ 
    & \overset{\eqref{eq:L_i-smooth} }{\leq} & L_{\max}\gamma \norm{v_k} + \norm{z_{k-1}^i} \\ 
    & \overset{\eqref{eq:triangle}}{\leq} & L_{\max}\gamma \norm{\nabla f(x_k)} + L_{\max}\gamma \frac{1}{n}\sum_{i=1}^n \norm{z_{k-1}^i} +\norm{z_{k-1}^i}. 
\end{eqnarray*}
By the fact that $\norm{z_k^i} \leq \nu$ and by setting $\gamma \leq 1/(2L)$, we have 
\begin{eqnarray*}
    \norm{\nabla f_i(x_{k+1}) - v_k^i}
    & \leq & L_{\max}\gamma \norm{\nabla f(x_k)} + (L_{\max}\gamma +1 ) \nu \\ 
    & \overset{\eqref{eq:ssss}}{\leq} & L_{\max}\gamma\sqrt{2L(f(x_k)-f(x^\star))} + (L_{\max}\gamma +1 ) \nu \\ 
    & \leq & L_{\max}\gamma\sqrt{2L \phi_k} + (L_{\max}\gamma +1 ) \nu \\ 
    & \leq &  L_{\max}\gamma\sqrt{\frac{1}{\gamma}\phi_k} + (L_{\max}\gamma +1 ) \nu.
\end{eqnarray*}
From \eqref{eq:DP_DescentIneq_DP_Off}, we have $\phi_k \leq \phi_0 + \frac{1}{2\mu}\nu^2$ and 
\begin{eqnarray*}
    \norm{\nabla f_i(x_{k+1}) - v_k^i}
    & \leq &  L_{\max}\gamma\sqrt{\frac{1}{\gamma}\phi_0 + \frac{1}{2\mu\gamma} \nu^2} + (L_{\max}\gamma +1 ) \nu \\
    & \leq & L_{\max} \sqrt{\gamma \phi_0} + \left( L_{\max}\sqrt{\frac{\gamma}{2\mu}} + L_{\max}\gamma +1 \right) \nu.
\end{eqnarray*}
We can hence conclude that $\norm{\nabla f_i(x_{k+1}) - v_k^i} \leq \tau$ if the following step-size condition holds: 
\begin{eqnarray}
    \nu \leq \frac{\tau}{2(\sqrt{2LC}+2)}, \quad L_{\max}\sqrt{\frac{\gamma}{2\mu}} \leq 1, \quad L_{\max}\gamma \leq 1 \quad  \text{and} \quad  4L_{\max} \sqrt{\gamma \phi_0} \leq \tau/4,
\end{eqnarray}
for some $C>0$. The last condition can be obtained from Claim~\ref{claim:stepsize_range_DPClip}. Finally, 
	putting all the conditions on $\gamma$ together, we obtain the results.

\section{Proof of Theorem \ref{thm:DPclipEF21_privacy}}
Let $[a,b]$ be the set such that each element of $a\in\mathbb{R}^d$ is $-\tau/\sqrt{d}$ and each element of $b\in\mathbb{R}^d$ is $\tau/\sqrt{d}$.
Next, by the fact that $\Exp{ \norm{z_k^i}^2 }  \leq \sigma^2$ and $\max \norm{ z_k^i}^2 \leq \nu^2$, we have $\Exp{ \norm{z_k^i}^2 }  \leq \min(\sigma^2,\nu^2)$.
For each client, the query $Q$ is $\clip_\tau(\nabla f_i(x_k) - v^i_{k-1})$ that has its Euclidean norm being upper-bounded by $\tau$. Therefore, by Theorem 3.4 of \citep{chen2022bounded} with $\| b-a \| \leq \| b \| + \|a\| \leq 2\tau$ and $\Delta Q \leq 2\tau$ and by the fact that $\mathbf{E}\| z_k^i \|^2\leq \min(\sigma^2,\nu^2)$, the multivariate bounded Gaussian mechanism ensures $\epsilon^\prime$-DP if 
\begin{align*}
	\min(\sigma^2,\nu^2) \geq \frac{6\tau^2}{\epsilon^\prime - \ln(\Delta C(\sqrt{\min(\sigma^2,\nu^2)},c^\star)) },
\end{align*}
where $c^\star$ is the optimal solution to the following problem: 
\begin{align*}
	\max_c & \quad  \Delta C(\sqrt{\min(\sigma^2,\nu^2)},c) \eqdef \frac{C(a,\sqrt{\min(\sigma^2,\nu^2)})}{C(a+c,\sqrt{\min(\sigma^2,\nu^2)})}  \\
	\text{subject to} & \quad  0 \leq c \leq b-a \quad \text{and} \quad  \| c \| \leq 2\tau.
\end{align*}
Here, $C(y,\sigma) \eqdef \nicefrac{1}{\int_a^b \exp(- \| x-y \|^2/\sigma^2)dx}$.

By the advanced composition theorem in Corollary 3.21 of \citep{dwork2014algorithmic}, given target privacy parameters $0<\epsilon<1$ and $\delta>0$, multi-node \algname{DP-Clip21-GD} satisfying Theorem \ref{thm:DP-clippedEF21} is $(\epsilon,\delta)$-DP if 
\begin{align*}
	\min(\sigma^2,\nu^2) \geq \frac{6\tau^2}{ \nicefrac{\epsilon}{(2\sqrt{2K \ln(1/\delta)})}  - \ln(\Delta C(\sqrt{\min(\sigma^2,\nu^2)},c^\star)) },
\end{align*} 
where $K$ is the number of steps.
Hence, multi-node \algname{DP-Clip21-GD} is $(\epsilon,\delta)$-DP if 
\begin{align}\label{eqn:DP_guarantee_Clip21}
	\min(\sigma^2,\nu^2) \geq \frac{12\tau^2\sqrt{2K \ln(1/\delta)}}{(1-\alpha)\epsilon}
\end{align} 
for some $0<\epsilon<1$, $\delta>0$ and $0<\alpha<1$ such that 
\begin{align*}
\nu \leq  \tau/6 \quad \text{and} \quad \ln(\Delta C(\sqrt{\min(\sigma^2,\nu^2)},c^\star)) \leq  \nicefrac{(\alpha\epsilon)}{(2\sqrt{2K \ln(1/\delta)})}.	
\end{align*}
In other words, for $\nu\in[\sigma,\tau/6]$
\begin{align*}
	\epsilon \geq  
	\max\left(  P_1 , P_2 \right)\sqrt{2K\ln(1/\delta)},
\end{align*}
or equivalently
\begin{align*}
	\delta \geq \exp\left( - \frac{1}{2K} \left[ \frac{\epsilon}{\max(P_1,P_2)} \right]^2 \right),
\end{align*}
where $P_1 \eqdef \frac{2\ln(\Delta C(\tau/6, c^\star)}{\alpha}$ and $P_2\eqdef  \frac{72\tau^2}{(1-\alpha)\tau}$. We complete the proof. 

\section{Proof of Theorem \ref{thm:DPclipEF21_utility}}
By plugging $\min(\nu^2,\sigma^2) = \sigma^2_{\min}(K)$ into the descent inequality \eqref{eq:descentIneq_DP_clipped_EF21}
\begin{align*}
		\Exp{\phi_{k+1}} 
		& \leq (1-\gamma\mu)\Exp{\phi_k} + \frac{2(1+2/\eta)}{\eta}\gamma \min \left(\nu^2,\sigma^2 \right) \\
		& = (1-\gamma\mu)\Exp{\phi_k} + \gamma \mu \cdot  A_3 \sigma^2_{\min}(K).
\end{align*}	
Next, by applying the resulting inequality recursively over $k=0,1,\ldots,K-1$, we have the result. 
\begin{align} \label{eq:descentIneq_DP_clipped_EF21_utility}
	\Exp{\phi_{K}} \leq (1-\gamma\mu)^K\phi_0 + A_3 \sigma_{\min}^2(K).
\end{align}

\clearpage

\section{Adding Communication Compression into the Mix}\label{sec_app:ClipPress21}

\begin{algorithm}[t]
	\centering
	\caption{\algname{Press-Clip21-GD} (Error Feedback for  Distributed Optimization with Compression and Clipping)}\label{alg:ClipPress21-GD}
	\begin{algorithmic}[1]
		\STATE \textbf{Input:} initial iterate $x_{0} \in \R^d$; learning rate $\gamma>0$; initial gradient shifts $v^1_{-1},\dots,v^n_{-1}\in \R^d$; clipping threshold $\tau>0$; deterministic contractive compression $\cC:\mathbb{R}^d\rightarrow\mathbb{R}^d$, i.e. $\| \cC(v) - v \|^2\leq (1-\alpha)\|v \|^2$ for $0<\alpha\leq 1$ and $v\in\mathbb{R}^d$ 
		\FOR{$k=0,1, 2, \dots, K-1 $}
		\STATE Broadcast $x_k$ to all workers 
		\FOR{each worker $i=1,\ldots,n$ in parallel}
		\STATE Compute 	$g_k^i = \cC( \clip_\tau( \nabla f_i(x_{k}) - v^i_{k-1} ) )$
		\STATE Update $v^i_{k} = v^i_{k-1} + g_k^i$
		\ENDFOR
		\STATE  $v_{k} = v_{k-1} + \frac{1}{n}\sum_{i=1}^n g_k^i$			
		\STATE  $x_{k+1} = x_k - \gamma  \frac{1}{n}\sum_{i=1}^n v^i_k$ 
		
		\ENDFOR
	\end{algorithmic}
\end{algorithm}

In this section, we consider \algname{Press-Clip21-GD}, which is described in Algorithm \ref{alg:ClipPress21-GD}.
This algorithm attains the descent inequality as described in the next theorem:
\begin{theorem}\label{thm:ClipPress21-GD}
	Consider the problem of minimizing $f(x)=\frac{1}{n}\sum_{i=1}^n f_i(x)$.
	Suppose that each $f_i(x)$ is $L$-Lipschitz gradient and the whole objective $f(x)$ is lower bounded by $f_{\emph{inf}}$.
	Let $v_{-1}^i=0$ for all $i$, $\eta\eqdef\min\left(1,\frac{\tau}{\max_i \| \nabla f_i(x_0)\|}\right)$, $\FF\eqdef f(x_0)-f_{\emph{inf}}$, and  $\GG \eqdef \sqrt{\frac{1}{n}\sum_{i=1}^n (\max\{0,\norm{\nabla f_i(x_0)} - \tau\} + \sqrt{1-\alpha}\tau)^2 }$.
	Then,   \algname{Press-Clip21-GD}  with 
	\begin{align} \label{eq:stepsize_compclip_clipped_EF21}
		\squeeze \gamma \leq  \min\left( \frac{1-\sqrt{1-\alpha}}{2\sqrt{1-\alpha} L_{\max} } , \frac{\phi_0}{(B-[1-\sqrt{1-\alpha}]\tau)^2} ,  \frac{1-1/\sqrt{2}}{L(1+\sqrt{1+2\beta_1})} , \frac{(1-\sqrt{1-\alpha})^2\tau^2}{16L_{\max}^2 \left[ \sqrt{\FF} + \sqrt{\beta_2} \right]^2} \right),	
	\end{align}
	where $\beta_1 \eqdef 2\frac{ \max\{(1-\beta)(1+2/\beta),(1-\alpha)(1+2/\alpha)\}}{\beta}\left( \frac{L_{\max}}{L}\right)^2$ and $\beta_2 \eqdef \FF + \frac{\GG(1-\sqrt{1-\alpha})\tau}{\sqrt{2\beta} L_{\max}}$ satisfies 
	\begin{align} \label{eq:descentIneq_compclip_clipped_EF21}
		\phi_{k+1} \leq \phi_k - \frac{\gamma}{2}\| \nabla f(x_k)\|^2,
	\end{align}
	where  $\phi_k \eqdef f(x_k)-f_{\emph{inf}}+\frac{\gamma}{n\beta}\sum_{i=1}^n\| \nabla f_i(x_k) - v_k^i \|^2$, and $\beta :=  1- \max(B_1,B_2(1-\eta)^2)\in[0,\alpha]$ with  $B_1=(1-\alpha) + (1+1/\theta_1)(1+\theta_2)(1-\alpha)$, $B_2=(1+\theta_1) + (1+1/\theta_1)(1+1/\theta_2)(1-\alpha)$ 
 for some $\theta_1,\theta_2>0$ and some $\alpha\in(0,1]$. In addition,
	\begin{align}\label{eq:descentIneq_compclip_clipped_EF21_final}
		\Exp{\| \nabla f(\hat x_K) \|^2} \leq \frac{2}{\gamma} \frac{\phi_0}{K},
	\end{align}
where $\hat x_K$ is the point selected uniformly at random from $\{x_0,x_1,\ldots,x_K\}$ for $K \geq 1$.
\end{theorem}

\section{Proof for Theorem~\ref{thm:ClipPress21-GD}}

 \algname{Press-Clip21-GD} can be equivalently expressed as:
 \begin{align}
    x_{k+1} & = x_k - \gamma v_k, \label{eq:x_update_comp_clip} \intertext{where}
     v_k & = \frac{1}{n}\sum_{i=1}^n v^i_k,  \label{eq:v_update_comp_clip}\\ 
    v_k^i &= (1-\eta_k^i) v_{k-1}^i + \eta_k^i \nabla f_i(x_k) + e_k^i,  \label{eq:v_each_update_comp_clip}
 \end{align}
 and 
$e_k^i = \cC(\clip_\tau(\nabla f_i(x_k) - v_{k-1}^i)) - \clip_\tau(\nabla f_i(x_k) - v_{k-1}^i)$ and $\eta_k^i = \min\left( 1 , \frac{\tau}{\| \nabla f_i(x_k) - v_{k-1}^i \|} \right)$. Note that the compressor $\mathcal{C}$ is contractive with $\alpha \in (0,1]$, i.e. 
\begin{align}\label{eq:contractive_comp}
    \norm{\mathcal{C}(v) - v}^2 \leq (1-\alpha)\norm{v}^2, \quad \forall v\in\mathbb{R}^d. 
\end{align}
Also recall that the Lyapunov function for this analysis is 
 \begin{eqnarray}
     \phi_k = f(x_k)-f_{\inf}+A\frac{1}{n}\sum_{i=1}^n \| \nabla f_i(x_k) - v_k^i \|^2,
 \end{eqnarray}
where $A=\frac{\gamma}{\beta}$ and $\beta :=  1- \max(B_1,B_2(1-\eta)^2)\in[0,\alpha]$ with  $B_1=(1-\alpha) + (1+1/\theta_1)(1+\theta_2)(1-\alpha)$, $B_2=(1+\theta_1) + (1+1/\theta_1)(1+1/\theta_2)(1-\alpha)$ 
 for some $\theta_1,\theta_2>0$ and some $\alpha\in(0,1]$.

\subsection{Useful claims}
We begin by presenting claims which are useful for deriving the result. 
\begin{claim}\label{claim:diff_v_f_multi_compclip}
Let each $f_i$ have $L_i$-Lipschitz gradient. Then, for $k \geq 0$, 
\begin{eqnarray}
    \norm{ \nabla f_i(x_{k+1}) - v_k^i }\leq   \max\{0, \norm{\nabla f_i(x_k) - v_{k-1}^i} - \tau\} + L_{\max}\gamma\norm{v_k} + \sqrt{1-\alpha} \tau.
\end{eqnarray}    
\end{claim}
\begin{proof}
 From the definition of $v_k^i$, 
 \begin{eqnarray*}
\norm{ \nabla f_i(x_{k+1}) - v_k^i }
&  \overset{\eqref{eq:triangle}}{\leq} & \norm{\nabla f_i(x_k)-v_k^i} + \norm{\nabla f_i(x_{k+1}) - \nabla f_i(x_k)} \\
&  \overset{\eqref{eq:L_i-smooth} + \eqref{eq:x_update_comp_clip}}{\leq} & \norm{\nabla f_i(x_k)-v_k^i} + L_{\max}\gamma\norm{v_k} \\ 
&  \overset{\eqref{eq:v_each_update_comp_clip}}{=} & \norm{\nabla f_i(x_k) - v_{k-1}^i - \mathcal{C}(\clip_{\tau}(\nabla f_i(x_k) - v_{k-1}^i))} + L_{\max}\gamma\norm{v_k} \\
& \overset{\eqref{eq:triangle}}{\leq} & \norm{\nabla f_i(x_k) - v_{k-1}^i - \clip_{\tau}(\nabla f_i(x_k) - v_{k-1}^i)} +  L_{\max}\gamma\norm{v_k} \\
&& + \norm{ \clip_{\tau}(\nabla f_i(x_k) - v_{k-1}^i) - \mathcal{C}(\clip_{\tau}(\nabla f_i(x_k) - v_{k-1}^i))} \\ 
&\overset{(\text{Lemma~\ref{lem:clip}(ii)-(iii)}) }{\leq} & \max\{0, \norm{\nabla f_i(x_k) - v_{k-1}^i} - \tau\} + L_{\max}\gamma\norm{v_k} \\
&& + \norm{ \clip_{\tau}(\nabla f_i(x_k) - v_{k-1}^i) - \mathcal{C}(\clip_{\tau}(\nabla f_i(x_k) - v_{k-1}^i))} \\
& \overset{ \eqref{eq:contractive_comp} }{\leq} & \max\{0, \norm{\nabla f_i(x_k) - v_{k-1}^i} - \tau\} + L_{\max}\gamma\norm{v_k} \\
&& + \sqrt{1-\alpha}\norm{ \clip_{\tau}(\nabla f_i(x_k) - v_{k-1}^i)} \\
& \leq &  \max\{0, \norm{\nabla f_i(x_k) - v_{k-1}^i} - \tau\} + L_{\max}\gamma\norm{v_k} + \sqrt{1-\alpha} \tau.
\end{eqnarray*}
Here, the last inequality comes from the fact that $\norm{\clip_\tau(v)} \leq \tau$ for $v\in\mathbb{R}^d$.
\end{proof}

\begin{claim}\label{claim:v_0_multi_compclip}
Let $v_{-1}^i=0$ for all $i$, $\max_i\norm{\nabla f_i(x_0)} := B > \tau$ and $\gamma \leq 2/L$. Then,  
\begin{eqnarray}
    \norm{v_0}  \leq  \sqrt{\frac{4}{\gamma}\phi_0}+ 2(B-[1-\sqrt{1-\alpha}]\tau).
\end{eqnarray}    
\end{claim}
\begin{proof}
 By the fact that $v_0 = \frac{1}{n}\sum_{i=1}^n v_0^i= \frac{1}{n}\sum_{i=1}^n \mathcal{C}(\clip_\tau(\nabla f_i(x_0)))$, 
 \begin{eqnarray*}
     \norm{v_0} 
     & \overset{\eqref{eq:triangle}}{\leq} & \norm{\nabla f(x_0)} + \norm{\frac{1}{n}\sum_{i=1}^n\mathcal{C}(\clip_\tau(\nabla f_i(x_0))) - \nabla f(x_0)} \\ 
     & \overset{\eqref{eq:triangle}}{\leq} & \norm{\nabla f(x_0)} + \frac{1}{n}\sum_{i=1}^n\norm{\mathcal{C}(\clip_\tau(\nabla f_i(x_0))) - \nabla f_i(x_0)} \\ 
      & \overset{\eqref{eq:triangle}}{\leq} & \norm{\nabla f(x_0)} + \frac{1}{n}\sum_{i=1}^n[\norm{\mathcal{C}(\clip_\tau(\nabla f_i(x_0))) - \clip_\tau(\nabla f_i(x_0))}  \\
      && + \norm{\clip_\tau(\nabla f_i(x_0))- \nabla f_i(x_0)}] \\ 
      & \overset{\eqref{eq:contractive_comp}}{\leq} & \norm{\nabla f(x_0)} + \frac{1}{n}\sum_{i=1}^n[\sqrt{1-\alpha}\norm{ \clip_\tau(\nabla f_i(x_0))} + \norm{\clip_\tau(\nabla f_i(x_0))- \nabla f_i(x_0)}] \\ 
      & \overset{(\text{Lemma~\ref{lem:clip}(ii)-(iii)}) }{\leq} & \norm{\nabla f(x_0)} + \frac{1}{n}\sum_{i=1}^n[\sqrt{1-\alpha}\norm{ \clip_\tau(\nabla f_i(x_0))} + \max\{0, \norm{\nabla f_i(x_0)} - \tau\}] \\
      & \leq & \norm{\nabla f(x_0)} + \sqrt{1-\alpha} \tau + \frac{1}{n}\sum_{i=1}^n \max\{0, \norm{\nabla f_i(x_0)} - \tau\}.
 \end{eqnarray*}
 Here, the last inequality comes from the fact that $\norm{\clip_\tau(v)} \leq \tau$ for $v\in\mathbb{R}^d$.

 If $\max_i \norm{\nabla f_i(x_0)} := B > \tau$ and $\gamma \leq 2/L$, then 
 \begin{eqnarray*}
     \norm{v_0}  
     & \leq &   \norm{\nabla f(x_0)} + \sqrt{1-\alpha} \tau + (B-\tau) \\
     & \overset{\eqref{eq:ssss}}{\leq} & \sqrt{2L[f(x_0)-f_{\inf}]}+ \sqrt{1-\alpha} \tau + (B-\tau) \\ 
     & \leq & \sqrt{2L\phi_0}+ \sqrt{1-\alpha} \tau + (B-\tau) \\
     & \leq & \sqrt{\frac{4}{\gamma}\phi_0}+ \sqrt{1-\alpha} \tau + (B-\tau) \\
     & \leq & \sqrt{\frac{4}{\gamma}\phi_0}+ 2(B-[1-\sqrt{1-\alpha}]\tau).
 \end{eqnarray*}
\end{proof}

\begin{claim}\label{claim:v_k_multi_compclip}
Fix $k\geq 1$. Let $f$ have $L$-Lipschitz gradient. Also suppose that  $\norm{\nabla f_i(x_k)-v_{k-1}^i}\leq {\max}_i \norm{\nabla f_i(x_0)} := B > \tau$, $\norm{v_{k-1}} \leq \sqrt{\frac{4}{\gamma}\phi_0} + 2(B-[1-\sqrt{1-\alpha}]\tau)$, $\norm{\nabla f(x_{k-1})} \leq \sqrt{\frac{2}{\gamma}\phi_0}$, and  $\gamma \leq \frac{1-1/\sqrt{2}}{L}$, then 
\begin{eqnarray}
        \norm{v_k} 
        & \leq & \sqrt{\frac{4}{\gamma}\phi_0} + 2(B-[1-\sqrt{1-\alpha}]\tau).
\end{eqnarray}    
\end{claim}
\begin{proof}
By the fact that $\nabla f(x) = \frac{1}{n}\sum_{i=1}^n \nabla f_i(x)$ and from the definition of $v^i_k$, 
\begin{eqnarray*}
    \norm{v_k} 
    &  \overset{\eqref{eq:triangle}}{\leq} & \norm{\nabla f(x_k)} + \norm{ \frac{1}{n}\sum_{i=1}^n \mathcal{C}(\clip_\tau(m_k^i)) - m_k^i} \\ 
    & \overset{\eqref{eq:triangle}}{\leq} & \norm{\nabla f(x_k)} +\frac{1}{n}\sum_{i=1}^n   \norm{ \mathcal{C}(\clip_\tau(m_k^i)) - m_k^i} \\
    & \overset{\eqref{eq:triangle}}{\leq} & \norm{\nabla f(x_k)} +\frac{1}{n}\sum_{i=1}^n   [\norm{ \mathcal{C}(\clip_\tau(m_k^i)) - \clip_\tau(m_k^i)} + \norm{ \clip_\tau(m_k^i) - m_k^i}] \\
    & \overset{\eqref{eq:contractive_comp}}{\leq} & \norm{\nabla f(x_k)} +\frac{1}{n}\sum_{i=1}^n   [\sqrt{1-\alpha}\norm{  \clip_\tau(m_k^i)} + \norm{ \clip_\tau(m_k^i) - m_k^i}] \\
    & \overset{(\text{Lemma~\ref{lem:clip}(ii)-(iii)}) }{\leq} &
    \norm{\nabla f(x_k)} +\frac{1}{n}\sum_{i=1}^n   [\sqrt{1-\alpha}\norm{  \clip_\tau(m_k^i)} + \max\{0,\norm{m_k^i}-\tau\}] \\
    & \leq &  \norm{\nabla f(x_k)} + \sqrt{1-\alpha}\tau + \frac{1}{n}\sum_{i=1}^n  \max\{0,\norm{m_k^i}-\tau\},
\end{eqnarray*} 
where $m_k^i = \nabla f_i(x_k)-v_{k-1}^i$. The last inequality comes from the fact that $\norm{\clip_\tau(v)} \leq \tau$ for $v\in\mathbb{R}^d$.

If $\norm{\nabla f_i(x_k)-v_{k-1}^i}\leq {\max}_i \norm{\nabla f_i(x_0)} := B > \tau$, then 
\begin{eqnarray*}
    \norm{v_k} 
    & \leq & \norm{\nabla f(x_k)} + \sqrt{1-\alpha}\tau + (B-\tau) \\ 
    & \overset{\eqref{eq:triangle}}{\leq} & \norm{\nabla f(x_{k-1})} + \norm{\nabla f(x_k) - \nabla f(x_{k-1})} + \sqrt{1-\alpha}\tau + (B-\tau) \\
    & \overset{\eqref{eq:L-smooth} + \eqref{eq:x_update_comp_clip}}{\leq} &  \norm{\nabla f(x_{k-1})} + L\gamma \norm{v_{k-1}} + \sqrt{1-\alpha}\tau + (B-\tau). 
\end{eqnarray*}

If $\norm{v_{k-1}} \leq \sqrt{\frac{4}{\gamma}\phi_0} + 2(B-[1-\sqrt{1-\alpha}]\tau)$ and $\norm{\nabla f(x_{k-1})} \leq \sqrt{\frac{2}{\gamma}\phi_0}$, then 
\begin{eqnarray*}
        \norm{v_k} 
        & \leq & (L\gamma + 1/\sqrt{2})\sqrt{\frac{4}{\gamma}\phi_0} + (2L\gamma +1)(B-[1-\sqrt{1-\alpha}]\tau ).
\end{eqnarray*}

If $\gamma \leq \frac{1-1/\sqrt{2}}{L}$, then $\gamma \leq 1/(2L)$ and 
\begin{eqnarray*}
        \norm{v_k} 
        & \leq & \sqrt{\frac{4}{\gamma}\phi_0} + 2(B-[1-\sqrt{1-\alpha}]\tau).
\end{eqnarray*}
\end{proof}

\begin{claim}\label{claim:stepsize_range_compclip}
If
\begin{eqnarray}
    0 < \gamma \leq \frac{(1-\sqrt{1-\alpha})^2\tau^2}{16L_{\max}^2 \left(\sqrt{F_0} + \sqrt{F_0 + \frac{G_0 (1-\sqrt{1-\alpha})\tau}{2\beta L_{\max}}} \right)^2},
\end{eqnarray}
for $F_0 = f(x_0)-f_{\inf}$,  $G_0 = \sqrt{\frac{1}{n}\sum_{i=1}^n (\max\{0,\norm{\nabla f_i(x_0)} - \tau\} + \sqrt{1-\alpha}\tau)^2 }$ and $\beta \geq 0$, 
then
\begin{eqnarray}
    4L_{\max}\sqrt{\gamma \phi_0} \leq\frac{ (1-\sqrt{1-\alpha})\tau}{2}.
\end{eqnarray}
\end{claim}
\begin{proof}
By the definition of $\phi_0$, 
\begin{eqnarray*}
     4L_{\max}\sqrt{\gamma \phi_0} 
     & = & 4L_{\max}\sqrt{\gamma F_0 + \frac{\gamma^2}{\beta} \tilde G_0} \\ 
     & \overset{\eqref{eq:triangle}}{\leq} & 4L_{\max}\sqrt{\gamma} \sqrt{F_0} +  4L_{\max} \frac{\gamma}{\sqrt{\beta}} \sqrt{\tilde G_0},
\end{eqnarray*}
where $F_0 = f(x_0)-f_{\inf}$ and $\tilde G_0 = \frac{1}{n}\sum_{i=1}^n \norm{\nabla f_i(x_0)-v_0^i}^2$. Since $v_{-1}^i=0$ for all $i$, we have $v_0^i = \mathcal{C}(\clip_\tau(\nabla f_i(x_0)))$ and 
\begin{eqnarray*}
    \norm{\nabla f_i(x_0)-v_0^i} 
    & \overset{\eqref{eq:triangle}}{\leq} & \norm{\nabla f_i(x_0) - \clip_\tau(\nabla f_i(x_0))} \\
    &&+ \norm{\clip_\tau(\nabla f_i(x_0)) - \mathcal{C}(\clip_\tau(\nabla f_i(x_0)))} \\ 
    & \overset{\eqref{eq:contractive_comp}}{\leq}&  \norm{\nabla f_i(x_0) - \clip_\tau(\nabla f_i(x_0))} + \sqrt{1-\alpha}\norm{\clip_\tau(\nabla f_i(x_0))} \\ 
    & \leq & \norm{\nabla f_i(x_0) - \clip_\tau(\nabla f_i(x_0))} + \sqrt{1-\alpha}\tau \\ 
    & \overset{(\text{Lemma~\ref{lem:clip}(ii)-(iii)}) }{\leq} & \max\{0,\norm{\nabla f_i(x_0)} - \tau\} + \sqrt{1-\alpha}\tau.
\end{eqnarray*}
Therefore, 
\begin{eqnarray*}
     4L_{\max}\sqrt{\gamma \phi_0} 
       & \leq & 4L_{\max}\sqrt{\gamma} \sqrt{F_0} +  4L_{\max} \frac{\gamma}{\sqrt{\beta}} G_0,
\end{eqnarray*} 
where $G_0 = \sqrt{\frac{1}{n}\sum_{i=1}^n (\max\{0,\norm{\nabla f_i(x_0)} - \tau\} + \sqrt{1-\alpha}\tau)^2 }$. 

Hence, $\gamma > 0$ satisfying 
\begin{align}\label{eqn:step_size_range_compclip_multinode}
    4L_{\max}\sqrt{\gamma} \sqrt{F_0} +  4L_{\max} \frac{\gamma}{\sqrt{\beta}} G_0 \leq \frac{ (1-\sqrt{1-\alpha})\tau}{2}
\end{align}
also satisfies $ 4L_{\max}\sqrt{\gamma \phi_0}  \leq \frac{ (1-\sqrt{1-\alpha})\tau}{2}$. This condition \eqref{eqn:step_size_range_compclip_multinode} can be expressed equivalently as:
\begin{align*}
    \frac{1}{\sqrt{\gamma}} - \frac{8\sqrt{F_0}}{(1-\sqrt{1-\alpha})\tau} L_{\max} - \sqrt{\gamma} \frac{8G_0}{\sqrt{\beta}(1-\sqrt{1-\alpha})\tau L_{\max}} L_{\max}^2 \geq 0.
\end{align*}
Applying Lemma~\ref{lemma:trick_stepsize_EF21} with $L=L_{\max}$, $\beta_1 = \frac{8\sqrt{F_0}}{(1-\sqrt{1-\alpha})\tau}$ and $\beta_2 =\frac{8G_0}{\sqrt{\beta}(1-\sqrt{1-\alpha})\tau L_{\max}}$ yields 
\begin{eqnarray*}
    0 < \sqrt{\gamma} \leq \frac{(1-\sqrt{1-\alpha})\tau}{4L_{\max} \left(\sqrt{F_0} + \sqrt{F_0 + \frac{G_0 (1-\sqrt{1-\alpha})\tau}{2\beta L_{\max}}} \right)}.
\end{eqnarray*}
Finally, taking the square, we obtain the final result. 
\end{proof}

\subsection{Proof for \eqref{eq:descentIneq_compclip_clipped_EF21}}

 To derive our result, let $\mathcal{I}_k$ be the subset from $\{1,2,\ldots,n\}$ such that $\| \nabla f_i(x_k) - v_{k-1}^i \|>\tau$. We then derive the descent inequality . 
	\begin{align*}
		\phi_{k+1} & \leq \phi_k - \frac{\gamma}{2}\| \nabla f(x_k)\|^2,
	\end{align*}
	 for two possible cases: (1) when $\vert \mathcal{I}_k \vert > 0$ and (2) when $\vert \mathcal{I}_k \vert = 0$.
	
	\subsubsection{Case (1): $\vert \mathcal{I}_k \vert > 0$} To derive the descent inequality we will show by induction the stronger result: $\| \nabla f_i(x_{k}) - v^i_{k-1}\| \leq B - \frac{k\tau}{2}$ for $i\in\mathcal{I}_k$, where $B \eqdef \max_i\|\nabla f_i(x_0)\|$ and
	\begin{equation}
		\phi_{k} 
		\leq  \phi_{k-1}  - \frac{\gamma}{2}\norm{\nabla f(x_{k-1}) }^2 - \frac{\gamma}{4}\| v_{k-1}\|^2 \label{eq:detailed_descent_inequality_compclip}
	\end{equation}
	for any $k \geq 0$, where for notational convenience we assume that $\nabla f_i(x_{-1}) = v^i_{-1} = 0$ and $\phi_{-1} = \phi_0$. The base of the induction is trivial: when $k = 0$ we have   $\| \nabla f_i(x_{k}) - v^i_{k-1}\| = \|\nabla f_i(x_0) - v^i_{-1}\| = \|\nabla f_i(x_0)\| \leq \max_i \|\nabla f_i(x_0)\| = B$ for $B>\tau$ and for $i\in\mathcal{I}_{-1}$ and \eqref{eq:detailed_descent_inequality_compclip} holds by definition. Next, we assume that for some $k\geq 0$ inequalities $\|\nabla f_i(x_t) - v^i_{t-1}\| \leq B$ for $i\in\mathcal{I}_{t-1}$ and \eqref{eq:detailed_descent_inequality_compclip} hold for $t = 0,1,\ldots, k$. 

Let $0 \leq \phi_k \leq \phi_0$ for $k \geq 0$. If $\gamma \leq (1-1/\sqrt{2})/L$, then $\gamma \leq 1/L$ and 
\begin{eqnarray*}
    \norm{\nabla f(x_{k-1})}^2 
    & \mathop{\leq}\limits^{\eqref{eq:ssss}} & 2L [f(x_{k-1})-f_{\inf}] \\ 
    & \mathop{\leq}\limits^{\eqref{eq:Lyapunov-89d}}  & 2L \phi_{k-1} \\ 
    & \leq & 2L \phi_0  \\ 
    & \leq & \frac{2}{\gamma}\phi_0.
\end{eqnarray*}
By using the above inequality, and Claim~\ref{claim:v_0_multi_compclip} and \ref{claim:v_k_multi_compclip}, for $k\geq 0$
\begin{align*}
    \norm{v_k} \leq \sqrt{\frac{4}{\gamma}\phi_0} + 2(B-[1-\sqrt{1-\alpha}]\tau).
\end{align*}

Next, from the above inequality, and from Claim~\ref{claim:diff_v_f_multi_compclip}, for $i\in\mathcal{I}_k$
\begin{eqnarray*}
    \norm{ \nabla f_i(x_{k+1}) - v_k^i }
    &\leq&    \norm{\nabla f_i(x_k) - v_{k-1}^i} - \tau + L_{\max}\gamma\norm{v_k} + \sqrt{1-\alpha} \tau \\
    & \leq  &  \norm{\nabla f_i(x_k) - v_{k-1}^i} - [1-\sqrt{1-\alpha}]\tau + L_{\max}\gamma \sqrt{\frac{4}{\gamma}\phi_0} \\
    && + 2L_{\max}\gamma(B - [1-\sqrt{1-\alpha}]\tau) \\
    & =  &  \norm{\nabla f_i(x_k) - v_{k-1}^i} - [1-\sqrt{1-\alpha}]\tau + 2L_{\max} \sqrt{\gamma\phi_0} \\
    && + 2L_{\max}\gamma(B - [1-\sqrt{1-\alpha}]\tau).
\end{eqnarray*}

{\bf STEP: Small step-size}
 
The above inequality and the inductive assumption imply: for $ i\in\mathcal{I}_{k}$
	\begin{align}
		\| \nabla f_i(x_{k+1}) - v^i_{k} \|
		& \leq  \| \nabla f_i(x_{k}) - v^i_{k-1} \| -(1-\sqrt{1-\alpha}) \frac{\tau}{2} \leq B - \frac{(k+1)(1-\sqrt{1-\alpha})\tau}{2} \label{eq:decrease_of_the_norm_diff_compclip}
	\end{align}
	if the step-size $\gamma>0$ satisfies 
	\begin{align}\label{eqn:stepsize_range_compclip_Case_open}
		2L_{\max}\gamma(B - [1-\sqrt{1-\alpha}]\tau)  \leq  2L_{\max} \sqrt{\gamma\phi_0}  \quad \text{and} \quad 4L_{\max}\sqrt{\gamma \phi_0} \leq\frac{ (1-\sqrt{1-\alpha})\tau}{2}.
	\end{align}
By Claim~\ref{claim:stepsize_range_compclip}, this condition can rewritten into: 
\begin{align*}
    \gamma \leq \frac{\phi_0}{(B-[1-\sqrt{1-\alpha}]\tau)^2} \quad \text{and} \quad  \gamma \leq \frac{(1-\sqrt{1-\alpha})^2\tau^2}{16L_{\max}^2 \left(\sqrt{F_0} + \sqrt{F_0 + \frac{G_0 (1-\sqrt{1-\alpha})\tau}{2\beta L_{\max}}} \right)^2}.
\end{align*}
In conclusion, under this step-size condition, $\norm{ \nabla f_i(x_{k+1}) - v^i_{k} } \leq B$ for $i\in\mathcal{I}_{k}$ and $k \geq 0$. In addition, $\mathcal{I}_{k+1} \subseteq \mathcal{I}_{k}$.

{\bf STEP: Descent inequality}

It remains to prove the descent inequality. By the inductive assumption proved above, 
We then have for $i\in\mathcal{I}_{k+1}$
\begin{align}\label{eqn:eta_k_plus_1_compclip}
    \eta^i_{k+1} = \frac{\tau}{\norm{\nabla f_i(x_{k+1}) - v_k^i}} \geq \frac{\tau}{B} \eqdef \eta.
\end{align}
Therefore, 
\begin{eqnarray*}
  \norm{\nabla f_i(x_{k+1}) - v^i_{k+1}}^2 
  & = & \norm{\nabla f_i(x_{k+1})  - v^i_k - \cC(\nabla f_i(x_{k+1})  - v^i_k)}^2 \mathbbm{1}(i\in\mathcal{I}_{k+1}^{\prime})  \\
  && +   \| \nabla f_i(x_{k+1})  - v^i_k - \cC(\clip_\tau(\nabla f_i(x_{k+1})  - v^i_k)) \|^2\mathbbm{1}(i\in\mathcal{I}_{k+1}) \\
  & \overset{\eqref{eq:Young}}{\leq} &  \norm{\nabla f_i(x_{k+1})  - v^i_k - \cC(\nabla f_i(x_{k+1})  - v^i_k)}^2 \mathbbm{1}(i\in\mathcal{I}_{k+1}^{\prime})  \\
  && +   (1+\theta_1)\| \nabla f_i(x_{k+1})  - v^i_k - \clip_\tau(\nabla f_i(x_{k+1})  - v^i_k)\|^2\mathbbm{1}(i\in\mathcal{I}_{k+1}) \\
  && +    (1+1/\theta_1)\| \clip_\tau(\nabla f_i(x_{k+1})  - v^i_k) - \cC(\clip_\tau(\nabla f_i(x_{k+1})  - v^i_k)) \|^2\mathbbm{1}(i\in\mathcal{I}_{k+1}) \\ 
  & \overset{\eqref{eq:contractive_comp}}{\leq}  & (1-\alpha)\norm{\nabla f_i(x_{k+1})  - v^i_k }^2 \mathbbm{1}(i\in\mathcal{I}_{k+1}^{\prime}) \\
  && +   (1+\theta_1)\| \nabla f_i(x_{k+1})  - v^i_k - \clip_\tau(\nabla f_i(x_{k+1})  - v^i_k)\|^2\mathbbm{1}(i\in\mathcal{I}_{k+1}) \\
  && +    (1+1/\theta_1)(1-\alpha)\| \clip_\tau(\nabla f_i(x_{k+1})  - v^i_k) \|^2\mathbbm{1}(i\in\mathcal{I}_{k+1}) \\
   & \overset{\eqref{eq:Young}}{\leq}  & B_1\norm{\nabla f_i(x_{k+1})  - v^i_k }^2 \mathbbm{1}(i\in\mathcal{I}_{k+1}^{\prime}) \\
  && +   B_2 \| \nabla f_i(x_{k+1})  - v^i_k - \clip_\tau(\nabla f_i(x_{k+1})  - v^i_k)\|^2\mathbbm{1}(i\in\mathcal{I}_{k+1}) \\
   & \leq  & B_1\norm{\nabla f_i(x_{k+1})  - v^i_k }^2 \mathbbm{1}(i\in\mathcal{I}_{k+1}^{\prime}) \\
  && +   B_2 (1-\eta_{k+1}^i)^2 \| \nabla f_i(x_{k+1})  - v^i_k \|^2\mathbbm{1}(i\in\mathcal{I}_{k+1}) \\
  & \overset{\eqref{eqn:eta_k_plus_1_compclip}}{\leq} & 
  \max(B_1,B_2(1-\eta)^2)\norm{\nabla f_i(x_{k+1})  - v^i_k }^2,
\end{eqnarray*}
for $B_1=(1-\alpha) + (1+1/\theta_1)(1+\theta_2)(1-\alpha)$, $B_2=(1+\theta_1) + (1+1/\theta_1)(1+1/\theta_2)(1-\alpha)$, and $\theta_1, \theta_2 >0$.

Suppose that there exists $\theta_1,\theta_2>0$ such that $1-\beta :=  \max(B_1,B_2(1-\eta)^2)$ for $\beta\in[0,\alpha]$. Then, 
\begin{eqnarray*}
  \norm{\nabla f_i(x_{k+1}) - v^i_{k+1}}^2 
    & \leq & (1-\beta)\norm{\nabla f_i(x_{k+1})  - v^i_k }^2 \\ 
    & \overset{\eqref{eq:Young}}{\leq} & (1+\theta)(1-\beta)\norm{\nabla f_i(x_k) - v^i_k}^2 \\
    && + (1+1/\theta)(1-\beta)\norm{\nabla f_i(x_{k+1}) - \nabla f_i(x_k)}^2 \\
    & \overset{\eqref{eq:L_i-smooth}}{\leq} & (1+\theta)(1-\beta)\norm{\nabla f_i(x_k) - v^i_k}^2 \\
    && + (1+1/\theta)(1-\beta)L_{\max}^2 \norm{x_{k+1} - x_k}^2, 
\end{eqnarray*} 
for $\theta>0$.

If $\theta = \beta/2$, then 
\begin{eqnarray}\label{eqn:TrickFromEF21_multinode_compclip}
  \norm{\nabla f_i(x_{k+1}) - v^i_{k+1}}^2 \leq (1-\beta/2)\norm{\nabla f_i(x_k) - v^i_k}^2 + (1+2/\beta)(1-\beta)L_{\max}^2 \norm{x_{k+1} - x_k}^2.
\end{eqnarray} 

Hence, we can obtain the descent inequality. From  Lemma \ref{lemma:descent_li2021page}, 
\begin{eqnarray*}
    \phi_{k+1}
   & = & f(x_{k+1}) - f_{\inf} + A \frac{1}{n}\sum_{i=1}^n \| \nabla f_i(x_{k+1}) - v^i_{k+1} \|^2\\
	& \mathop{\leq}\limits^{\eqref{eq:descent_ineq}} & f(x_k)- f_{\inf} - \frac{\gamma}{2}\norm{ \nabla f(x_k)  }^2 - \left( \frac{1}{2\gamma} - \frac{L}{2} \right)\norm{ x_{k+1} - x_k }^2  \\
	&& + \frac{\gamma}{2}\frac{1}{n}\sum_{i=1}^n\norm{ \nabla f_i(x_k) - v^i_{k} }^2 + A\frac{1}{n}\sum_{i=1}^n\norm{   \nabla f_i(x_{k+1}) - v^i_{k+1}    }^2  \\
	& = & \phi_k - \frac{\gamma}{2}\norm{ \nabla f(x_k)  }^2 - \left( \frac{1}{2\gamma} - \frac{L}{2} \right)\norm{ x_{k+1} - x_k }^2  \\
	&& + \left(\frac{\gamma}{2} -A\right)\frac{1}{n}\sum_{i=1}^n\norm{ \nabla f_i(x_k) - v^i_{k} }^2 + A\frac{1}{n}\sum_{i=1}^n\norm{   \nabla f_i(x_{k+1}) - v^i_{k+1}    }^2  \\
	& \mathop{\leq}\limits^{\eqref{eqn:TrickFromEF21_multinode_compclip}} & \phi_k - \frac{\gamma}{2}\norm{ \nabla f(x_k)  }^2 - \left( \frac{1}{2\gamma} - \frac{L}{2} -  A  \left( 1 + \frac{2}{\beta}\right) (1-\beta) L_{\max}^2 \right)\norm{ x_{k+1} - x_k }^2 \\
	&& + \left( \frac{\gamma}{2} + A(1-\beta/2) - A\right)\frac{1}{n}\sum_{i=1}^n\norm{ \nabla f_i(x_{k}) - v^i_{k}   }^2  .
\end{eqnarray*}
Since $A= \frac{\gamma}{\beta}$, 
\begin{eqnarray*}
    \phi_{k+1} \leq \phi_k - \frac{\gamma}{2}\norm{ \nabla f(x_k)  }^2 - \left( \frac{1}{2\gamma} - \frac{L}{2} -  \frac{\gamma}{\beta}  \left( 1 + \frac{2}{\beta}\right) (1-\beta) L_{\max}^2 \right)\norm{ x_{k+1} - x_k }^2. 
\end{eqnarray*}

If the step-size $\gamma>0$ satisfies
\begin{eqnarray*}
    \gamma \leq \frac{1}{L\left(1 + \sqrt{1 + \frac{4}{\beta}(1+2/\beta)(1-\beta) \frac{L_{\max}^2}{L^2} } \right)},
\end{eqnarray*}
then from  Lemma \ref{lemma:trick_stepsize_EF21} with $\beta_1=2$ and $\beta_2 = \frac{4}{\beta}(1+2/\beta)(1-\beta) \frac{L_{\max}^2}{L^2}$, this condition implies that $  \frac{1}{2\gamma} - \frac{L}{2} -  \frac{\gamma}{\beta}  \left( 1 + \frac{2}{\beta}\right) (1-\beta) L_{\max}^2 \geq \frac{1}{4\gamma}$ and that 
\begin{eqnarray*}
    \phi_{k+1} 
    & \leq & \phi_k - \frac{\gamma}{2}\norm{ \nabla f(x_k)  }^2 - \frac{1}{4\gamma}\norm{ x_{k+1} - x_k }^2 \\
    & \overset{\eqref{eq:x_update_comp_clip}}{=} & \phi_k - \frac{\gamma}{2}\norm{ \nabla f(x_k)  }^2 - \frac{\gamma}{4}\norm{ v_k }^2.
\end{eqnarray*}

This concludes the proof in the case (1).

\subsubsection{Case (2): $\vert \mathcal{I}_k \vert=0$}
	Suppose $\vert \mathcal{I}_k \vert=0$. Then, 
	 we show by the induction that $\norm{ \nabla f_i(x_k) - v^i_{k-1}  } \leq \tau$ for all $i$ and 
	\begin{align}\label{eq:detailed_descent_inequality_compclip_case2}
		\phi_{k+1} \leq \phi_k - \frac{\gamma}{2}\|  \nabla f(x_{k-1})\|^2 - \frac{\gamma}{4}\|  v_{k-1}\|^2.
	\end{align}

First, we prove the first statement by the induction.
$\norm{ \nabla f_i(x_t) - v^i_{t-1}  } \leq \tau$ is true for $t=0,1,\ldots,k$.

 Let $0 \leq \phi_k \leq \phi_0$ for $k \geq 0$. If $\gamma \leq (1-1/\sqrt{2})/L$, then $\gamma \leq 1/L$ and 
\begin{eqnarray*}
    \norm{\nabla f(x_{k})}^2 
    & \mathop{\leq}\limits^{\eqref{eq:ssss}} & 2L [f(x_{k})-f(x^\star)] \\ 
    & \mathop{\leq}\limits^{\eqref{eq:Lyapunov-89d}}  & 2L \phi_{k} \\ 
    & \leq & 2L \phi_0  \\ 
    & \leq & \frac{2}{\gamma}\phi_0.
\end{eqnarray*}
Hence, from Claim~\ref{claim:diff_v_f_multi_compclip}, 
\begin{eqnarray*}
    \norm{ \nabla f_i(x_{k+1}) - v_k^i }
    &\leq &   L_{\max}\gamma\norm{v_k} + \sqrt{1-\alpha} \tau \\
    & = & L_{\max}\gamma\norm{\frac{1}{n}\sum_{i=1}^n v_{k-1}^i + \mathcal{C}(\nabla f_i(x_k)-v_{k-1}^i)} + \sqrt{1-\alpha} \tau \\
    & \overset{\eqref{eq:triangle}}{\leq} & L_{\max}\gamma \norm{\nabla f(x_k)} + L_{\max}\gamma \frac{1}{n}\sum_{i=1}^n \norm{(\nabla f_i(x_k)-v_{k-1}^i)-\mathcal{C}(\nabla f_i(x_k)-v_{k-1}^i)} \\
    && + \sqrt{1-\alpha} \tau \\
     & \overset{\eqref{eq:contractive_comp}}{\leq} & L_{\max}\gamma \norm{\nabla f(x_k)} + L_{\max}\gamma \sqrt{1-\alpha} \frac{1}{n}\sum_{i=1}^n \norm{\nabla f_i(x_k)-v_{k-1}^i} + \sqrt{1-\alpha} \tau \\
     & \leq & L_{\max}\gamma \norm{\nabla f(x_k)} + (L_{\max}\gamma + 1) \sqrt{1-\alpha} \tau \\
     & \leq &  L_{\max}\gamma \sqrt{\frac{2}{\gamma}\phi_0}+ (L_{\max}\gamma + 1) \sqrt{1-\alpha} \tau
\end{eqnarray*}   

{\bf STEP: Small step-size}

$ \| \nabla f_i(x_{k+1}) - v_k^i \|  \leq \tau$ holds if the step-size $\gamma>0$ satisfies
	\begin{align*}
		\gamma \leq \frac{1-\sqrt{1-\alpha}}{2\sqrt{1-\alpha}L_{\max}} \quad \text{and} \quad  4L_{\max}\sqrt{\gamma \phi_0} \leq\frac{ (1-\sqrt{1-\alpha})\tau}{2}.
	\end{align*}
 Here, the second condition is  fulfilled when it satisfies \eqref{eqn:stepsize_range_compclip_Case_open}.

{\bf STEP: Descent inequality}

	It remains to prove the descent inequality.  Note that \algname{Press-Clip21-GD}  reduces to \algname{EF21} at step $k$. Thus, 
	%Then, by using \eqref{eqn:ineq_v_k} and by the fact that $\eta_k = \eta = 1$, we have $v_k = \nabla F(x_k)$. 
	%
	\begin{align*}
		\| \nabla f_i(x_{k+1}) - v^i_{k+1} \|^2 
		& = \| \nabla f_i(x_{k+1}) - v_k^i - \cC(\nabla f_i(x_{k+1}) - v_k^i ) \|^2 \\
		& \leq (1-\alpha) \|  \nabla f_i(x_{k+1}) - v_k^i \|^2 \\
		& \leq (1-\alpha)(1+\theta)\| \nabla f_i(x_{k}) - v_k^i  \|^2 +  (1-\alpha)(1+1/\theta)\| \nabla f_i(x_{k}) - \nabla f_i(x_{k-1})  \|^2.
	\end{align*}
	By letting $\theta= \alpha/2$ and by the smoothness of each $f_i(x)$, 
	\begin{align}\label{eqn:TrickFromEF21_Comp_Clipped_Case2}
		\| \nabla f_i(x_{k+1}) - v^i_{k+1} \|^2 
		& \leq (1-\alpha/2)\| \nabla f_i(x_{k}) - v_k^i  \|^2 +  (1-\alpha)(1+2/\alpha) (L_{\max})^2 \| x_{k} - x_{k-1}  \|^2.
	\end{align}
	From the definition of $\phi_{k}$ and Lemma \ref{lemma:descent_li2021page},	
	\begin{align*}
		\phi_{k+1}  
		& \leq  f(x_k)- f_{\inf} - \frac{\gamma}{2}\norm{ \nabla f(x_k) }^2 - \left( \frac{1}{2\gamma} - \frac{L}{2} \right)\| x_{k+1} - x_k \|^2  \\
		&\hspace{0.5cm} + \frac{\gamma}{2}\frac{1}{n}\sum_{i=1}^n\norm{ \nabla f(x_k) - v_{k}   }^2 + A\frac{1}{n}\sum_{i=1}^n\| \nabla f_i(x_{k+1}) - v^i_{k+1} \|^2   \\
		& = \phi_k - \frac{\gamma}{2}\norm{ \nabla f(x_k) }^2 - \left( \frac{1}{2\gamma} - \frac{L}{2} \right)\| x_{k+1} - x_k \|^2  \\
		&\hspace{0.5cm} + \left(\frac{\gamma}{2} -A\right)\frac{1}{n}\sum_{i=1}^n\norm{ \nabla f(x_k) - v_{k}   }^2 + A\frac{1}{n}\sum_{i=1}^n\| \nabla f_i(x_{k+1}) - v^i_{k+1} \|^2   \\
		& \mathop{\leq}^{\eqref{eqn:TrickFromEF21_Comp_Clipped_Case2}} \phi_k - \frac{\gamma}{2}\norm{ \nabla f(x_k) }^2 - \left( \frac{1}{2\gamma} - \frac{L}{2} -  A  (1-\alpha)(1+2/\alpha) (L_{\max})^2 \right)\| x_{k+1} - x_k \|^2 \\
		&\hspace{0.5cm} + \left( \frac{\gamma}{2} + A(1-\alpha/2) - A\right)\frac{1}{n}\sum_{i=1}^n\norm{ \nabla f_i(x_{k}) - v^i_{k}   }^2  .
	\end{align*}
	Since $A = \frac{\gamma}{\beta}$ with $\beta \in [0,\alpha]$, we get
	\begin{align*}
		\phi_{k+1} \leq \phi_k - \frac{\gamma}{2}\norm{ \nabla f(x_k) }^2 - \frac{1}{2\gamma}\left(1 - \gamma L - \gamma^2 \cdot \frac{2(1-\alpha)(1+2/\alpha) }{\beta}  (L_{\max})^2\right)\| x_{k+1} - x_k \|^2.
	\end{align*}
	
	If the step-size $\gamma$ satisfies 
	\begin{align*}
		0 < \gamma \leq \frac{1}{L \left(1 + \sqrt{1 +  \frac{4(1-\alpha)(1+2/\alpha) }{\beta} \left( \frac{L_{\max}}{L} \right)^2 }\ \right)},
	\end{align*}
	then from  Lemma \ref{lemma:trick_stepsize_EF21} with $L=1$, $\beta_1=2L$ and $\beta_2 = \frac{4(1-\alpha)(1+2/\alpha) }{\beta}\left( \frac{L_{\max}}{L} \right)^2$, this condition implies $1 - \gamma L - \gamma^2 \cdot \frac{2(1-\alpha)(1+2/\alpha) }{\beta}  (L_{\max})^2  \geq \frac{1}{2}$
	and thus
	\begin{align*}
		\phi_{k+1} \leq \phi_k - \frac{\gamma}{2}\norm{ \nabla f(x_k) }^2 - \frac{1}{4\gamma}\| x_{k+1} - x_k \|^2.
	\end{align*}
	Since $x_{k+1}-x_k = -\gamma v_k$,
	\begin{align}\label{eqn:descentIneq_EF21_Compclip}
		\phi_{k+1} 
		& \leq  \phi_k  - \frac{\gamma}{2}\norm{ \nabla f(x_k) }^2 - \frac{\gamma}{4}\norm{ v_k  }^2. 
	\end{align}
	Putting all the conditions on $\gamma$ together, we obtain the results.
	
\subsection{Proof for the convergence bound  \eqref{eq:descentIneq_compclip_clipped_EF21_final}}	
Let $\hat x_K$ be selected uniformly at random from $\{x_0,x_1,\ldots,x_{K-1}\}$. Then,
\begin{align*}
	\Exp{ \| \nabla f(\hat x_K) \|^2 } & = \frac{1}{K}\sum_{k=0}^{K-1} \norm{ \nabla f(x_k) }^2.
\end{align*}
From \eqref{eq:descentIneq_compclip_clipped_EF21},
\begin{align*}
	\Exp{ \| \nabla f(\hat x_K) \|^2 } & \leq \frac{2}{\gamma}\frac{1}{K}\sum_{k=0}^{K-1} (\phi_k - \phi_{k+1}) = \frac{2(\phi_0 - \phi_K)}{\gamma K} \leq \frac{2\phi_0}{\gamma K}.
\end{align*}

\end{document}